\documentclass{article}
\usepackage[margin=1in]{geometry}
\usepackage{natbib}
\usepackage{mathtools} 
\usepackage{amsmath,amssymb,amsfonts}
\usepackage{algorithmic}
\usepackage{graphicx}
\usepackage{textcomp}
\usepackage{mathtools}
\usepackage{xcolor}
\def\BibTeX{{\rm B\kern-.05em{\sc i\kern-.025em b}\kern-.08em
    T\kern-.1667em\lower.7ex\hbox{E}\kern-.125emX}}
\usepackage{booktabs}
\usepackage{bm} %
\usepackage{amsmath} %

\usepackage[usenames,dvipsnames]{xcolor}
\usepackage{color}
\usepackage{colortbl}
\usepackage{subcaption} %
\usepackage[printonlyused]{acronym}
\usepackage{tikz}
\usetikzlibrary{patterns}
\usetikzlibrary{shapes.geometric, arrows.meta, positioning, fit, shapes.arrows, shapes}
\usepackage{pgfplots}
\pgfplotsset{compat=1.18} %
\usepgfplotslibrary{groupplots}

\usepackage{amsthm}  %
\newtheorem{theorem}{Theorem}[section]  %
\newtheorem{lemma}{Lemma}[section]     %
\newtheorem{proposition}{Proposition}[section] %
\newtheorem{definition}{Definition}[section]  %
\newtheorem{remark}{Remark}[section]  %
\newtheorem{fact}{Fact}[section]  %
\newtheorem{assumption}{Assumption}[section]  %
\newtheorem{configuration}{Configuration}

\usepackage{hyperref}
\usepackage{cleveref}
\usepackage{url}
\Crefname{assumption}{Assumption}{Assumptions}
\Crefname{fact}{Fact}{Facts}
\definecolor{lightblue}{rgb}{0.83,0.85,1.0}
\definecolor{white}{rgb}{1.0,1.0,1.0}
\definecolor{light-gray}{gray}{0.95}
\renewcommand{\arraystretch}{1.5}
\newcolumntype{C}{>{\centering\arraybackslash}p{0.5cm}}
\newcolumntype{T}{>{\columncolor{lightblue}}c}
\newcolumntype{V}{>{\columncolor{light-gray}}c}
\usepackage{diagbox}

\usepackage{float}

\usepackage{comment}

\newcommand{\R}{\mathbb{R}}

\definecolor{mohsen}{RGB}{0,128,0}

\usepackage{caption}
\usepackage{wrapfig}
\usepackage{xspace}

\newcommand{\calD}{\mathcal{D}}
\newcommand{\calE}{\mathcal{E}}
\newcommand{\calF}{\mathcal{F}}
\newcommand{\calG}{\mathcal{G}}
\newcommand{\calH}{\mathcal{H}}
\newcommand{\calI}{\mathcal{I}}

\newcommand{\calM}{\mathcal{M}}

\newcommand{\calO}{\mathcal{O}}
\newcommand{\calP}{\mathcal{P}}

\newcommand{\calS}{\mathcal{S}}
\newcommand{\calT}{\mathcal{T}}
\newcommand{\calU}{\mathcal{U}}

\newcommand{\calX}{\mathcal{X}}
\newcommand{\calY}{\mathcal{Y}}
\newcommand{\calZ}{\mathcal{Z}}

\DeclareMathOperator{\argmin}{arg\,min}

\newcommand{\lp}{\left(}   %
\newcommand{\rp}{\right)}  %

\newcommand{\lb}{\left[}   %
\newcommand{\rb}{\right]}  %

\newcommand{\lbr}{\left\{}  %
\newcommand{\rbr}{\right\}} %

\newcommand{\lv}{\left\lvert}  %
\newcommand{\rv}{\right\rvert} %
\newcommand{\iid}{i.i.d.\xspace}
\newcommand{\dkl}{D_{KL}}

\newcommand{\PA}{P^{(A)}}
\newcommand{\QA}{Q^{(A)}}
\newcommand{\That}{\hat{T}}
\newcommand{\EA}{\mathbb{E}^{(A)}}
\newcommand{\EE}{\mathbb{E}}
\newcommand{\qtext}[1]{\quad\text{#1}\quad}
\usepackage{autonum}

\newcommand{\LNF}{L^{\mathrm{NF}}}
\newcommand{\LhatNF}{\widehat{L}^{\mathrm{NF}}}
\newcommand{\empPX}{\mathbb{P}^{(X)}}
\newcommand{\empPZ}{\mathbb{P}^{(Z)}}
\newcommand{\PX}{\mathbb{P}^{(X)}}
\newcommand{\PZ}{\mathbb{P}^{(Z)}}
\newcommand{\rmD}{\mathrm{D}}

\newcommand{\wtilde}{\widetilde{w}}
\newcommand{\Utilde}{\widetilde{U}}

\newcommand{\clip}{\mathrm{clip}}

\newcommand{\RadComp}{\mathfrak{R}}

\newcommand{\htilde}{\tilde{h}}
\newcommand{\xx}{\mathbf{x}}
\newcommand{\yy}{\mathbf{y}}
\newcommand{\zz}{\mathbf{z}}
\newcommand{\dx}{d_{\mathbf{x}}}
\newcommand{\dy}{d_{\mathbf{y}}}
\newcommand{\dz}{d_{\mathbf{z}}}
\newcommand{\LhatY}{\widehat{L}_{\mathbf{y},n}}
\newcommand{\LhatZ}{\widehat{L}_{\mathbf{z}, n}}
\newcommand{\Lhat}{\widehat{L}_n}
\newcommand{\Jac}{\mathbb{J}}
\newcommand{\NFrealizability}{\textbf{(NF1)}\xspace}
\newcommand{\NFvariational}{\textbf{(NF3)}\xspace}
\newcommand{\NFmoments}{\textbf{(NF4)}\xspace}
\newcommand{\NFconvergence}{\textbf{(NF2)}\xspace}

\begin{document}

\newacro{NF}{normalizing flow}
\newacro{IK}{inverse kinematics}
\newacro{NLL}{negative log-likelihood}
\newacro{INN}{invertible neural network}
\newacro{BNN}{Bayesian neural network}
\newacro{MMD}{maximum mean discrepancy}
\newacro{ABC}{approximate Bayesian computation}
\newacro{PINN}{physics-informed neural networks}
\newacro{GAN}{generative adversarial network}
\newacro{IPM}{integral probability metric}
\newacro{SL}{supervised loss}
\newacro{USL}{unsupervised loss}
\newacro{RKHS}{reproducing kernel Hilbert space}
\newacro{ERM}{empirical risk minimization}
\newacro{VINA}{variational invertible neural architecture}
\newacro{GI}{
Geoacoustic inversion}
\newacro{KL}{
Kullback-Leibler}
\newacro{JS}{Jensen–Shannon}

\title{
VINA: Variational Invertible Neural Architectures
}

\author{
  Shubhanshu Shekhar\thanks{Equal contribution.}\\
  EECS Department\\
  University of Michigan, Ann Arbor\\
  \texttt{shubhan@umich.edu}
  \and
  Mohammad Javad Khojasteh\footnotemark[1]\\
  EME Department\\
  Rochester Institute of Technology\\
  \texttt{mjkem@rit.edu}
  \and
  Ananya Acharya\\
  EME Department\\
  Rochester Institute of Technology\\
  \texttt{aa2334@rit.edu}
  \and
  Tony Tohme\\
  Department of Mechanical Engineering\\
  Massachusetts Institute of Technology\\
  \texttt{tohme@mit.edu}
  \and
  Kamal Youcef-Toumi\\
  Department of Mechanical Engineering\\
  Massachusetts Institute of Technology\\
  \texttt{youcef@mit.edu}
}

\date{}

\maketitle
\begin{abstract}
The distinctive architectural features of normalizing flows (NFs), notably bijectivity and tractable Jacobians, make them well-suited for generative modeling.
Invertible neural networks (INNs) build on these principles to address supervised inverse problems, enabling direct modeling of both forward and inverse mappings.
In this paper, we revisit these architectures from both theoretical and practical perspectives and address a key gap in the literature: the lack of theoretical
guarantees on approximation quality under realistic assumptions, whether for posterior inference in INNs or for generative modeling with NFs.

We introduce a unified framework for INNs and NFs based on variational unsupervised loss functions, inspired by analogous formulations in related areas such as generative
adversarial networks (GANs) and the Precision-Recall divergence for training normalizing flows. Within this framework, we derive theoretical performance guarantees, quantifying
posterior accuracy for INNs and distributional accuracy for NFs, under assumptions that are weaker and more practically realistic than those used in prior work.

Building on these theoretical results, we conduct extensive case studies to distill general design principles and practical guidelines. We conclude by demonstrating the
effectiveness of our approach on a realistic ocean-acoustic inversion problem.
\end{abstract}

\section{Introduction}
\label{sec:introduction}

Recently, machine learning approaches, particularly those based on deep neural networks, have emerged as effective alternatives to conventional inverse problem solvers. Among these,  \acp{INN}~\citep{ardizzone2018analyzing} stand out due to their ability to model complex, non-linear relationships while ensuring invertibility between input and output spaces. \ac{INN}s offer advantages such as bijectivity and computationally tractable Jacobians, which make them particularly suitable for solving inverse problems, where both forward and inverse mappings need to be computed accurately and efficiently.  \acp{INN} have been applied for solving inverse problems in various fields, including epidemiology~\citep{radev2021outbreakflow}, astrophysics~\citep{ardizzone2018analyzing}, optics~\citep{Luce_2023}, geophysics~\citep{zhang2021bayesian,WuHuaZha:23}, and reservoir engineering~\citep{padmanabha2021solving}.

\begin{wrapfigure}{r}{0.5\textwidth}
\centering
\resizebox{0.47\textwidth}{!}{%
\usetikzlibrary{arrows.meta,positioning,calc}
\begin{tikzpicture}[>=Stealth, font=\small]

% side boxes
\node[draw, rounded corners=2pt, minimum width=10mm, minimum height=28mm, align=center] (L) {$X$};
\node[draw, rounded corners=2pt, minimum width=10mm, minimum height=28mm, align=center, right=75mm of L] (R) {};

% geometry for the double-headed arrow
\coordinate (A) at ($(L.east)+(8mm,0)$);
\coordinate (B) at ($(R.west)+(-8mm,0)$);
\def\h{3mm} % <-- half thickness of the arrow (make this smaller/thicker as desired)
\def\w{4mm} % <-- head length

\coordinate (TopR) at ($(R.west)+(0,9mm)$);
\coordinate (BotR) at ($(R.west)+ (0,-9mm)$);
\node at (R.center |- TopR) {$Y$};
\node at (R.center |- BotR) {$Z$};

% thin double-headed arrow as a polygon
\filldraw[fill=blue!12, draw=blue!60, line width=0.7pt]
  ($(A)+(-\w,0)$) --
  ($(A)+(0,\h)$) --
  ($(B)+(0,\h)$) --
  ($(B)+(\w,0)$) --
  ($(B)+(0,-\h)$) --
  ($(A)+(0,-\h)$) -- cycle;

% label centered on the arrow
\node[font=\bfseries] at ($(A)!0.5!(B)$) {Invertible Neural Network};

% dashed guide arrows + labels
\draw[dashed,gray,->, very thick] ($(L.east)+(0,9mm)$) -- ($(R.west)+(0,9mm)$)
  node[midway, above=2pt, black]{forward modeling};
\draw[dashed,gray,->, very thick] ($(R.west)+(0,-9mm)$) -- ($(L.east)+(0,-9mm)$)
  node[midway, below=2pt, black]{inverse prediction};

\end{tikzpicture}
}
\caption{\acp{INN} can be represented by an invertible map $T$ that approximates the relation from the input $X \in \mathbb{R}^{\dx}$ to the output $Y \in \mathbb{R}^{\dy}$ and a latent variable $Z \in \mathbb{R}^{\dx - \dy}$. The invertibility of $T$ means that for any $\yy \in \mathbb{R}^{\dy}$, we also have an approximate  posterior sampling  distribution via $T^{-1}(\yy, Z)$.}
\label{INN_Figure}
\end{wrapfigure}
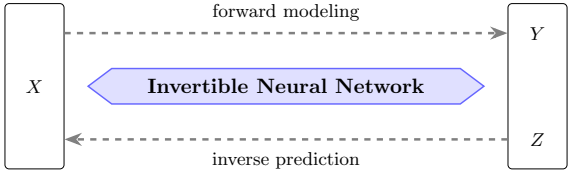

The name \ac{INN} in the literature has been used to refer to supervised 
models trained to solve inverse problems, where the training loss involves a supervised loss (to model observed outputs) as well as an unsupervised component~\citep{ardizzone2019guided,guan2024efficient,ardizzone2018analyzing,hagemann2021stabilizing}. 
There exists a related class of models based on similar architectural principles, referred to as \acp{NF}, which are primarily used for unsupervised density estimation and generative modeling~\citep{papamakarios2021normalizing,papamakarios2017masked,gomez2017reversible,kingma2018glow,dinh2014nice}. 
For ease of exposition, we also use this convention and distinguish between \acp{INN} and \acp{NF} based on the existence of the supervised loss in the training process.

A typical inverse problem can be described as follows:  We are given a dataset $\mathcal{D} = \{(X_i, Y_i) \in \mathbb{R}^{\dx + \dy}\}_{i=1}^n$ consisting of $n$ input-output pairs, where $X_i \in \mathbb{R}^{d_{\mathbf{x}}}$ is the input feature vector and $Y_i \in \mathbb{R}^{d_{\mathbf{y}}}$ is the target vector, from a joint distribution $(X, Y) \sim P_{X,Y}$.  Usually we have $\dx \geq \dy$ representing some inherent information loss in a measurement process~(forward map). This joint distribution can be decomposed into a prior $P_X$ and a possibly randomized transformation $P_{Y \mid X}$ representing the forward process. The general goal of the Bayesian approach to inverse problems~\citep{stuart2010inverse,dashti2012besov} is to learn or approximate the posterior distribution $P_{X \mid Y}$ given the data $\mathcal{D}$.  
As depicted in~\Cref{INN_Figure}, the approach taken by \acp{INN} to achieve this goal is to incorporate an additional latent output variable,  denoted by $Z \in \mathbb{R}^{\dz}$ where $\dz = \dx - \dy$. This latent variable is designed to represent information related to $X$ that is not contained in $Y$, and its distribution $P_Z$ is selected by the data analyst~\citep{hagemann2021stabilizing}. An \ac{INN} consists of a map $T \equiv (T_{\yy}, T_{\zz}): \mathbb{R}^{\dx} \to \mathbb{R}^{\dy \times \dz}$  that is continuous and invertible with a continuous inverse, and usually they are also assumed to be continuously differentiable to enable gradient based training procedures.
Let $\calT$ denote an appropriate parametrized collection of such functions~(we discuss some architectural details in~\Cref{sec:background}).  The model is trained by minimizing an objective of the form
\begin{align}
\label{eq:forwardcostINN}
    T^*_{\ac{INN}} \in \argmin_{T \in \calT} \; \lp \mathbb{E}_{XY}[\|T_\yy(X) -Y\|_2^2] \;\;+\;\; \lambda \, D(P_{Y,Z}, P_{Y,T_{\zz}(X)}) \rp,
\end{align}
where the two terms correspond to the supervised and the unsupervised (population) losses, respectively, and $\lambda>0$ is a regularizing constant. 
In this display, we use $D$ to denote some divergence or distance measure~(such as $f$-divergence, kernel-MMD, Wasserstein metric, etc.) between the joint distributions of $P_{Y,T_{\zz}(X)}$ and $P_{Y,Z} = P_Y  P_Z$. 
The supervised term encourages the component $T_{\yy}$ to approximate the forward map $P_{Y\mid X}$, while the unsupervised term matches the joint distribution of $P_{Y, T_{\zz}(X)}$ with the output-latent  variable pair $P_{Y,Z}$. Crucially, due to the invertibility of $T$, it is possible to show that under certain conditions, making these two terms small induces an approximate posterior sampling distribution: for a given $Y=\yy$, the distribution of $T^{-1}(\yy, Z)$ for $Z \sim P_Z$ is close~(in the same divergence $D$) to the true posterior $P_{X\mid Y=\yy}$.  Note that unlike~\citet{hagemann2021stabilizing}, we use $P_{Y, T_{\zz}(X)}$ instead of $P_{T_{\yy}(X), T_{\zz}(X)}$ as the second argument of $D(\cdot, \cdot)$ to simplify some of the technical arguments. We expect that our results can be extended to the unsupervised loss with $P_{T_{\yy}(X), T_{\zz}(X)}$, and we consider this choice in some of our experiments.

In the purely unsupervised setting, the same invertible models reduce to %
\acp{NF}.  In particular, we have access to a dataset $\calD = \{X_i \in \mathbb{R}^{d} \}_{i=1}^n$, and our goal is to estimate the distribution $P_X$ of these observations. As in the case of \acp{INN}, the idea is to introduce a \emph{latent variable}  $Z \sim P_Z$, also in $\mathbb{R}^{d}$, and learn an invertible map $T: \mathbb{R}^{d} \to \mathbb{R}^d$ that minimizes some notion of divergence between $P_Z$ and the forward map $T(X)$; that is, 
\begin{align}
    T^*_{\ac{NF}} \in \argmin_{T \in \calT} \; \; D( P_Z, P_{T(X)}). 
\end{align}
By making the divergence between the distributions of $Z$ and $T(X)$ small, the invertibility of $T$ can be used to establish closeness (in the same divergence) between the  distributions of the inverse map $X$, and $T^{-1}(Z)$, as desired. Within this formulation, in this paper, we view \acp{NF} as a purely unsupervised variant of the more general class of \acp{INN}.  

The choice of divergence $D$ in the training objective critically shapes the behavior of \acp{INN}~(and \acp{NF}) as different choices encourage different behavior in the trained model. For instance, forward relative entropy~(or KL divergence) induces mode coverage behavior, while reverse relative entropy encourages mode seeking. Jensen-Shannon~(JS) divergence provides a natural compromise between these two extremes~\citep{polyanskiy2025information}. Other popular choices include the Wasserstein metrics that are known to provide informative gradients, and kernel \ac{MMD} which is an important instance of the family of \acp{IPM}~\citep{sriperumbudur2012empirical}. In practice, these choices influence not only the final model's eventual behavior but also the optimization dynamics, sample quality, and robustness to constraints such as limited data and architectural constraints.

\subsection{Overview of our contributions} 
We now present an overview of our main contributions, which can be divided into the proposal of unifying variational framework, theoretical characterization of the quality of approximation achieved by the \ac{ERM}-based model, experimental case studies that verify and augment the theory, and a real-world ocean acoustic inversion application. 

\paragraph{A unifying framework.} The starting point of our work is the observation that a large class of divergence measures used to define the unsupervised loss term in the \ac{INN} and~\ac{NF} training objective admit a variational (or dual) representation. This allows an interpretation of the training process as a saddle-point optimization problem. While we present the formal details in~\Cref{sec:general-framework}, we illustrate the idea using an $f$-divergence between two distributions on $\mathbb{R}^d$, denoted by $D_f(P \parallel Q)$, which admits the following classical Donsker-Varadhan~(DV) type variational representation~\citep[Theorem 7.26]{polyanskiy2025information}: 
\begin{align}
    D_f(P \parallel Q) = \sup_{g: \mathbb{E}_Q[f^*(g(X))]< \infty} \; \lbr \mathbb{E}_P[g(X)] - \mathbb{E}_Q[f^{\ast}(g(X))] \rbr, \qtext{where} f^{\ast}(b) = \sup_{a \in \mathbb{R}} \; a b - f(a)  
\end{align}
denotes the convex conjugate of $f$ { (also known as Fenchel conjugate~\citep{rockafellar2015convex})}. 
With this formulation, the training of an \ac{INN} given the dataset~(augmented with samples from the latent distribution $P_Z$) $\calD = \{(X_i, Y_i, Z_i) \sim P_{X,Y,Z} \}_{i=1}^n$, a model class $\calT$, and a ``critic class''  $\calG \subset \{g: \mathbb{E}_Q[f^*(g(X))] < \infty\}$, can be represented as 
\begin{align}
    \That_n \in \argmin_{T \in \calT} \lbr  \frac{1}{n} \sum_{i=1}^n \|T_\yy(X_i) - Y_i\|_2^2 + \lambda \sup_{g \in \calG} \lp \frac 1 n \sum_{i=1}^n g(Y_i, Z_i)  - \frac{1}{n} \sum_{j=1}^n f^*\big(g (Y,T_{\zz}(X))\big) \rp \rbr. 
\end{align}
Similar expressions can be obtained for the general family of \acp{IPM}, Wasserstein-1~($W_1$) metric using Kantorovich-Rubinstein duality~\citep{villani2008optimal}. We refer to this class of models as \emph{Variational Invertible Neural Architectures}~(VINA).  

\paragraph{Theoretical analysis.} In~\Cref{subsec:variational-inn}~(\acp{INN}) and~\Cref{subsec:variational-nf}~(\acp{NF}), we analyze the quality of the \ac{ERM} model $\That_n$ under the unified variational framework. In both cases, we present our results at three ``levels''. At the first level~(\Cref{theorem:finite-sample} and~\Cref{theorem:nf-estimation-error}), we state how certain assumptions, such as realizability, uniform learnability, moment bounds, and representation ability of the critic class can allow us to convert the empirical loss into a high probability guarantee on the approximation quality of~$\That_n$ in terms of the $W_1$ metric~(see~\Cref{remark:main-theorem-W1} for justification of this choice).  At the next level~(\Cref{subsec:INN-interpretable-conditions} and~\Cref{subsubsec:interpretable-conditions-nf}), we identify sufficient and more naturally verifiable conditions under which the requirements of the previous level are satisfied. In particular, we present conditions on the Rademacher complexity, Lipschitz regularity of the model class, and certain tightness properties associated with the latent distribution and model class. At the third and final level~(\Cref{subsec:INN-example} and~\Cref{subsubsec:variational-nf-example}), we present a concrete instantiation of a practically relevant model class~(iResNet models), and critic class~(norm constrained \acp{RKHS}), and verify that they satisfy all the conditions obtained at the previous level. 

Unlike existing theoretical guarantees, our results incorporate some practical aspects of training \acp{INN} and \acp{NF}. \citet{ardizzone2018analyzing} showed that if a model $T$ achieves exactly zero population loss~(supervised squared error + unsupervised kernel-MMD loss), then sampling via $T^{-1}(\yy, Z)$ exactly recovers the posterior $P_{X \mid \yy}$. In practice, however, zero population loss is seldom achieved.~\citet{hagemann2021stabilizing} addressed this issue by bounding the $W_1$ distance between the true posterior and the estimated posterior in terms of the population loss~(combination of supervised and unsupervised loss).
However, their analysis places a rather strong bounded support assumption on the input and output of the network, which excludes common choices such as Gaussian latent variables.  
In this work, we relax this assumption by only requiring a finite-moment assumption which significantly broadens the scope of our result to more practical scenarios. In fact, we obtain the first formal quantification of the quality of the posterior distribution represented by \acp{INN} trained via empirical risk minimization~(ERM) in terms of the Wasserstein~($W_1$) metric,  and also discuss how it can be extended to other metrics such as kernel-MMD.

\paragraph{Empirical Results.}  We then perform a thorough empirical study  to deliver practically useful insights and general design principles that complement the theoretical results mentioned above. 
    More specifically, we study the impact of the different design choices involved in our framework, such as the effect of the choice of divergence / metric, the choice of the latent distribution, and the prior loss. The key insights obtained from these empirical evaluations are summarized below. 
    \begin{itemize}

     \item In~\Cref{sec:observation-1-prior}, we study the significance of integrating the knowledge of prior distribution on the input space~($\calX = \mathbb{R}^{\dx}$) into the training of \acp{INN}. Our empirical findings show that \emph{employing a well-specified prior can significantly improve the performance of \acp{INN}}, but a misspecified can lead to a degradation in the quality of the generated samples.

      \item  \sloppy In~\Cref{sec:observation-2-f-div}, we compare coupling-based and iResNet architectures using $f$-divergence-based unsupervised losses in both forward and backward directions. 
We observe that training performance is influenced by multiple factors, including the network architecture, the training procedure, and the capacity of the critic class, with \textit{backward $f$-divergence losses tending to yield lower inference error}. 
In line with prior work in the literature (e.g.,~\citet{behrmann2019invertible}), coupling-based architectures demonstrate improved efficiency relative to iResNet-based models.

        \item In~\Cref{sec:observation-3}, we observe \emph{an
 approximately local U-shaped dependence between latent dimension and inference performance of \acp{INN}.} More specifically, as we change the latent dimension size $\dz$~(and appropriately modify $\dx$ by padding), we observe that  increasing dimension initially improves results, and then starts degrading,  and the overall trend is  multimodal.

    \item In~\Cref{sec:observation-4}, we investigate the practical aspects of solving inverse problems using Wasserstein Distance~\citep{villani2008optimal}. 
    Calculating the Wasserstein distance is computationally challenging and susceptible to the curse of dimensionality~\citep{fournier2015rate}. Here, we deploy the entropic estimate of the Wasserstein distance ($W_2$), both Sinkhorn approximation\citep{cuturi2013sinkhorn} and Sinkhorn divergence~\citep{peyre2019computational}, as a metric within our training process. We observe that \emph{lower entropic regularization improves sample quality but increases training time for both Sinkhorn Divergence and Sinkhorn approximation of Wasserstein Distance, while Sinkhorn divergence demonstrates greater stability than the Sinkhorn approximation.}  
       
    \item In~\Cref{sec:observation-5}, we study the effect of the support  of the latent  distribution when training with the  KL divergence and the Sinkhorn divergence. It is well known that  $f$-divergences  perform poorly when the two distributions are mutually singular\citep{zhang2019variational,zhang2020spread}, while 
      Wasserstein distances  can provide informative gradients even when supports do not overlap because it reflects the underlying geometric cost.
    Our empirical observations in this section validate this hypothesis, and the \emph{models trained with the Sinkhorn divergence were more robust to mismatch between the supports of the latent and prior distributions.} 
     
    \item In~\Cref{sec:observation-6}, we study the relation between the number of finite moments of $X$, and the performance of the \ac{NF} model. As \emph{predicted by our theoretical results}, we find that as the number of finite moment of $X$ increases, the $W_1$ distance between the true and estimated distributions decreases.

    \end{itemize}

Overall, our exploratory experiments in~\Cref{subsec:observations} provide empirical validation of some of the theoretical predictions of~\Cref{sec:main-results}, and also provide some general practical insights in training invertible architectures. 

\paragraph{Ocean-Acoustic Application.} 

To demonstrate the practical relevance of our approach,  we employ the insights gained from the exploratory experiments in~\Cref{subsec:observations} to the ocean–acoustic inversion setting of the SWellEx-96 experiment conducted off the coast of San Diego near Point Loma~\citep{yardim2010geoacoustic,MeyGem:J21}.
\ac{GI}, as demonstrated in the SWellEx-96 study, is a challenging and computationally intensive problem~\citep{DosDet:11,huang2006uncertainty,chapman2021review}. This paper utilizes synthetic data to simplify the task for the initial application of an \ac{INN}-based framework.
These preliminary studies revealed that the latent-space dimensionality plays a critical role in estimating the posterior, and that physically meaningful priors, such as uniform priors for uncertain quantities like sound speed, improve stability. 
We demonstrate that invertible architectures can offer a favorable computational trade-off relative to likelihood-based MCMC sampling for \ac{GI}, in which likelihood evaluations rely on repeated calls to the forward model KRAKEN~\citep{porter1992kraken}. Although no method is universally best, moving this computational burden offline into training enables a pre-trained \ac{INN} to support rapid, near-real-time posterior inference at test time.

To summarize, our work advances both the theory and practice of invertible neural architectures. We provide a first unified \ac{ERM} analysis of invertible models that provides explicit bounds~(in $W_1$ metric) on the quality of approximation achieved by the trained model under realistic moment and capacity assumptions. Through concrete instantiations we show that these conditions are satisfied by practically useful models, and finally we implement a series of empirical case studies and apply our ideas to a  real-world ocean-acoustic inversion problem. 

\subsection{Organization of the paper}
The remainder of the paper is organized as follows. We present a thorough review of the related work in~\Cref{sec:related-work}, and then recall some background on the architecture and training details of existing invertible models. We present our main results in~\Cref{sec:main-results}, and in particular, propose our variational training strategy in~\Cref{sec:general-framework} and derive theoretical results for  \acp{INN} in~\Cref{subsec:variational-inn}. Also, we extend our theoretical analysis to variational  \acp{NF} in~\Cref{subsec:variational-nf}. We then move on to the empirical part of the paper in~\Cref{sec:case-studies}. In particular, in~\Cref{subsec:observations} we present a series of observations about the effect of various design choices on the performance of invertible models. Some of these observations verify the theoretical predictions from the previous section (such as the effect of the complexity of critic class, and the effect of the number of finite moments of the latent distributions), while others concern important aspects such as the latent dimension, choice of divergence measure etc. Finally, in~\Cref{subsec:geoacoustic_inversion}, we apply our ideas to a real-world task of ocean-acoustic inversion.

\section{Related Works}
\label{sec:related-work}
\paragraph{Invertible architectures~(\acp{NF} and \acp{INN}):}
There are numerous ways to achieve invertibility in neural architectures~\citep{tabak2013family,kobyzev2020normalizing,keller2021self},  such as by using residual connections~\citep{gomez2017reversible,jacobsen2018revnet,behrmann2019invertible},  triangularization~\citep{bogachev2005triangular,marzouk2016introduction,parno2016multiscale,jm2017improving}, and  coupling-based normalizing flows (e.g., RealNVP~\citet{dinh2016density} and Glow~\citet{kingma2018glow}).
In practice, coupling-based architecture strikes the right balance between computational efficiency and expressivity. In fact, under appropriate assumptions, they have been shown to be universal diffeomorphism approximators~\citep{teshima2020coupling,jin2024approximation}. 

While \acp{NF} were developed largely for tractable density modeling, their supervised variants \acp{INN} have become increasingly popular for solving inverse problems.
We explore this particular direction in~\Cref{subsec:geoacoustic_inversion}. However, \acp{INN}~(as defined in our paper) are not the only way to employ invertible neural networks to solve inverse problems.  There exist other approaches, such as using conditional~\acp{INN}~\citep{ardizzone2019guided,winkler2019learning}.  Another approach is to use conditional \acp{NF} to learn the likelihoods~\citep{papamakarios2019sequential},  which can be used to estimate posterior by combining it with a prior distribution using Bayes rule. \acp{NF} have also been integrated with sampling-based Bayesian inference, to combine the advantages of both approaches~\citep{song2017nice,winter2023machine,kruse2025enhanced}, as will be discussed further below.

Our research focuses on \acp{INN}, which provide a probabilistic  framework for addressing inverse problems. However, it's important to recognize the growing popularity of neural networks in solving inverse problem. For instance,  In non-probabilistic contexts, works~\citet{ying2022solving, fan2019solving, khoo2019switchnet} have  designed specialized neural architectures that embed physical formulation to recover unknown parameters while reducing the reliance on large amounts of data.

\paragraph{\acp{INN} vs. sampling-based Bayesian methods:}
Conventional Markov Chain Monte Carlo (MCMC) methods~\citep{mackay2003information, brooks2011handbook, andrieu2003introduction, doucet2005monte,korattikara2014austerity,kungurtsev2023decentralized,atchade2005adaptive} are widely used for sampling from complex probability distributions, including posterior distributions arising in inverse problems~\citep{geweke1989bayesian}.
From a pushforward viewpoint, MCMC defines a Markov transition kernel  with the target posterior as its stationary distribution; under standard ergodicity conditions, the chain converges asymptotically to the exact posterior.
However, standard likelihood-evaluation-based MCMC methods often require substantial computational time to achieve adequate sampling~\citep{roy2020convergence,jones2022markov}, which becomes a bottleneck, particularly in inverse problems where the forward process is computationally expensive to evaluate. Likelihood-free MCMC variants exist but can be substantially more computationally demanding in practice~\citep{beaumont2019approximate}.

On the other hand, invertible generative models parameterize an invertible map that pushes a simple base distribution to an induced (learned) distribution that approximates the target distribution. 
When the Jacobian determinant is tractable, this learned  distribution can be evaluated exactly via the change-of-variables formula. 
In context of inverse problem, once an \ac{INN} is trained,
approximating the posterior involves running the network in reverse, which is computationally inexpensive. 
We demonstrate that invertible architectures can offer a favorable computational trade-off relative to repeated likelihood-evaluation-based MCMC sampling for a practical inverse problem.
This speedup comes from shifting computation offline into training. 
In fact, under appropriate assumptions, \acp{INN} can even perform in amortized settings, where a model trained extensively can generalize effectively to new instances without retraining~\citep{radev2023bayesflow,radev2021outbreakflow}. 

Since  practically relevant  performance guarantees on invertible architecture are scarce in literature~(we discuss some existing results later in this section), \citet{gabrie2021efficient,gabrie2022adaptive,brofos2022adaptation,schonle2023optimizing} have incorporated \acp{NF} to characterize the proposal distributions of MCMC methods, thereby accelerating convergence in practical tasks, while retaining theoretical guarantees for MCMC convergence.
In this work, we establish error bounds for our proposed class of invertible models under appropriate assumptions, providing theoretical characterizations of posterior accuracy for \acp{INN} and generative accuracy for \acp{NF}. 
Taken together, our results clarify when offline-trained invertible maps with controlled approximation error can replace repeated online likelihood-based Monte Carlo sampling.

Posterior sampling in high dimensions is often computationally challenging~\citep{montanari2023posterior}, and there is no universally best configuration for these problems. %
As demonstrated in our empirical studies, \acp{INN} exhibit sensitivity to several design and optimization choices such as  the learning rate and architectural configuration. To address this, in our cases studies, we employed tree-structured Parzen estimator ~\citep{bergstra2011algorithms} to  search for promising hyperparameters including: the size of the subnetworks, the latent dimensionality, and the number of coupling layers.

\paragraph{Unsupervised loss:} The choice of training objective critically shapes the behavior of \ac{INN}s as these models are trained through  optimization. 
\ac{INN}s were initially trained using \ac{MMD}~\citep{ardizzone2018analyzing}, which only requires samples and avoids the need for explicit density modeling. However, due to challenges such as poor scaling in high dimensions and sensitivity to kernel choice~\citep{ramdas2015decreasing}, \ac{NLL}~\citep{dinh2016density} has become more prevalent in recent \ac{INN} architectures~\citep{ren2020benchmarking, kruse2021benchmarking}.
The use of \ac{NLL} can be justified as an approximation of training by minimizing the KL-divergence~\citep{kingma2019introduction,papamakarios2021normalizing}. 
Training an \ac{INN} using \ac{NLL} in the literature is based on making assumptions about the distribution of the output samples, for instance Gaussian distribution has been used in prior work~\citep{ren2020benchmarking, kruse2021benchmarking}. In this paper, we eliminate the need for such distributional assumptions by relying solely on samples and leveraging the variational formulation of f-divergences, as introduced by ~\citet{nguyen2010estimating} and follow-up works~\citep{nowozin2016f,ruderman2012tighter,ke2021imitation}.  

Our approach is inspired by analogous formulations in related domains, such as generative adversarial networks (GANs)~\citep{goodfellow2014generative,nowozin2016f}, Flow-GAN~\citep{grover2018flow} and the 
Precision-Recall divergence for training \acp{NF}~\citep{verine2023precision}. 
To train invertible architectures, we take a unified look at the usage of a class of variational distance metrics over the space of probability measures, such as \acp{IPM}~\citep{sriperumbudur2012empirical}, which include the \ac{MMD}~\citep{li2017mmd} and Wasserstein metrics~\citep{arjovsky2017wasserstein,coeurdoux2022sliced}, and  the $f$-divergence family~\citep{csiszar1967information,arjovsky2017towards}, and we study its implication in  supervised inverse problems (variational \ac{INN}) and  in unsupervised generative modeling tasks~(Variational \ac{NF}). Training based on the variational representation of f-divergence enables stronger theoretical guarantees.  By extending and unifying the results of ~\citet{nguyen2010estimating} and~\citet{hagemann2021stabilizing}, we derive new convergence bounds for \ac{INN} that are both tighter and hold under weaker assumptions than prior work.
In particular,~\citet[Theorem 2]{hagemann2021stabilizing} establish an approximation guarantee for the posterior distribution of an \ac{INN} represented by a homeomorphism $T$ in terms of the error bounds on the supervised and unsupervised losses. Our results extend their result in two main ways: first we relax a strong requirement of bounded observations imposed by~\citet{hagemann2021stabilizing} that omits usual choices like Gaussian latent distributions, and second, we analyze the performance of a data-driven model $\That_n$ trained using an \ac{ERM} strategy, unlike \citet{hagemann2021stabilizing} who work with a fixed $T$ satisfying certain approximation guarantees.

\paragraph{Latent distribution:} 
It is well established that the structure and topology of the latent space play a crucial role in the accuracy of generative models~\citep{gurumurthy2017deligan,bevins2023piecewise,stimper2022resampling,laszkiewicz2021copula,hickling2024flexible,fadel2021principled}.
While the most commonly used latent distributions for \acp{INN} is the standard Gaussian distribution~\citep{ardizzone2018analyzing},
the works in \citet{behrmann2021understanding,hagemann2021stabilizing} shows that a simple Gaussian prior is often insufficient for multimodal problems and choosing a suitable latent distribution can improve robustness.  In this paper, we study the effect of size of latent variable and  support  set of this variable in performance of \ac{INN}.

\begin{table*}[t]
    \caption{Definitions of important acronyms.}
    \label{tab:acronyms}
    \renewcommand{\arraystretch}{1.5}
    \providecommand {\cellLeft} [1] {\hspace{0.2cm}#1}
    \centering
    \begin{tabular}{  m{1.2cm}   m{5.3cm} | m{1.2cm} m{5.3cm}   }
        \hline

        \rowcolor{Turquoise!12!white} 
        \cellLeft{\acs{NF}
        } & \acl{NF}
        &
        \cellLeft{\acs{INN}
        } & \acl{INN}  
        \\
        
        \cellLeft{\acs{IK}
        } & \acl{IK}
        &
        \cellLeft{\acs{BNN}
        } & \acl{BNN} 
        \\
        \rowcolor{Turquoise!12!white} 
        
            \cellLeft{\acs{NLL}
        } & \acl{NLL}
        &
        \cellLeft{\acs{MMD}
        } & \acl{MMD} 
        \\
        \cellLeft{\acs{ERM}
        } & \acl{ERM}
        &
        \cellLeft{\acs{KL}
        } & \acl{KL}
        \\ 
        \rowcolor{Turquoise!12!white}
        \cellLeft{\acs{JS}
        } & \acl{JS}
        &
        \cellLeft{\acs{IPM}
        } & \acl{IPM}
        \\

        \cellLeft{\acs{SL}
        } & \acl{SL}
        &
        \cellLeft{\acs{USL}
        } & \acl{USL}
        \\ 
        \rowcolor{Turquoise!12!white}
        \cellLeft{\acs{GAN}
        } & \acl{GAN}
        &
        \cellLeft{\acs{GI}
        } & \acl{GI}
        \\

        \hline
    \end{tabular}
\end{table*}

A selective discussion of the existing architectural choices and design of \acp{INN} is provided in~\Cref{sec:background}.

\section{Main Results}
\label{sec:main-results}
We begin this section by presenting a general unsupervised training cost for \ac{VINA}  in~\Cref{sec:general-framework}. As mentioned earlier, this formulation unifies several commonly used cost functions in the literature. For simplicity we focus our presentation in~\Cref{sec:general-framework} for the case of \acp{INN}, but we then illustrate how it naturally applies to \acp{NF} in~\Cref{subsec:variational-nf}.

  \subsection{Variational \ac{USL}}
\label{sec:general-framework}
Our main methodological contribution in this paper begins with the observation, also used in priors works such as~\citet{nowozin2016f, zhang2019variational, grover2018flow}, that several commonly used statistical divergence measures used in generative modeling admit variational representations. More specifically, with $T:\mathbb{R}^{\dx} \to \mathbb{R}^{\dy} \times \mathbb{R}^{\dz}$ denoting an \ac{INN}, we wish to minimize some divergence of the form 
    \begin{align}
    D_\phi(P_{Y,Z}, P_{Y,T_{\zz}(X)}) = \sup_{g\in \calG_{\phi}} \mathbb{E}\lb \phi(g, X, Y, Z, T) \rb, \quad \stackrel{\text{emp. loss}}{\implies} \quad 
         \hat{L} = \sup_{g \in \calG_\phi} \; \frac{1}{n} \sum_{i=1}^n \phi\lp g, X_i, Y_i, Z_i, T \rp.  \label{eq:variationa-USL}
    \end{align}
Here $\calG_\phi$ denotes some class of functions, usually from $\calX = \mathbb{R}^{\dx}$ to $\mathbb{R}$, and we will refer to it as the ``critic class'' following the convention used in the \ac{GAN} literature~\citep{goodfellow2014generative}. In practice, this class can be represented by machine learning models such as neural networks, or more analytically tractable classes such as \acp{RKHS}. Our theoretical results in the next two subsections explore the trade-offs involved in choosing more expressive $\calG_\phi$ and approximation guarantees. 
For different choices of $(\phi, \calG_\phi)$, the above formulation recovers various popular loss functions such as relative entropy~(and more generally, the $f$-divergence family), kernel-\ac{MMD}~(and more generally, the \ac{IPM} family), energy distance, and the Wasserstein metric.  Besides the conceptual unification, this approach also allows us to obtain theoretical guarantees on the approximation performance of \acp{INN} trained via empirical risk minimization strategy, as we discuss in~\Cref{subsec:variational-inn} for \acp{INN}~(and in~\Cref{subsec:variational-nf}  for \acp{NF}). 

Before proceeding further, we recall two important family of distance or divergence measures that are realizations of~\eqref{eq:variationa-USL}. The first class of distance metrics are the \acp{IPM}, which for some function class $\calG$ that is closed under negation~(i.e., if $g \in \calG$, then so does $-g$), are defined as 
\begin{align}
   \mathrm{IPM}(P_{Y,Z}, P_{Y, T_{\zz}(X)}) = \sup_{g \in \calG}  \mathbb{E}\lb g(Y,Z) - g(Y, T_{\zz}(X))  \rb 
\end{align}
which corresponds to  $\phi_{\mathrm{IPM}}(g, x, y, z, T) =  g(y,z) - g(y,T_{\zz}(x))$.  
Perhaps the most important element of the \ac{IPM} family is the kernel-\ac{MMD} metric, where the critic class~(also known as the witness class) is the unit norm ball \ac{RKHS} associated with a positive-definite kernel $k:\calX \times \calX \to \mathbb{R}$. An immediate consequence of the reproducing property of such $\calG$ is that for square integrable kernels, we have 
\begin{align}
    \mathrm{MMD}(P_{Y,Z}, P_{Y, T_{\zz}(X)}) = \| \mu_{P_{Y,Z}} - \mu_{P_{Y,T_{\zz}(X)}}\|_k, \qtext{where} \mu_P = \int k(x, \cdot) dP(x), 
\end{align}
is referred to as the kernel mean-embedding of a distribution $P$, $\|\cdot\|_k$ denotes the \ac{RKHS} norm, and $\int$ above is a Bochner integral. This leads to a very natural interpretation of \ac{MMD} metric between two distributions $P$ and $Q$: it is the distance~(as measured by the \ac{RKHS} norm) between the two representative elements $\mu_P$ and $\mu_Q$ associated with the two distributions. Other important \acp{IPM} include the total variation distance, and the Wasserstein~$1$~metric~(also known as the earth-mover distance), which by using Kantorovich-Rubinstein duality~(Theorem 1.14~\citet{villani2021topics}, and Theorem 11.8.2.~\citet{dudley2018real}) can be represented as an \ac{IPM} associated with the class $\calG$ of $1$-Lipschitz functions.

The second important class of divergence measures captured by~\eqref{eq:variationa-USL} is the family of $f$-divergences. For any convex, lower semicontinuous $f:(0,\infty) \to (0,\infty)$ with $f(1)=0$, the $f$-divergence between  $P_{Y,Z}$ and $P_{Y,T_{\zz}(X)}$, assuming that $P_{Y,Z} \ll P_{Y,T_{\zz}(X)}$, is defined as 
\hspace{-1mm}
\begin{align}
    D_f(P_{Y,Z} \parallel P_{Y,T_{\zz}(X)}) 
    = \int q_{T}(y, T_{\zz}(x)) f \lp \frac{p_{Y,Z}(y,z)}{p_T(y,T_{\zz}(x))} \rp \nu(x,y,z) dx dy dz, 
\end{align}
where  we have assumed that both distributions admit densities $p_T$ and $p_{Y,Z}$, with respect to some common dominating measure $\nu$, for all choices of $T$ in the model class. 
These divergence measures are also known to admit the following variational representation 
\begin{align}
    D_f(P_{Y,Z}\parallel P_{Y,T_{\zz}(X)} ) = \sup_{g} \lp \mathbb{E}[g(Y, Z) - f^*(g(Y,T_{\zz}(X)))]\rp, \label{eq:f-div}
\end{align}
which corresponds to $\phi_f(g, x, y, z, T) = g(y,z) - f^*(g(y,T_{\zz}(x))$.
Here,  the supremum is over all measurable $g:\calY \times \calZ \to \mathbb{R}$ make the right-hand side (RHS) well-defined~(by avoiding $\infty - \infty$). This implies that if we restrict our attention to any smaller class of functions $\calG$, get a lower bound  on the divergence, and this gap~(that we later refer to as the variational gap) decreases by enlarging the critic class $\calG$. This fact will play a central role in our analysis in the next two subsections. 

The canonical member of the $f$-divergence family is the (forward)  relative entropy~(or \ac{KL} divergence), with generator function $f(u) = u \log u$. The corresponding $\phi$ function in~\eqref{eq:variationa-USL} is $\phi_{KL}(g, x, y, z, T) = \log g(y,z) - g(T(x)) + 1$. Other important $f$-divergences are the squared Hellinger divergence with $f(u) = (\sqrt{u} - 1)^2$ and $\phi_H(g, x, y, z, T) = g(y,z) - g(T(x))/(1-g(T(x)))$, and the Jensen Shannon divergence with $f(u) = -u(+1) \log((1+u)/2) + u \log u$ and $\phi_{JS}(g, x, y, z, T) = g(y,z) + \log \lp 2 - e^{g(T(x))}\rp$.

We now proceed to a discussion of \acp{INN} defined using a variational objective as introduced in~\eqref{eq:variationa-USL}.

\subsection{Variational  invertible neural networks (V-INN)}
\label{subsec:variational-inn}

Let $\calX = \mathbb{R}^{\dx}$, $\calY = \mathbb{R}^{\dy}$ and $\calZ = \mathbb{R}^{\dz}$, with $\dx = \dy + \dz$, and suppose we have a dataset $\{(X_i, Y_i, Z_i) \in \calX \times \calY \times \calZ: 1 \leq i \leq n\}$ drawn \iid from the joint distribution $P_{XYZ}$. Our goal is to learn an invertible neural network~(\ac{INN}) from a family of homeomorphisms~(i.e., each $T:\calX \to \calY\times \calZ$ is a continuous bijective map with a continuous inverse) using this training dataset.  Each candidate \ac{INN} $T$ can be decomposed into $T_{\mathbf{y}}:\calX \to \calY$ and $T_{\mathbf{z}}:\calX \to \calZ$, and consequently, the \ac{INN} training process involves population loss functions: 
\begin{align}
    &L_{\mathbf{y}}(T) = \mathbb{E}_{P_{XY}}\lb \|T_{\mathbf{y}}(X) - Y\|_2^2 \rb,  \qtext{and} 
    L_{\mathbf{z}}(T) \leq D(P_{Y, Z}, P_{Y, T_{\mathbf{z}}(X)}), 
\end{align} 
where $D(P, Q)$ is an appropriate notion of distance or divergence between probability distributions $ P, Q \in \calP(\calY \times \calZ)$. 
In this section, we will focus on the case of $D$ being an $f$-divergence and employ its variational definition stated in~\eqref{eq:f-div} within the empirical risk minimization~(ERM) framework.
The use of an ``$\leq$'' instead of ``$=$'' when introducing $L_{\mathbf{z}}(T)$ above is a consequence of the variational definition with a restricted critic class as we make precise in \textbf{(INN4)} in~\Cref{assump:INN-theorem}. More explicitly, we work with a function class $\calG_n$~(the ``critic class''), and define $L_\zz(T)$ as 
\begin{align}
L_\zz(T) = \sup_{g \in \calG_n} \lbr \EE_{P_{YZ}}[g(Y,Z)] - \EE_{P_{XY}}[f^*(g(Y, T_\zz(X)))] \rbr \leq D_f(P_{Y,Z}, P_{Y, T_{\zz}(X)}).      \label{eq:def-population-Lz}
\end{align}
To train an \ac{INN} from a family $\calT$ (i.e., $\calT$ might represent all \acp{INN} with a fixed architecture), we use the empirical analogs of the population loss terms,  
\begin{align}
    &\LhatY(T) = \frac{1}{n} \sum_{i=1}^n \|T_{\mathbf{y}}(X_i) - Y_i\|_2^2,  \quad 
    \LhatZ(T) = \sup_{g \in \calG_n}\lbr \frac{1}{n} \sum_{i=1}^n g(Y_i, Z_i) - \frac{1}{n} \sum_{i=1}^n f^* \circ g(Y_i, T_{\mathbf{z}}(X_i)) \rbr, 
\end{align}
where $f^*$ denotes the convex conjugate of $f$.  The model learned by ERM can then be defined as 
\begin{align}
    \hat{T}_n \in \argmin_{T \in \calT} \; \LhatY(T) + \lambda \LhatZ(T), \label{eq:erm-INN}
\end{align}
for some regularization parameter $\lambda >0$. The first loss term~($\LhatY$) forces the learned map $\That_n$ to approximate $Y$ from an input $X$, while the second loss term~($\LhatZ$) ensures that the joint law of $(Y, \That_{n, \zz}(X))$ matches the true joint distribution $(Y,Z)$. 
Thus, intuitively, if the sample-size $n$ is large enough, we should expect $\hat{T}_n$ to be a good proxy for $T^*$, an element of $\calT$ that minimizes the population risk: 
\begin{align}
    T^* \in \argmin_{T \in \calT}\; L_{\mathbf{y}}(T) + \lambda L_{\mathbf{z}}(T). \label{eq:population-optimal-INN}
\end{align}
The performance of $\That_n$ defined above is governed by two effects: (i) the finite sample effect caused by the uniform deviations between the empirical objectives $\LhatY$ and $\LhatZ$ from their population counterparts; and (ii) the variational or approximation error induced by restricting the divergence representation to a class $\calG_n$. To make the variational loss smaller, we need to increase the capacity of $\calG_n$. However, that also causes the uniform deviations between the empirical and population losses to increase. Thus, the crucial challenge is to find the right trade-offs between these two effects to simultaneously drive the overall error to zero. 

For the rest of this section, we follow a three-step roadmap. \emph{First}, we state a general result~(\Cref{theorem:finite-sample}) that obtains an upper bound on the quality of the posterior approximation provided by $\That_n$ under a set of high-level assumptions~(stated in~\Cref{assump:INN-theorem}). \emph{Second}, we present more concrete, verifiable, sufficient conditions for satisfying the requirements of~\Cref{assump:INN-theorem} in~\Cref{subsec:INN-interpretable-conditions}. \emph{Finally}, we specialize the general theorem to a specific choice of \ac{INN} architecture~(iRes-Nets in~\Cref{def:inn-calT}) and critic class~(\ac{RKHS} over truncated domain in~\Cref{def:inn-calG}) using Jensen-Shannon divergence in~\Cref{subsec:INN-example}.

We now present the assumptions  required for stating the main result of this section. 
\begin{assumption}
\label{assump:INN-theorem}
To analyze the performance of the \ac{INN} model $\That_n$ defined in~\eqref{eq:erm-INN}, we place the following assumptions: 
\begin{itemize}
    \item \textbf{(INN1): Realizability and Bi-Lipschitz.}
    There exists a $T^* \in \calT$ such that $Y = T_{\yy}^*(X)$ almost surely (a.s.) and $P_{Y, Z} = P_{Y, T^*_z(X)}$, and furthermore 
    \begin{align}
        \sup_{T \in \calT} \lbr \mathrm{Lip}(T), \, \mathrm{Lip}(T^{-1}) \rbr \leq J < \infty. 
    \end{align}

    \item \textbf{(INN2): Uniform Convergence.} For any confidence level $\delta \in (0,1)$, there exist deterministic sequences $\{r_n \equiv r_n(\delta): n \geq 1\}$ and $\{u_n \equiv u_n(\delta): n \geq 1\}$, with $r_n, u_n \to 0$, such that the following conditions hold with probability at least $1-\delta$: 
    \begin{align}
        \sup_{T \in \calT} |\LhatY(T) - L_{\yy}(T)| \leq u_n, \qtext{and} \sup_{T \in \calT} |\LhatZ(T) - L_{\zz}(T)| \leq r_n. 
    \end{align}
    \item \textbf{(INN3): Moment Bounds.} There exists an $R<\infty$, such that the following holds for some $a>0$:
    \begin{align}
        \max \lbr \mathbb{E}[\|(Y, Z)\|^{1+a}], \; \sup_{T \in \calT} \mathbb{E}[\|(Y, T_{\mathbf{z}}(X)\|^{1+a}] \rbr \leq R. 
    \end{align}
    Here $\|\cdot\|$ to denotes the $\ell_2$ norm, and we use $\|(y,z)\|$ as a shorthand for $\sqrt{\|y\|^2  + \|z\|^2}$. 
    \item \textbf{(INN4): Variational Approximation Gap.} The function class $\calG_n$ contains the $\boldsymbol{0}$ function for all $n \geq 1$, and there exists a vanishing deterministic sequence $\eta_n \to 0$, such that 
    \begin{align}
        \sup_{T \in \calT} D_f(P_{Y,Z} \parallel P_{Y, T_{\mathbf{z}}(X)} ) -  L_{z}(T) \leq \eta_n. 
    \end{align}
    In other words, we assume that the capacity of the critic class $\calG_n$ grows with $n$ to approximate the likelihood ratios of all distributions modeled by elements of $\calT$ and the true data distribution $P_{Y,Z}$. 
\end{itemize}
\end{assumption}

\begin{theorem}
    \label{theorem:finite-sample} 
    Suppose~\Cref{assump:INN-theorem} holds, and $D_f$ satisfies the Pinsker-type inequality $c_f \sqrt{D_f(P \parallel Q)} \geq TV(P, Q)$ for all distributions $P, Q$, and for some constant $c_f>0$.  Then, for every measurable $A \subset \calY$ with $P_Y(A)>0$, on the $(1-\delta)$ probability event of \textbf{(INN2)}, we have 
    \begin{align}
        W_1\lp \PA_X ,  \PA_{\That_n^{-1}(Y,Z) }\rp  & \lesssim (r_n + u_n + \eta_n)^{\frac{a}{2(1+a)}},
    \end{align}
    where  $W_1$ denotes the $1$-Wasserstein metric, $\PA_X = P_{X|Y\in A}$, $\PA_{\That_n^{-1}(Y,Z)} = P_{\hat{T}^{-1}_n(Y, Z)|Y \in A}$, and $\lesssim$ suppresses the constant factors and lower order terms.  The exact expression of the upper bound is in~\eqref{eq:main-theorem-proof-6}.   
\end{theorem}
\sloppy The proof of this result is in~\Cref{proof:finite-sample}. Note that the Pinsker-type inequality $TV(P, Q) \leq c_f \sqrt{D_f(P \parallel Q)}$ is satisfied by several divergences, such as relative entropy, Jensen Shannon, Hellinger, and chi-squared divergence. 

\begin{remark}
    \label[remark]{remark:comparison-with-HM2023} 
    The result of~\Cref{theorem:finite-sample} is obtained by combining two ingredients: (i) using the Pinsker-type inequality, we convert the control over the given $f$-divergence into control over total variation, and (ii) a truncation argument~(\Cref{lemma:truncation}) that converts the total variation into a bound on $W_1$ under just finite $(1+a)$ moment condition. 
    
    One  crucial advantage of our result over \citet[Theorem 2]{hagemann2021stabilizing} is that we do not impose the bounded support condition on the random variables $(X, Y, Z)$. The mild $(1+a)$ moment requirement significantly broadens the applicability of our result, and in particular, allows us to consider more realistic models used in practice. 
    Additionally, the dependence on the parameter $a$ quantifies how the heavy tails affect the approximation quality. We empirically verify this insight in~\Cref{sec:observation-6}.

\end{remark}
\begin{remark}
        Although~Theorem~\ref{theorem:finite-sample} characterizes the distance between the true~(and unknown) posterior and the \ac{INN} posterior in terms of $W_1$ metric,  we can use existing inequalities to translate it into other distances. For example, by~\citet[Theorem 21]{sriperumbudur2010hilbert}, we know that $\mathrm{MMD}(\PA_{X}, \PA_{\That_n^{-1}(Y,Z)}) \lesssim W_1(\PA_{X}, \PA_{\That_n^{-1}(Y,Z)})$ for commonly used kernels, such as the Gaussian and Mat\'ern  kernels, and thus~\Cref{theorem:finite-sample} also implies a bound on the kernel-MMD distance between the true and estimated posteriors. 
\end{remark}

\begin{remark}
    \label[remark]{remark:main-theorem-W1} A closer look at the proof of~Theorem~\ref{theorem:finite-sample} reveals several reasons why $W_1$ metric is the appropriate choice to present the result: (i) By Kantorovich-Rubinstein duality, the $W_1$ metric is defined via Lipschitz critic functions. This fact coupled with the assumed bi-Lipschitz  structure~\textbf{(INN1)}  allows us to transfer bounds from posterior distributions to the forward distributions. (ii) Since $W_1$ is a metric,  it satisfies the  triangle inequality which leads to a natural decomposition in the proof into terms that correspond directly to the supervised and unsupervised training losses. (iii) The Lipschitz formulation of $W_1$ is exactly what makes our truncation argument(\Cref{lemma:truncation}) effective for handling heavy tailed distributions under a mild $(1+a)$ moment assumption.
\end{remark}
\begin{remark}
    \label[remark]{remark:inn-moment-bound} While stating~\Cref{theorem:finite-sample}, we only needed to explicitly place a $(1+a)$ moment requirement, for some $a>0$. However, since $L_{\mathbf{y}}$ consists of squared losses, the existence of the uniform convergence rates in \textbf{(INN2)} implicitly places a stronger $(2+\beta)$ moment conditions, for some $\beta>0$. We will make this explicit when identifying verifiable sufficient conditions of these assumptions in~\Cref{prop:inn2-uniform-convergence}. 
\end{remark}

\subsubsection{Verifiable Sufficient Conditions for~\Cref{assump:INN-theorem}}
\label{subsec:INN-interpretable-conditions} 
Theorem~\ref{theorem:finite-sample} gives us a general result that translates uniform convergence and the variational approximation gap guarantees on the two losses, $L_{\mathbf{y}}(\That_n)$ and $L_{\mathbf{z}}(\That_n)$ into a bound on the $W_1$ metric between the true posterior, and the \ac{INN} posterior.  
We now present more easily verifiable sufficient conditions for the assumptions \textbf{(INN2)-(INN4)} required by~\Cref{theorem:finite-sample}, working under the realizability part of the  assumption \textbf{(INN1)}. 
We begin with the usual characterization of uniform convergence in terms of the Rademacher complexities~(see~\Cref{def:rademacher-complexity} in~\Cref{appendix:background}) of the associated function classes. 
\begin{proposition}
    \label[proposition]{prop:inn2-uniform-convergence}
    Assume that the critic class $\calG_n$ consists of uniformly bounded functions; that is, $\sup_{g \in \calG_n} \|g \|_{\infty} \leq b_n$, for some $b_n<\infty$. 
    Let $\calF_{\yy} = \{x \mapsto \langle u, T_{\mathbf{y}}(x)\rangle: T \in \calT, \, \|u\|\leq 1\}$, and define its Rademacher complexity as 
    \begin{align}
        \RadComp_n(\calF_{\yy}) = \mathbb{E}_{\epsilon^n, X^n} \lb \sup_{f \in \calF_{\yy}} \sum_{i=1}^n \epsilon_i f(X_i) \rb, \qtext{with} \epsilon^n \stackrel{\iid}{\sim} \mathrm{Rademacher}(\{-1, +1\}). 
    \end{align}
    Introduce the two critic class complexities 
    \begin{align}
        \RadComp_n^{(1)}(\calG_n) = \mathbb{E}_{\epsilon^n, Y^n, Z^n} \lb \sup_{g \in \calG_n} \frac{1}{n} \sum_{i=1}^n \epsilon_i g(Y_i, Z_i) \rb, \quad \RadComp_n^{(2)}(\calG_n, \calT) =\mathbb{E}_{\epsilon^n, X^n, Y^n} \lb \sup_{T, g} \frac{1}{n} \sum_{i=1}^n \epsilon_i g(Y_i, T_{\mathbf{z}}(X_i)) \rb. 
    \end{align}
    Suppose $\max\{\EE[\|Y\|^{2+\beta}], \,\sup_{T \in \calT} \mathbb{E}[\|T_{\mathbf{y}}(X)\|^{2+\beta}] \} \leq R_{\yy} < \infty$ for some $\beta>0$, and $\Pi_{K_n}:\R^{\dy} \to \{y \in \R^{\dy}: \|y\|\leq K_n\}$ denoting the projection on a ball of radius $K_n$ in $\calY = \R^{\dy}$,  define the projected empirical and population supervised losses: 
    \begin{align}
        \LhatY^{K_n}(T) &= \frac{1}{n} \sum_{i=1}^n \|\Pi_{K_n}(T_{\mathbf{y}}(X_i)) - \Pi_{K_n}(Y_i)\|^2, \qtext{and}
        L_{\mathbf{y}}^{K_n}(T) =  \mathbb{E}\lb \|\Pi_{K_n}(T_{\mathbf{y}}(X)) - \Pi_{K_n}(Y)\|^2  \rb. 
    \end{align}
    Assume that the function $f$~(in fact, its convex conjugate $f^*$) satisfies
    \begin{align}
        A_{1,n} \coloneqq \sup_{|u| \leq b_n} |(f^*)'(u)| < \infty, \qtext{and} 
        A_{2,n} \coloneqq  \sup_{|u| \leq b_n} |f^*(u)| < \infty.  \label{eq:A1-A2-def}
    \end{align}
    Then, the following two conditions hold with probability at least $(1-\delta)$, for a given $\delta\ in (0,1)$: 
    \begin{align}
            \sup_{T \in \calT} |\LhatZ(T) - L_{\mathbf{z}}(T)| &\;\lesssim\;  \RadComp_n^{(1)} +  A_{1,n} \RadComp_n^{(2)} + (b_n + A_{2,n})n^{-1/2},  
        \label{eq:inn2-prop-Lz} \\
        \sup_{T \in \calT} |\LhatY(T) - L_{\mathbf{y}}(T)| &\;\lesssim\;  K_n \RadComp_n(\calF_{\yy}) +  K_n^2 n^{-1/2} + {4 R_{\yy}}{K_n^{-\beta}}, \label{eq:inn2-prop-Ly} 
   \end{align}
    where $\lesssim$ suppresses constants that depend on $\delta$ and $R_{\yy}$.
    Thus, a sufficient condition to satisfy~\textbf{(INN2)} with vanishing sequences of $\{r_n, u_n: n \geq 1\}$, is to select an appropriate sequence of $K_n, b_n$, and $\calG_n$ to drive both these terms to $0$. 
\end{proposition}
The proof of~\eqref{eq:inn2-prop-Lz} relies on some standard arguments from empirical process theory~\citep{shorack2009empirical}, while the justification of~\eqref{eq:inn2-prop-Ly} requires a combination of symmetrization and vector contraction along with a truncation argument, and we present the details in~\Cref{proof:inn2-uniform-convergence}.

Next, we present a simple result stating that the uniform moment bound required in~\textbf{(INN3)} can be satisfied if the model class $\calT$ is uniformly Lipschitz. As we will observe in~\Cref{subsec:INN-example}, this condition holds for an important class of \acp{INN}. 
\begin{proposition}
    \label[proposition]{prop:inn3-uniform-moment-bound}
    Suppose there exist constants $J_0, J < \infty$, such that with $\|\cdot\|$ denoting the $\ell_2$ norm, we have $\|T(x)\| \leq J_0 + J\|x\|$, for all $x \in \calX$, and for all $T \in \calT$. 
    Then, assuming that $\mathbb{E}[\|X\|^{1+a}] < \infty$ for some $a>0$,  the following uniform moment bounds hold under the realizability assumption: 
    \begin{align}
        &\max \lbr \mathbb{E}\lb \|(Y,Z)\|^{1+a} \rb, \, \sup_{T\in \calT} \mathbb{E}\lb \|(Y, T_{\mathbf{z}}(X)\|^{1+a} \rb \rbr \leq R,   \qtext{with} 
        R = 2^{2a+1} \lp J_0^{1+a} + J^{1+a}\mathbb{E}[\|X\|^{1+a}] \rp. 
    \end{align}
\end{proposition}
\begin{proof}
   Since we are working under the realizability assumption, we have $(Y, Z) \stackrel{d}{=} T^*(X)$.  Using the fact that $(x+y)^p \leq 2^{p-1}(x^p + y^p)$, with $p=1+a$, and the condition that $\|T(x)\| \leq J_0 + J\|x\|$, we have $\mathbb{E}\lb \|(Y,Z)\|^p \rb \leq 2^{p-1}\lp J_0^p + J^p \mathbb{E}[\|X\|^p] \rp$. 
   Next, we look at the term $(Y, T_{\mathbf{z}}(X))$ and observe that 
   \begin{align}
       &\|(Y, T_{\mathbf{z}}(X))\|^p \leq 2^{p-1} \lp \|Y\|^p + \|T_{\mathbf{z}}(X)\|^p \rp  \leq 2^{p-1} \lp \|(Y, Z)\|^p + \|T(X)\|^p \rp  
       \end{align}
       which implies that 
       \begin{align}
 \mathbb{E}\lb \|(Y, T_{\mathbf{z}}(X))\|^p \rb &\leq 2^{p-1}\mathbb{E}[ \|(Y, Z)\|^p]  + 2^{2(p-1)} \lp J_0^p + J^p \mathbb{E}[\|X\|^p] \rp \\ 
&\leq 2 \times 2^{2(p-1)}\lp J_0^p + J^p \mathbb{E}[\|X\|^p] \rp.  
   \end{align}
   Taking the maximum of the two bounds gives us the required expression for $R$. 
\end{proof}

\begin{proposition}
    \label[proposition]{prop:inn4-variational-gap} 
    Introduce the notation $P \equiv P_{YZ}$ and $Q_T \equiv P_{Y, T_{\mathbf{z}}(X)}$, and assume that for all $T \in \calT$, we have  $P \ll Q_T$. For each $T \in \calT$, let $g^*_T$ denote a maximizer in the definition of the population $L_{\mathbf{z}}(T)$ in~\eqref{eq:def-population-Lz}, and   let $\bar{g}_T = \clip(g^*_T, -b_n, b_n)$ denote its clipped version for $b_n > 0$. For some $K_n > 0$, let $B_{K_n}$ denote the ball $\{u \in \mathbb{R}^{\dy + \dz}: \|u\| \leq K_n\}$, and define the following terms: 
    \begin{itemize}
        \item An approximation error term $\delta_n \equiv \delta_n(b_n, K_n)$ defined as  
        \begin{align}
            \delta_n \equiv \delta_n(b_n, K_n) \coloneqq \sup_{T \in \calT} \inf_{g \in \calG_n} \lbr \mathbb{E}_P[|\bar{g}_T - g| \boldsymbol{1}_{B_{K_n}}]  +  \mathbb{E}_{Q_T}[|\bar{g}_T - g| \boldsymbol{1}_{B_{K_n}}]\rbr. \label{eq:inn-approximation-term}
        \end{align}
        \item The tightness term $\tau_1 \equiv \tau_1(K_n)$,  and the clipping error term $\tau_2 \equiv \tau_2(b_n)$, as 
        \begin{align}
            &\tau_1(K_n) = \sup_{T \in \calT} P(B_{K_n}^c)+Q_T(B_{K_n}^c), \\ 
            &\tau_2(b_n) = \sup_{T \in \calT} \lbr \mathbb{E}_P[(|g_T^*| - b_n)^+  ] + A_{1,n} \mathbb{E}_{Q_T} \lb (|g_T^*| - b_n)^+ \rb \rbr, \label{eq:inn-tightness-terms}
        \end{align}
        where we use $(a)^+$ to denote $\max\{0, a\}$ for any $a \in \R$. 
    \end{itemize}
    Suppose there exists a sequence of $\{K_n, b_n: n \geq 1\}$  such that 
    \begin{align}
        \lim_{n \to \infty} \lbr   \tau_2 + A_{1,n} \delta_n + (b_n + A_{2,n}) \, \tau_1  \rbr = 0,  \label{eq:inn-variational-gap-sufficient-condition}
    \end{align}
    where $A_{1,n}$ and $A_{2,n}$ were defined in~\eqref{eq:A1-A2-def}. 
    Then the variational gap $\eta_n$ also converges to zero; that is, 
    \begin{align}
    \text{if \eqref{eq:inn-variational-gap-sufficient-condition} is true,}\quad \implies \quad 
         \eta_n &= \sup_{T \in \calT} \lbr D_f(P_X \parallel Q_T) - L_{\mathbf{z}}(T) \rbr  \;\stackrel{n \to \infty}{\longrightarrow}\;  0. 
    \end{align}   
\end{proposition}
The proof of this result is in~\Cref{proof:inn4-variational-gap}. The structure of~\eqref{eq:inn-variational-gap-sufficient-condition} is worth further discussion. The term $\delta_n$ is a purely approximation property of the critic class $\calG_n$, and it enforces the requirement that every optimal critic function can be  approximated in $L^1$ by a bounded element in $\calG_n$ with vanishing error. The terms $\tau_1$ and $\tau_2$ can be interpreted as tail approximation terms: $\tau_1$ is controlled by the moment bounds from~\textbf{(INN3)}, while $\tau_2$ depends on how quickly the tails of the optimal critic function decays with the clipping level $b_n$. 

To summarize, the results of this section provide a concrete recipe to apply the more abstract~\Cref{theorem:finite-sample} to a particular instantiation of the variational \ac{INN} pipeline. In particular, for a chosen model and critic class pair, it suffices to show that the Rademacher complexities of the associated function classes decay sufficiently quickly to $0$~(\Cref{prop:inn2-uniform-convergence}), to establish a uniform bi-Lipschitz property, and through it the moment bounds~(\Cref{prop:inn3-uniform-moment-bound}), and show that the critic class and well-approximate the optimal witness function on high probability regions~(\Cref{prop:inn4-variational-gap}). In the next subsection, we illustrate these steps for a specific variational \ac{INN} architecture with Gaussian \ac{RKHS} critic class.

\subsubsection{INN Example}
\label{subsec:INN-example}

In this section, we instantiate our abstract assumptions and conditions from~\Cref{subsec:INN-interpretable-conditions} and~\Cref{assump:INN-theorem} with a concrete class of \ac{INN} models and critic function classes. The main goal of this section is demonstrate the the conditions derived in the previous sections are satisfied by a non-trivial, reasonably expressive, and practically relevant pair of  model and critic classes. In particular, we show that a residual \ac{INN} with bounded weights and biases, trained with a Jensen-Shannon based loss and a Gaussian RKHS critic class, fits into our framework developed in the previous sections.  
We begin with a formal description of the model class. 
\begin{definition}
    \label[definition]{def:inn-calT} 
    As before, we assume $\dx = \dy + \dz$, and for an integer $M \geq 1$ and $s \in (0, 1)$, define $T:\mathbb{R}^{\dx} \to \mathbb{R}^{\dx}$ as follows, with $I_{\dx}$ denoting the $\dx \times \dx$ identity matrix: 
    \begin{align}
        T = (I_{\dx} + F_M) \circ (I_{\dx} + F_{M-1})\circ \cdots \circ (I_{\dx} + F_1),
        \quad F_j(x) \coloneqq W_{j,2} \tanh(W_{j,1}x + b_{j,1}) + b_{j,2}, 
    \end{align}
    with the following constraints for all $j \in [M] = \{1, \ldots, M\}$: 
    \begin{align}
        \|W_{j,1} \|_{op} \leq s, \quad \|W_{j,2}\|_{op}\leq s, \quad \max \lbr \|b_{j,1}\|, \|b_{j,2}\| \rbr \leq B. 
    \end{align}
    For every $j \in [M]$, the matrices $W_{j,1}$ and $W_{j,2}$ have dimensions $d_j \times \dx$ and $\dx \times d_j$ respectively, and let $H \coloneqq \max_{j \in [M]} d_j$.  We split the output coordinates to write $T(x) = (T_{\mathbf{y}}(x), T_{\mathbf{z}}(x))$, with $T_{\mathbf{y}}(x) \in \calY = \mathbb{R}^{\dy}$ and $T_{\mathbf{z}}(x) \in \calZ = \mathbb{R}^{\dz}$. It is easy to verify that this function class is bi-Lipschitz and there exist constants $J_0, J$~(depending on $M, B, s, H$) such that for all $x \in \calX$, we have $\|T(x)\| \leq J_0 + J \|x\|$ and $\|T^{-1}(y,z)\| \leq J_0 + J\|(y,z)\|$. 
\end{definition}
This choice of the \ac{INN} model class $\calT$ above is driven by the need for a global bi-Lipschitz bounds on both the forward and inverse maps. The residual architecture with $\tanh$ activations and the norm constraints on the weight and bias terms provide a convenient way of controlling global Lipschitz constants. In contrast the popular coupling-based architectures such as RealNVP of~\citet{dinh2014nice} generally do not admit such global control of Lipschitz constants.  We now present the details of our critic class construction. 
\begin{definition}
\label[definition]{def:inn-calG}
    Let $k_n \equiv k_{\gamma_n}$ denote the Gaussian kernel on $\mathbb{R}^{\dy+\dz}$, 
    \begin{align}
        k_n((y,z), (y',z')) = \exp \lp - \gamma_n \|(y,z) - (y',z')\|^2 \rp    = \exp \lp - \gamma_n \lp \|y-y'\|^2 + \|z-z'\|^2 \rp \rp, 
    \end{align}
    and let $\calH_{k_n}$ denote its  RKHS. Fix a radius parameter $K_n>0$ and a uniform upper bound $b_n>0$ and define the critic class as 
    \begin{align}
        \calG_n = \lbr g = h \boldsymbol{1}_{C_{K_n}} \;:\; h \in \calH_{k_n}, \; \|h\|_{k_n} \leq b_n\rbr, \qtext{with} C_{K_n} = [-K_n, K_n]^{\dx}.  
    \end{align}
    Note that $\boldsymbol{0} \in \calG_n$, and each $g \in \calG_n$ also satisfies $\|g\|_{\infty} \leq b_n$ due to the reproducing property and the fact that $\sup_x \sqrt{k_n(x,x)} = 1$.  
\end{definition}
This particular choice of the critic class $\calG_n$ is chosen as a compromise between analytical tractability and representation power. The \ac{RKHS} associated with Gaussian kernels admit clean bounds on their Rademacher complexity, and  are also known to be \emph{universal} in the sense that they are approximate any continuous bounded function over compact domains arbitrarily well in $\sup$ norm. These are the exact properties that we need in~\Cref{prop:inn2-uniform-convergence} to ensure uniform convergence, and in~\Cref{prop:inn4-variational-gap} for controlling the variational approximation gap. The additional truncation on cube $C_{K_n}$ by the indicator function is a purely technical choice: it ensures the uniform boundedness of the critic class while also decoupling the approximation on $C_{K_n}$ from tail events, which can be handled separately via moment bounds. Furthermore, since the $\ell_2$-ball $B_{K_n}$ used by~\Cref{prop:inn4-variational-gap} is contained in $C_{K_n}$, any approximation bound $C_{K_n}$ is also valid on $B_{K_n}$.  We now present the choice of the $f$-divergence to be used for the unsupervised loss.
\begin{definition}
\label[definition]{def:jensen-shannon}    
We will employ a  Jensen-Shannon divergence based unsupervised loss in our \ac{INN} training. This corresponds to $f(t) = x \log \tfrac{2x}{x+1} + \log \tfrac{2}{x+1}$,  and $f^*(t) = \log(1+e^t) - \log 2$ with 
\begin{align}
    A_{1,n} = \sup_{|t| \leq b_n} | (f^*)'(t)| \leq 1, \qtext{and} A_{2,n} = \sup_{|t| \leq b_n} |f^*(t)| \leq \log((1+e^{b_n})/2) \leq b_n.  
\end{align}
\end{definition}
Having introduced all the main components, we can now present the main result of this section. 
\begin{theorem}
    \label{theorem:inn-example}
    Let $(\calT, \calG_n)$ denote a pair of model and critic classes as introduced in~\Cref{def:inn-calT} and~\Cref{def:inn-calG} respectively, and define the parameters 
    \begin{align}
        K_n &= M (s \sqrt{H} + B) + \sqrt{\dx} + \sqrt{2 \log n}, \\
        \gamma_n &= K_n^{2 + \epsilon} \quad \text{for some fixed } \epsilon>0, \\
        b_n &= C_b \gamma_n^{\dx/4} K_n^{2 + \dx/2} = C_b K_n^{2 + \dx + \frac{\dx}{4} \epsilon}.
    \end{align}
    Then, under the realizability assumption \textbf{(INN1)}, and with these choices of the parameters, the ERM model~$\That_n$ defined in~\eqref{eq:erm-INN} satisfies 
    \begin{align}
        W_1 \lp P_{X \mid Y \in A}, \,  P_{\hat{T}^{-1}_n(Y, Z)|Y \in A}\rp = o(1), \quad \text{w.p. at least } 1-\delta.  \label{eq:inn-theorem-main-result}
    \end{align}
    In other words, with these parameters, the models introduced above satisfy the sufficient conditions for \textbf{(INN2)}-\textbf{(INN4)} to hold, as derived in~\Cref{subsec:INN-interpretable-conditions}. 
\end{theorem}
The proof of this result is in~\Cref{proof:inn-example}.

As mentioned earlier, the purpose of this theorem is to illustrate that our abstract theoretical results are applicable to a non-trivial, reasonably realistic variational \ac{INN} pipeline. In particular,~\Cref{theorem:inn-example} demonstrates that one can work with architectures that are expressive enough to model high-dimensional data and still verify the technical assumptions required for our posterior accuracy guarantees. Extending our arguments to other classes of invertible architectures, such as coupling-based flows equipped with appropriate spectral  or Jacobian regularization, is an interesting direction for future work.

\section{Empirical Results}
\label{sec:case-studies}
In this section, we first present a series of small-scale probing experiments to understand the effects of various design choices involved in training invertible models~(\S~\ref{subsec:observations}), and then apply the insights gained to a practically relevant ocean-acoustic inversion problem~(\S~\ref{subsec:geoacoustic_inversion}).  
\noindent Our objective in~\Cref{subsec:observations} is two-fold:
\begin{itemize}
    \item The first objective is to provide  empirical support for several theoretical claims derived from the results of the previous section on \acp{INN} and from \Cref{subsec:variational-nf} on \acp{NF}. In particular, we investigate the use of the variational \ac{USL} framework \Cref{sec:general-framework} for training invertible architectures, and we compare the performance of different $f$-divergences with that of \acp{IPM}. We further analyze the behavior of the forward and backward training losses  when variational $f$-divergences are employed. Additionally, we empirically examine the main theoretical result for \acp{NF} presented in \Cref{subsec:variational-nf} by studying the relationship between the number of finite moments of  $X$ and the $W_1$ distance between the true and the estimated distributions. The observed polynomial convergence rates are consistent with our theoretical predictions and analogous to the behavior is established for \acp{INN} in \Cref{theorem:finite-sample}.
    \item The other direction is empirically explore the effects of certain design choices in practical performance of \acp{INN}. This includes the prior loss or reconstruction loss, the effect of dimensionality of the latent space, and role of the support set, and the effect of entropy regularization in Wasserstein metric.  This empirical analysis complements the existing literature and providing insights for applied implementations.
\end{itemize}
The software accompanying this paper is available on GitHub,
\href{https://github.com/ananya-ac/INN-Project}{https://github.com/ananya-ac/INN-Project}.

\subsection{Effect of design choices on \ac{INN} training}
\label{subsec:observations}
As mentioned earlier, in this section, we study the effects of various design choices (such as prior distribution, latent dimension, entropic regularization parameter, etc.) on the training of invertible neural networks.
Before presenting the empirical studies in detail, we first clarify two general aspects. First, unless otherwise stated, all case studies were trained for $10$ epochs using a batch size of $512$. The model parameters were optimized with the Adam optimizer with a learning rate of $10^{-4}$ and $\ell_2$ regularization. Second, since the  \ac{NLL}  is a widely adopted loss function for training\acp{INN}~\citep{dinh2016density,ren2020benchmarking, kruse2021benchmarking},   we incorporate the \ac{NLL} in our case studies, complementing the USL described in~\Cref{sec:general-framework} (the relevant definitions are reviewed in~\Cref{sec:NLL}).
Also, as discussed in \Cref{sec:related-work}, while the Gaussian assumption underlying the computation of the \ac{NLL} can be restrictive, we avoid such distributional assumptions by leveraging the variational formulation of f-divergences.

\subsubsection{Effect of prior on training}
\label{sec:observation-1-prior}
In this case study, we use the \ac{IK} framework described next. 

\noindent\textbf{Inverse kinematics example:}
\Ac{IK} is a core problem in robotics, where the objective is to determine the joint configurations required to achieve a desired end-effector position in space~\citep{niku2020introduction}. Efficiently solving this problem is essential for tasks like motion control, path planning, and real-time manipulation. Traditional methods such as analytical solutions and numerical optimization often face challenges as the robot configurations become complicated and number of degrees of freedom (DOF) increases. 
Machine learning algorithms are increasingly being employed for solving \ac{IK} problems~\citep{toquica2021analytical}, with the 4 DOF robot example becoming a standard benchmark in analyzing the analyzing the practical aspect of training invertible architecture
\ac{INN}~\citep{ardizzone2018analyzing,kruse2021benchmarking,hagemann2021stabilizing}. 

We consider an articulated robotic arm that moves vertically along a rail and rotates at three joints. 
These four degrees of freedom constitute the parameter vector 
\(X = [X^{(1)}, X^{(2)}, X^{(3)}, X^{(4)}]^\top\).
The dataset is generated using Gaussian priors defined as 
$X^{(i)} \sim \mathcal{N}(0, \sigma_i)$,
where the standard deviations are specified as 
\(\sigma_1 = 0.25\) and \(\sigma_2 = \sigma_3 = \sigma_4 = 0.5\), and 
the forward kinematic process of the robotics arm is modeled as
\begin{align}
Y^{(1)} &= X^{(1)} + l_1 \sin(X^{(2)}) + l_2 \sin(X^{(3)} - X^{(2)}) + l_3 \sin(X^{(4)} - X^{(2)} - X^{(3)}); \label{eq:y1}\\
Y^{(2)} &= l_1 \cos(X^{(2)}) + l_2 \cos(X^{(3)} - X^{(2)}) + l_3 \cos(X^{(4)} - X^{(2)} - X^{(3)}), \label{eq:y2}
\end{align}
where the arm segment lengths are given by \(l_1 = 0.5\), \(l_2 = 0.5\), and \(l_3 = 1.0\).

 We trained an \ac{INN} to model the relationship between the joint angles of a robotic arm and a point in its 2D workspace.  All models were trained using the \ac{NLL} loss formulation for the INN. The input consisted of 4 variables corresponding to the joint angles of the robotic arm, and the output was a 2-dimensional vector representing a point in the 2D workspace.

For this case study, we use two different regularization priors: a Gaussian prior and a uniform prior.
The Gaussian prior represents the correct prior (Normal prior with $\mu = [0,0,0,0]$  and $\sigma = [0.25, 0.5, 0.5, 0.5]$) as the data-generation process follows the same Gaussian distribution. The corresponding prior  loss can be computed using the negative log-likelihood, up to a  constant, as follows  
\begin{align}
\label{eq:prior_new_Guass}
\prescript{p}{}{L_x} = %
\frac{1}{n} \sum_{i=1}^{n} (T^{-1} (Y_i,Z_i) - \tilde{X}_{i})^2 \,.
\end{align}
Here,  $\{\tilde{X}_{i}\}_{i=1}^N$ is the data generated from the input.
A uniform prior~\citep{andrle2021invertible} also represents an inconsistent with the true distribution.  If the input variable ranges from $a$ to $b$ ($b>a$), then the loss term is as follows:
 \begin{align}
 \prescript{p}{}{L_x} = \frac{1}{n} \sum_{i=1}^{n} \left( \max(0, T^{-1} (Y_i,Z_i) - b) + \max(0, a - T^{-1} (Y_i,Z_i)) \right).
\end{align}
Here, $a=0$ and $b=1$. 

\begin{figure}[!htbp]
    \centering

    \begin{subfigure}[t]{0.24\textwidth}
        \centering
        \includegraphics[width=\textwidth]{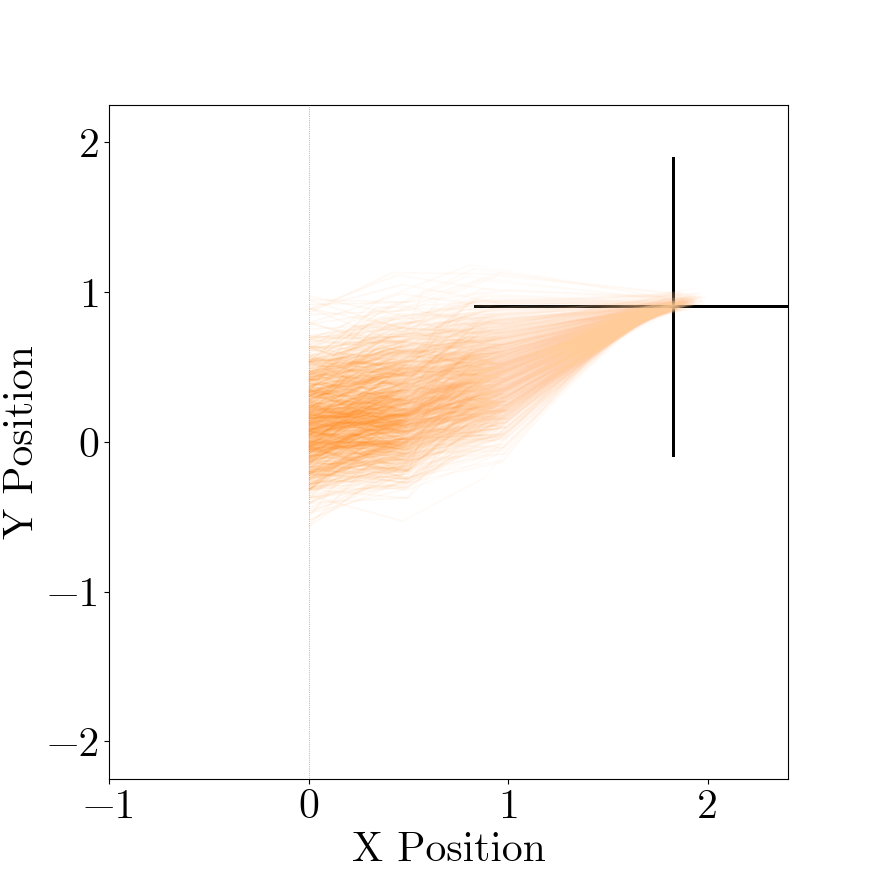}
        \caption{No prior}
        \label{noprior}
    \end{subfigure}
    \begin{subfigure}[t]{0.24\textwidth}
        \centering
        \includegraphics[width=\textwidth]{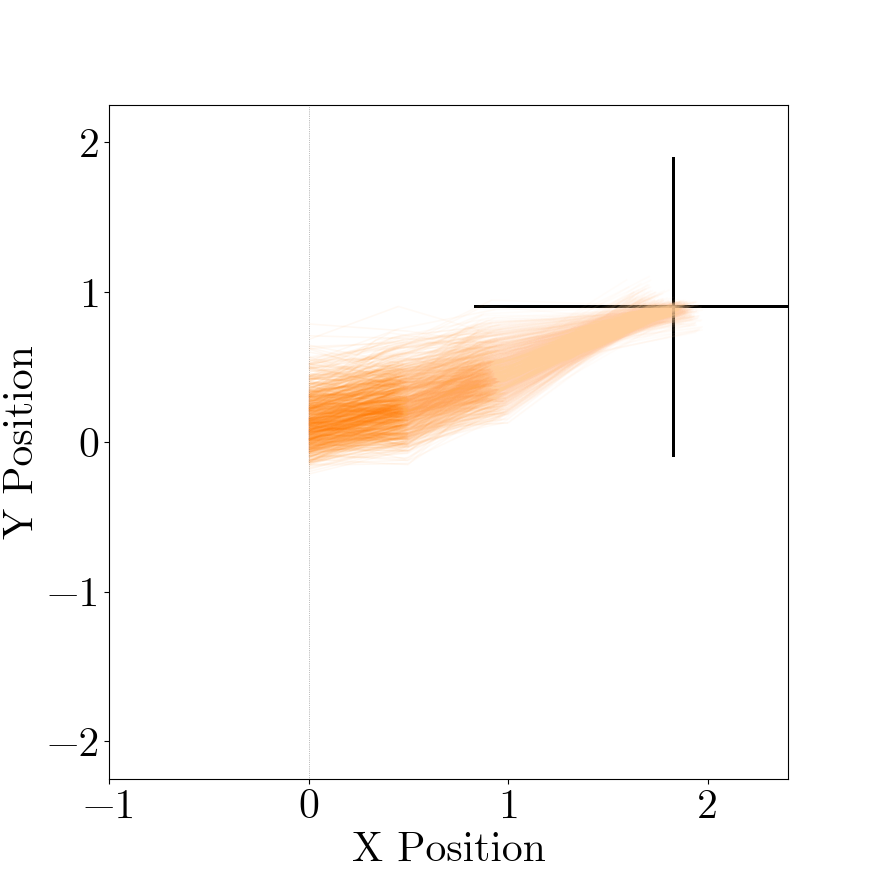}
        \caption{Correct Prior,\\ $\lambda'=1$}
        \label{subfigure:gprior}
        
    \end{subfigure}
    \begin{subfigure}[t]{0.24\textwidth}
        \centering
        \includegraphics[width=\textwidth]{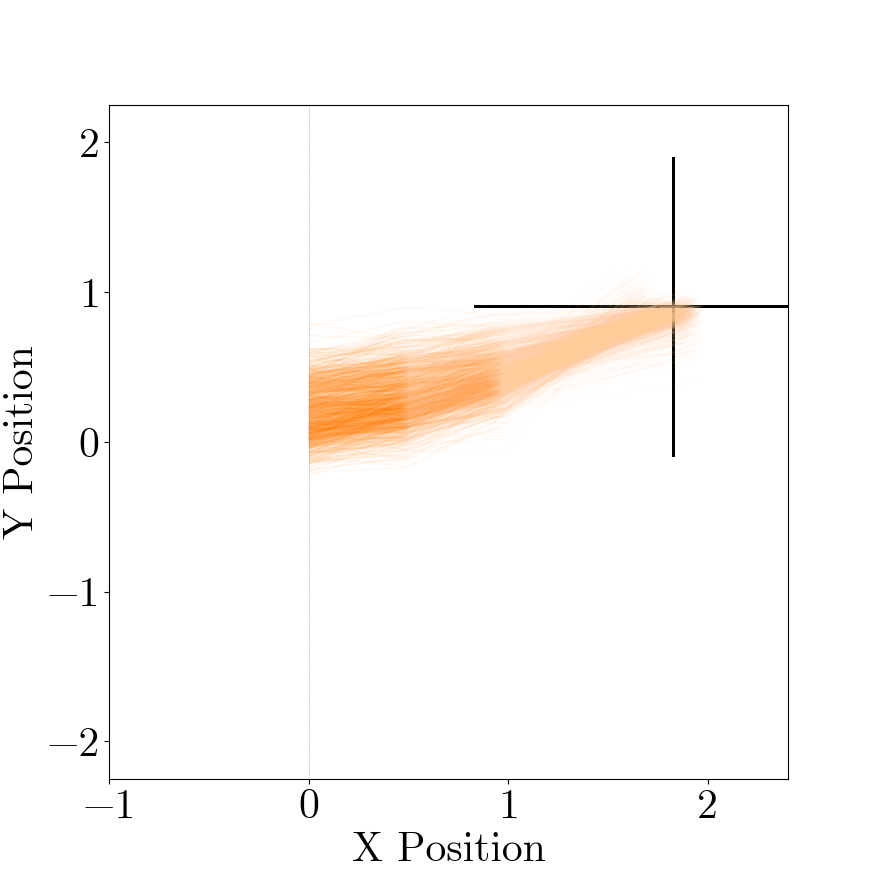}
        \caption{\centering Misspecified Prior,\\ $\lambda'=1$}
        \label{subfigure:bprior}
        
    \end{subfigure}
    \begin{subfigure}[t]{0.24\textwidth}
        \centering
        \includegraphics[width=\textwidth]{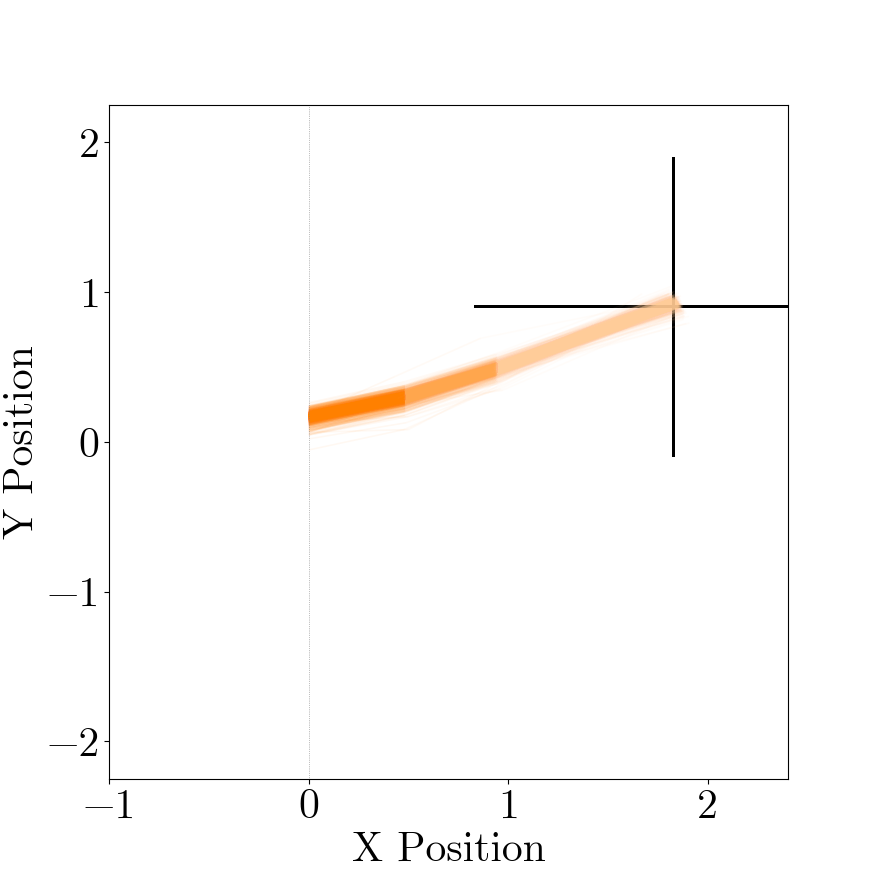}
        \caption{Correct Prior,\\ $\lambda' =100$}
        \label{subfigure:maxprior}
        
    \end{subfigure}

    \caption{ 
    The four plots illustrate the results of the inverse kinematics case study with different prior regularizations.
    (a) shows samples from a model trained without any prior loss (i.e., with prior weight $\lambda' = 0$). (b) and (c) display samples from models trained with a prior weight of $\lambda' = 1$. In (b), the model assumes a Gaussian prior over the joint configuration, resulting in a scaled $L_2$ loss between the ground truth and the generated samples. In contrast, (c) uses a uniform prior $\mathcal{U}(0,1)$, which is inconsistent with the true distribution. (d) illustrates results from a model trained with a Gaussian prior, using $\lambda' = 100$. The cross indicates the true end-effector position. All models were trained using the \ac{NLL} loss formulation for the INN. 
    }
    \label{fig:posterior_comparison}
\end{figure}
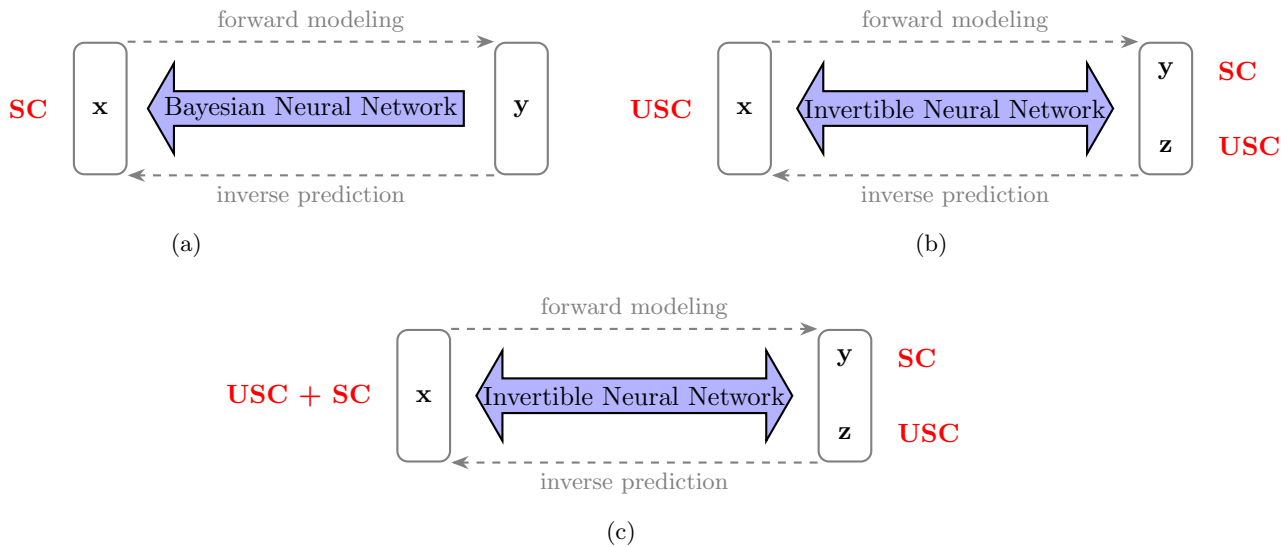
\begin{figure}[t]
    \centering
    \begin{minipage}{0.3\textwidth}
        \centering
        \begin{subfigure}[t]{\textwidth}
            \begin{tikzpicture}[
    scale=1, transform shape,
    box/.style={draw, rectangle, rounded corners, minimum width=0.7cm, minimum height=1.75cm, thick, gray},
    dashedarrow/.style={thick, dashed, gray, -Stealth},
    networkarrow/.style={single arrow, fill=blue!30, draw=black, thick, minimum height=4.2cm, minimum width=1.2cm, single arrow head extend=0.21cm, single arrow tip angle=120},
    label/.style={font=\small},
]

\node[box] (left) at (0,0) {};
\node[left=0.21cm of left, red, font=\bfseries] {SC};
\node at (left.center) {\bfseries x};

\node[box] (right) at (5.6,0) {};
\node at (right.center) {\bfseries y};

\node[networkarrow, rotate=180] (net) at (2.8,0) {};
\node[rotate=0] at (net.center) {Bayesian Neural Network};

\draw[dashedarrow] (left.north east) -- node[above, label] {forward modeling} (right.north west);

\draw[dashedarrow] (right.south west) -- node[below, label] {inverse prediction} (left.south east);

\end{tikzpicture}

            \caption{}
            \label{bnn}
        \end{subfigure}
    \end{minipage}
    \hfill
    \begin{minipage}{0.5\textwidth}
        \centering
        \begin{subfigure}[t]{\textwidth}
            \begin{tikzpicture}[
    scale=1, transform shape,
    box/.style={draw, rectangle, rounded corners, minimum width=0.7cm, minimum height=1.75cm, thick, gray},
    dashedarrow/.style={thick, dashed, gray, -Stealth},
    label/.style={font=\small},
]

\node[box] (left) at (0,0) {};
\node[left=0.21cm of left, red, font=\bfseries] {USC};
\node at (left.center) {\bfseries x};

\node[box, minimum height=1.75cm] (right) at (5.6,0) {};
\node at (right.center) [yshift=0.5cm] {\bfseries y};
\node at (right.center) [yshift=-0.5cm] {\bfseries z};
\node[right=0.21cm of right, red, font=\bfseries, yshift=0.5cm] {SC};
\node[right=0.21cm of right, red, font=\bfseries, yshift=-0.5cm] {USC};

\node[double arrow, fill=blue!30, draw=black, thick, minimum height=4.2cm, minimum width=1.2cm, 
      double arrow head extend=0.21cm, double arrow tip angle=120] (net) at (2.8,0) {};
\node at (net.center) {Invertible Neural Network};

\draw[dashedarrow] (left.north east) -- node[above, label] {forward modeling} (right.north west);

\draw[dashedarrow] (right.south west) -- node[below, label] {inverse prediction} (left.south east);

\end{tikzpicture}

            \caption{}
            \label{noprior_inn}
        \end{subfigure}
    \end{minipage}

    \vspace{1em} %

    \begin{minipage}{0.65\textwidth}
        \centering
        \begin{subfigure}[t]{\textwidth}
\begin{tikzpicture}[
    scale=1, transform shape,
    box/.style={draw, rectangle, rounded corners, minimum width=0.7cm, minimum height=1.75cm, thick, gray},
    dashedarrow/.style={thick, dashed, gray, -Stealth},
    label/.style={font=\small},
]

\node[box] (left) at (0,0) {};
\node[left=0.21cm of left, red, font=\bfseries] {USC + SC};
\node at (left.center) {\bfseries x};

\node[box, minimum height=1.75cm] (right) at (5.6,0) {};
\node at (right.center) [yshift=0.5cm] {\bfseries y};
\node at (right.center) [yshift=-0.5cm] {\bfseries z};
\node[right=0.21cm of right, red, font=\bfseries, yshift=0.5cm] {SC};
\node[right=0.21cm of right, red, font=\bfseries, yshift=-0.5cm] {USC};

\node[double arrow, fill=blue!30, draw=black, thick, minimum height=4.2cm, minimum width=1.2cm, 
      double arrow head extend=0.21cm, double arrow tip angle=120] (net) at (2.8,0) {};
\node at (net.center) {Invertible Neural Network};

\draw[dashedarrow] (left.north east) -- node[above, label] {forward modeling} (right.north west);

\draw[dashedarrow] (right.south west) -- node[below, label] {inverse prediction} (left.south east);

\end{tikzpicture}            \caption{}
            \label{yesprior_inn}
        \end{subfigure}
    \end{minipage}

    \caption{Schematic illustration of how different training paradigms are used to solve inverse problems in the Gaussian setting. (a) Bayesian neural network for solving inverse problem. There is supervised cost (SC) between predicted and true $\textbf{x}$. (b) INN Trained without $\prescript{p}{}{L_x}$. There is unsupervised cost (USC) between predicted and true $\textbf{x}$. (c) Shows an INN model trained with $\prescript{p}{}{L_x}$, which serves as a reconstruction loss.}
    \label{fig:INNBNN}
\end{figure}
To demonstrate the effect of incorporating a prior objective during training, the prior weight $\lambda'$ was set to values of $0$, $1$, and $100$.   Figure~\ref{fig:posterior_comparison} illustrates the effect of incorporating the prior loss $^pL_x$  into the training objective. We observe that using a suitable prior, along with a well-chosen prior weight $\lambda'$, leads to more plausible samples. In the context of the inverse kinematics case study, this corresponds to configurations that keep joint angles close to zero, resulting in ``straighter" arm trajectories, as shown in Figure~\ref{subfigure:gprior}. Since the prior had a mean of $0$ and a small variance, it led to samples that favored configurations where the arm positions were straight, with all joint angles close to $0$ degrees. However, when the assumed prior is incorrect, incorporating it into the training objective leads to 
inaccurate joint configurations, as seen in Figure~\ref{subfigure:bprior}. Finally, Figure~\ref{subfigure:maxprior} demonstrates that when the prior weight $\lambda'$ is set too high, even with a suitable prior, the prior loss dominates the training, resulting in degraded performance and unrealistic outputs.
 Increasing the coefficient for $\prescript{p}{}{L_x}$ highlights this effect further.

We conclude that using prior distribution knowledge in training \acp{INN} acts as a form of regularization by constraining the estimated posterior to align with known distributions. By incorporating such priors, the model is guided toward more plausible solutions, preventing it from learning overly complex or irrelevant patterns. Essentially, the prior acts as a form of bias that encourages the model to explore a more reasonable subset of possible solutions, leading to  robustness in generation tasks.
The designers need to navigate the fine line between Bayesian wisdom and risky assumptions. A well-chosen prior can improve the \ac{INN} performance by effectively regularizing the model, and poorly chosen priors can dominate the results, leading to biased or misleading conclusions.

We conclude this section by discussing the reason behind the importance of prior cost on training by focusing on the Gaussian case. Fig.~\ref{fig:INNBNN}  compares the training procedure of \ac{BNN}, \ac{INN}, and \ac{INN} with prior for solving an inverse problem. 
The standard direct formulation in \ac{BNN} depends on a discriminative supervised loss term on $X$, which is one motivation for introducing \ac{INN} for inverse problems~\citep{ardizzone2018analyzing}.
In the case of Gaussian distributions, we have demonstrated that $\prescript{p}{}{L_x}$ can serve as an effective regularizer by imposing a supervised cost on $X$. In this case, as demonstrated in Fig.~\ref{yesprior_inn},
training with prior propose a hybrid approach applying both supervised and unsupervised costs on $X$. This approach leverages  the regularization with prior to enhance the performance of \ac{INN}. 

 \subsubsection{Comparison of architectures}
\label{sec:observation-2-f-div}

This observation compares the performance of \acp{INN} on an inverse problem using coupling-based and iResNet architectures, trained with different variational f-divergences as loss functions, considering both forward and backward losses.
 In general, it is challenging to optimize over all measurable function in the  the critic  class $\calG$, defined in Section~\ref{sec:general-framework}. Hence, in practice, finite-parameter family of functions are used to approximate the optimal critic function. We adopt the method proposed in~\citet{nowozin2016f}, to approximate the critic function 
 by a neural network. In this case, we replace the function $g$ in equation~\eqref{eq:f-div} with 
$
V_\omega(x) = g_f(V'_\omega(x))$, 
where \( V'_\omega : \mathcal{X} \to \mathbb{R} \) denotes a neural network with parameters $\omega$, and \( g_f \) is a monotonic activation function (mapping into the domain of \( f^* \)).  Here, \( V'_\omega \) is implemented as a fully connected network with two hidden layers of size 64 with ReLU activations~\citep{goodfellow2016deep}.   The function $g_f$ is determined by the specific choice of the $f$-divergence~\citep{nowozin2016f}. 
For the Kullback--Leibler (KL) divergence, $g_f$ employs an identity activation, i.e., $g_f(v) = v$. In the case of the Reverse KL divergence, the activation is given by $g_f(v) = -\exp(-v)$, whereas for the Jensen--Shannon (JS) divergence, the activation takes the form $g_f(v) = \log(2) - \log(1 + \exp(-v))$.

In this numerical study, we consider the \ac{USL} in both the backward and forward directions. %
 In the backward direction, we wish to minimize $D_f(P_{T^{-1}(Y,Z)} \parallel P_X)$ using the following joint objective function
\begin{equation}
\mathcal{F}(\theta, \omega) = 
\mathbb{E}[V_\omega(T^{-1}_\theta(Y,Z)] -
\mathbb{E} [f^*(V_\omega(X))],
\label{eq:vdm-objective-backward}
\end{equation}
where $T^{-1}$ and $V$ are parametrized by $\theta$ and $\omega$, respectively. In particular, in this case study, we refer to the following regularized empirical estimation as the \textit{backward loss}
\begin{equation}
\mathcal{L}_{\text{backward}}(\theta, \omega) = \frac{1}{N}\sum_{i=1}^{N} V_\omega(T^{-1}_\theta(Y_i, Z_i)) -\frac{1}{N}\sum_{i=1}^{N} f^*( V_\omega(X_i)) + \frac{1}{N}\sum_{i=1}^{N} \|T^{-1}_\theta(Y_i, Z_i) - X_i\|^2
\label{eq:vdm-objective-backward-empirical}
\end{equation}
Similarly, in the forward direction, we wish to minimize $ D_f(P_{Y,Z} \parallel P_{T(X)} ) $, using the the objective function
\begin{equation}
\mathcal{F}(\theta, \omega) = \mathbb{E}[V_\omega(Y,Z)] -\mathbb{E} [f^*(V_\omega(T_\theta(X)))].
\label{eq:vdm-objective-forward}
\end{equation}
Here, the following empirical estimation is called the \textit{forward loss}
\begin{equation}
\mathcal{L}_{\text{forward}}(\theta, \omega) = \frac{1}{N}\sum_{i=1}^{N} V_\omega(Y_i, Z_i)-\frac{1}{N}\sum_{i=1}^{N} f^*(V_\omega(T_\theta(X_i)))  +  \frac{1}{N}\sum_{i=1}^{N} \|T_{\theta}(X_i) - (Y_i,Z_i)\|^2 \, .
\label{eq:vdm-objective-backward-empirical-2}
\end{equation}
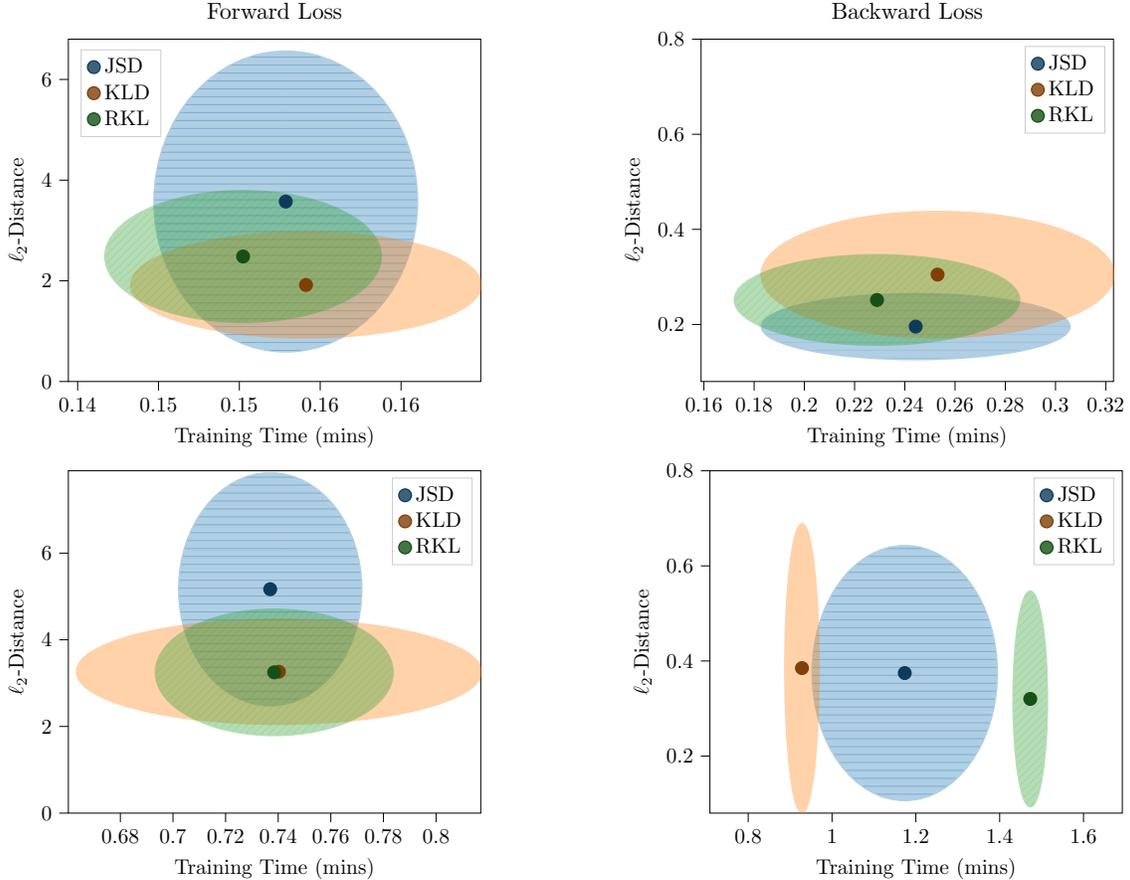
\begin{figure}[!tb]
    \centering
    \begin{subfigure}{0.49\linewidth}
        \centering
        \begin{tikzpicture}[scale=0.8]

\definecolor{darkgray176}{RGB}{176,176,176}
\definecolor{darkgreen228022}{RGB}{22,80,22}
\definecolor{darkorange25512714}{RGB}{255,127,14}
\definecolor{darkslategray155990}{RGB}{15,59,90}
\definecolor{forestgreen4416044}{RGB}{44,160,44}
\definecolor{lightgray204}{RGB}{204,204,204}
\definecolor{saddlebrown127637}{RGB}{127,63,7}
\definecolor{steelblue31119180}{RGB}{31,119,180}

\begin{axis}[
legend cell align={left},
legend style={
  fill opacity=0.8,
  draw opacity=1,
  text opacity=1,
  at={(0.03,0.97)},
  anchor=north west,
  draw=lightgray204
},
tick align=outside,
tick pos=left,
x grid style={darkgray176},
xlabel={Training Time (mins)},
xmin=0.139425517149703, xmax=0.164928278393846,
xtick style={color=black},
y grid style={darkgray176},
ylabel={$\ell_2$-Distance},
ymin=0, ymax=6.8,
ytick style={color=black},
title={\fontsize{11}{17}\selectfont Forward Loss}
]
\draw[draw=steelblue31119180,fill=steelblue31119180,opacity=0.35,very thick,postaction={pattern=horizontal lines, pattern color=steelblue31119180, fill opacity=0.35}] (axis cs:0.15286987281981,3.57493996052515) ellipse (0.00811683077174429 and 2.97583139794213);
\draw[draw=darkorange25512714,fill=darkorange25512714,opacity=0.35,very thick] (axis cs:0.154122326866029,1.9177303654807) ellipse (0.0108059515278173 and 1.04068247619129);
\draw[draw=forestgreen4416044,fill=forestgreen4416044,opacity=0.35,very thick,postaction={pattern=north east lines, pattern color=forestgreen4416044, fill opacity=0.35}] (axis cs:0.150231468677521,2.482901598726) ellipse (0.00850995884701292 and 1.2970025851613);
\addplot [semithick, darkslategray155990, mark=*, mark size=3, mark options={solid}, only marks]
table {%
0.15286987281981 3.57493996052515
};
\addlegendentry{JSD}
\addplot [semithick, saddlebrown127637, mark=*, mark size=3, mark options={solid}, only marks]
table {%
0.154122326866029 1.9177303654807
};
\addlegendentry{KLD}
\addplot [semithick, darkgreen228022, mark=*, mark size=3, mark options={solid}, only marks]
table {%
0.150231468677521 2.482901598726
};
\addlegendentry{RKL}
\end{axis}

\end{tikzpicture}
        \label{fig:fdiv_3}
    \end{subfigure}
    \hfill
    \begin{subfigure}{0.49\linewidth}
        \centering
        \begin{tikzpicture}[scale=0.8]

\definecolor{darkgray176}{RGB}{176,176,176}
\definecolor{darkgreen228022}{RGB}{22,80,22}
\definecolor{darkorange25512714}{RGB}{255,127,14}
\definecolor{darkslategray155990}{RGB}{15,59,90}
\definecolor{forestgreen4416044}{RGB}{44,160,44}
\definecolor{lightgray204}{RGB}{204,204,204}
\definecolor{saddlebrown127637}{RGB}{127,63,7}
\definecolor{steelblue31119180}{RGB}{31,119,180}

\begin{axis}[
legend cell align={left},
legend style={fill opacity=0.8, draw opacity=1, text opacity=1, draw=lightgray204},
tick align=outside,
tick pos=left,
x grid style={darkgray176},
xlabel={Training Time (mins)},
xmin=0.15888004486044, xmax=0.32309066612102,
xtick style={color=black},
y grid style={darkgray176},
ylabel={$\ell_2$-Distance},
ymin=0.08, ymax=0.8,
ytick style={color=black},
title={\fontsize{11}{17}\selectfont Backward Loss}
]
\draw[draw=steelblue31119180,fill=steelblue31119180,opacity=0.35,very thick,postaction={pattern=horizontal lines, pattern color=steelblue31119180, fill opacity=0.35}] (axis cs:0.2443229334695,0.195491413275401) ellipse (0.061237907462851 and 0.0683092807552644);
\draw[draw=darkorange25512714,fill=darkorange25512714,opacity=0.35,very thick] (axis cs:0.253044174209474,0.304880958582674) ellipse (0.0700464919115459 and 0.131798781676307);
\draw[draw=forestgreen4416044,fill=forestgreen4416044,opacity=0.35,very thick,postaction={pattern=north east lines, pattern color=forestgreen4416044, fill opacity=0.35}] (axis cs:0.228926536771986,0.251503958233765) ellipse (0.0565990604541264 and 0.0938564848066086);
\addplot [semithick, darkslategray155990, mark=*, mark size=3, mark options={solid}, only marks]
table {%
0.2443229334695 0.195491413275401
};
\addlegendentry{JSD}
\addplot [semithick, saddlebrown127637, mark=*, mark size=3, mark options={solid}, only marks]
table {%
0.253044174209474 0.304880958582674
};
\addlegendentry{KLD}
\addplot [semithick, darkgreen228022, mark=*, mark size=3, mark options={solid}, only marks]
table {%
0.228926536771986 0.251503958233765
};
\addlegendentry{RKL}
\end{axis}

\end{tikzpicture}
        \label{fig:fdiv_4}
    \end{subfigure}
    \vspace{1em}
    \begin{subfigure}{0.49\linewidth}
        \centering
        \begin{tikzpicture}[scale=0.8]

\definecolor{darkgray176}{RGB}{176,176,176}
\definecolor{darkgreen228022}{RGB}{22,80,22}
\definecolor{darkorange25512714}{RGB}{255,127,14}
\definecolor{darkslategray155990}{RGB}{15,59,90}
\definecolor{forestgreen4416044}{RGB}{44,160,44}
\definecolor{lightgray204}{RGB}{204,204,204}
\definecolor{saddlebrown127637}{RGB}{127,63,7}
\definecolor{steelblue31119180}{RGB}{31,119,180}

\begin{axis}[
legend cell align={left},
legend style={fill opacity=0.8, draw opacity=1, text opacity=1, draw=lightgray204},
tick align=outside,
tick pos=left,
x grid style={darkgray176},
xlabel={Training Time (mins)},
xmin=0.660229012596322, xmax=0.816990059687569,
xtick style={color=black},
y grid style={darkgray176},
ylabel={$\ell_2$-Distance},
ymin=0, ymax=7.9,
ytick style={color=black}
]
\draw[draw=steelblue31119180,fill=steelblue31119180,opacity=0.35,very thick,postaction={pattern=horizontal lines, pattern color=steelblue31119180, fill opacity=0.35}] (axis cs:0.736956180655767,5.16721945149558) ellipse (0.034598220304999 and 2.67456286861783);
\draw[draw=darkorange25512714,fill=darkorange25512714,opacity=0.35,very thick] (axis cs:0.740262891628124,3.26295985115899) ellipse (0.0767271680594448 and 1.19916152622965);
\draw[draw=forestgreen4416044,fill=forestgreen4416044,opacity=0.35,very thick,postaction={pattern=north east lines, pattern color=forestgreen4416044, fill opacity=0.35}] (axis cs:0.738471567063105,3.24842510336921) ellipse (0.0449620738355365 and 1.4448231288365);
\addplot [semithick, darkslategray155990, mark=*, mark size=3, mark options={solid}, only marks]
table {%
0.736956180655767 5.16721945149558
};
\addlegendentry{JSD}
\addplot [semithick, saddlebrown127637, mark=*, mark size=3, mark options={solid}, only marks]
table {%
0.740262891628124 3.26295985115899
};
\addlegendentry{KLD}
\addplot [semithick, darkgreen228022, mark=*, mark size=3, mark options={solid}, only marks]
table {%
0.738471567063105 3.24842510336921
};
\addlegendentry{RKL}
\end{axis}

\end{tikzpicture}
        \label{fig:fdiv_5}
    \end{subfigure}
    \hfill
    \begin{subfigure}{0.49\linewidth}
        \centering
        \begin{tikzpicture}[scale=0.8]

\definecolor{darkgray176}{RGB}{176,176,176}
\definecolor{darkgreen228022}{RGB}{22,80,22}
\definecolor{darkorange25512714}{RGB}{255,127,14}
\definecolor{darkslategray155990}{RGB}{15,59,90}
\definecolor{forestgreen4416044}{RGB}{44,160,44}
\definecolor{lightgray204}{RGB}{204,204,204}
\definecolor{saddlebrown127637}{RGB}{127,63,7}
\definecolor{steelblue31119180}{RGB}{31,119,180}

\begin{axis}[
legend cell align={left},
legend style={fill opacity=0.8, draw opacity=1, text opacity=1, draw=lightgray204},
tick align=outside,
tick pos=left,
x grid style={darkgray176},
xlabel={Training Time (mins)},
xmin=0.708695396720536, xmax=1.69262751512478,
xtick style={color=black},
y grid style={darkgray176},
ylabel={$\ell_2$-Distance},
ymin=0.08, ymax=0.8,
ytick style={color=black}
]
\draw[draw=steelblue31119180,fill=steelblue31119180,opacity=0.35,very thick,postaction={pattern=horizontal lines, pattern color=steelblue31119180, fill opacity=0.35}] (axis cs:1.17352908509118,0.37458312582402) ellipse (0.219646879239569 and 0.26696520582551);
\draw[draw=darkorange25512714,fill=darkorange25512714,opacity=0.35,very thick] (axis cs:0.928342275960105,0.38503199886708) ellipse (0.0396913990892333 and 0.302501204822745);
\draw[draw=forestgreen4416044,fill=forestgreen4416044,opacity=0.35,very thick,postaction={pattern=north east lines, pattern color=forestgreen4416044, fill opacity=0.35}] (axis cs:1.47298063588521,0.320186023201261) ellipse (0.0397846994859983 and 0.225721918756054);
\addplot [semithick, darkslategray155990, mark=*, mark size=3, mark options={solid}, only marks]
table {%
1.17352908509118 0.37458312582402
};
\addlegendentry{JSD}
\addplot [semithick, saddlebrown127637, mark=*, mark size=3, mark options={solid}, only marks]
table {%
0.928342275960105 0.38503199886708
};
\addlegendentry{KLD}
\addplot [semithick, darkgreen228022, mark=*, mark size=3, mark options={solid}, only marks]
table {%
1.47298063588521 0.320186023201261
};
\addlegendentry{RKL}
\end{axis}

\end{tikzpicture}
        \label{fig:fdiv_6}
    \end{subfigure}
    \vspace{1em}
    \caption{
    Comparison of $f$-divergence performance for coupling-based and iResNet architectures in the inverse kinematics case study.
    The x-axis shows training time and the y-axis shows $\ell_2$-distance between true and generated end-effector positions. Ellipses are centered at the mean values, with horizontal and vertical diameters indicating the standard deviations in training time and $\ell_2$-distance. 
    The right and left columns correspond to forward loss~\eqref{eq:vdm-objective-forward} and backward losses~\eqref{eq:vdm-objective-backward-empirical}, respectively.
    The top row corresponds to the coupling-based architecture whereas the figures on the bottom row generated based on iResNet layers.}
    \label{fig:f-div_comparison}
\end{figure}

The training scheme, which incorporates f-divergences, differs from the configuration described in the previous section. Unlike models trained with the other cost functions, this setup includes both a discriminator network and a generator network (cf.~\citet{nowozin2016f}). 
Both backward and forward losses are obtained by finding the saddle point of $\mathcal{F}(\theta, \omega)$:
\begin{equation}
\theta^*, \omega^* = \arg \min_{\theta} \max_{\omega} \mathcal{F}(\theta, \omega).
\end{equation}
Each training epoch includes gradient updates for both the discriminator and the generator, and both networks were trained using the Adam optimizer with a learning rate of $2 \times 10^{-4}$.

To evaluate the  impact of different f-divergences we consider the inverse kinematics problem, describe in Section~\ref{sec:observation-1-prior}, aiming to recover joint parameters from the end-effector position.
We train INN models using each of the previously discussed f-divergence objectives to investigate their influence on the quality of generated samples. 
Training performance depends on multiple factors, including the network architecture, and, in this case study, we consider both coupling-based and residual invertible architectures.
Both coupling-based and iResNet architecture  comprise four layers, each with affine transformations parameterized by feedforward neural networks with a hidden size of 128. The initial weights of the network were initialized by sampling values from a standard normal distribution. The latent dimension was chosen to be $14$ for both configurations.

Figure \ref{fig:f-div_comparison} compares the mean square error and the computation times of the f-divergence losses. We take the mean error over 20 runs 
of the f-divergences under the given task. As depicted in this figure, training performance is influenced by many factors, including the network architecture, forward and backward training, as well as the capacity of the critic class. Also, 
      different variational formulations of f-divergences may be better suited to different settings.
     As shown in Figure \ref{fig:f-div_comparison}, the backward loss in equation~\eqref{eq:vdm-objective-backward-empirical} yields better performance, and coupling-based designs are more efficient.

\subsubsection{Effect of latent dimension in invertible neural networks}
\label{sec:observation-3}

\begin{figure}[t]
  \centering
  \begin{minipage}{0.48\textwidth}
    \centering
    \resizebox{\linewidth}{!}{%
    \begin{tikzpicture}%[font=\LARGE]

\definecolor{darkgray176}{RGB}{176,176,176}
\definecolor{green}{RGB}{0,128,0}
\definecolor{lightgray204}{RGB}{204,204,204}
\definecolor{orange}{RGB}{255,165,0}

\begin{axis}[
width=10.5cm,  % adjust as needed
height=9.7cm,  % adjust as needed
scale only axis,
legend cell align={left},
tick label style={font=\large},
label style={font=\Large},
legend style={fill opacity=0.8, draw opacity=1, text opacity=1, draw=lightgray204},
tick align=outside,
tick pos=left,
x grid style={darkgray176},
xlabel={Latent Dimension},
xmin=-0.5, xmax=9.5,
xtick style={color=black, font=\Large},
xticklabel style={font=\large},
xtick={0,1,2,3,4,5,6,7,8,9},
xticklabels={
  \(\displaystyle 2\),
  \(\displaystyle 6\),
  \(\displaystyle 10\),
  \(\displaystyle 14\),
  \(\displaystyle 18\),
  \(\displaystyle 22\),
  \(\displaystyle 26\),
  \(\displaystyle 30\),
  \(\displaystyle 34\),
  \(\displaystyle 38\)
},
xtick style={font=\Large},
y grid style={darkgray176},
ylabel=\textcolor{green}{KLD, $\ell_2$-Distance},
title = {Effect of latent dimension},
title style={font=\large,yshift=0.3cm},
ymin=-0.10657880802584, ymax=3.95228865408607,
ytick style={color=black, font=\Large},
yticklabel style={color=black, font=\Large}
]
\path [draw=green, draw opacity=0.8, very thick]
(axis cs:0,0.0779151675247011)
--(axis cs:0,0.475807135303815);

\path [draw=green, draw opacity=0.8, very thick]
(axis cs:1,1.63154910965001)
--(axis cs:1,3.76779467853553);

\path [draw=green, draw opacity=0.8, very thick]
(axis cs:2,0.52713570667102)
--(axis cs:2,1.12396696242778);

\path [draw=green, draw opacity=0.8, very thick]
(axis cs:3,0.209545804746449)
--(axis cs:3,0.429104186035693);

\path [draw=green, draw opacity=0.8, very thick]
(axis cs:4,0.270602234121826)
--(axis cs:4,0.540120175729195);

\path [draw=green, draw opacity=0.8, very thick]
(axis cs:5,0.148376725170584)
--(axis cs:5,0.344970426477847);

\path [draw=green, draw opacity=0.8, very thick]
(axis cs:6,0.238666679477319)
--(axis cs:6,0.52444616262801);

\path [draw=green, draw opacity=0.8, very thick]
(axis cs:7,0.2165593426852)
--(axis cs:7,0.617054666791643);

\path [draw=green, draw opacity=0.8, very thick]
(axis cs:8,0.777243125651564)
--(axis cs:8,2.39366130992061);

\path [draw=green, draw opacity=0.8, very thick]
(axis cs:9,0.108772109040902)
--(axis cs:9,0.320387755032806);

\addplot [semithick, green, opacity=0.8, mark=-, mark size=5, mark options={solid}, only marks, forget plot]
table {%
0 0.0779151675247011
1 1.63154910965001
2 0.52713570667102
3 0.209545804746449
4 0.270602234121826
5 0.148376725170584
6 0.238666679477319
7 0.2165593426852
8 0.777243125651564
9 0.108772109040902
};
\addplot [semithick, green, opacity=0.8, mark=-, mark size=5, mark options={solid}, only marks, forget plot]
table {%
0 0.475807135303815
1 3.76779467853553
2 1.12396696242778
3 0.429104186035693
4 0.540120175729195
5 0.344970426477847
6 0.52444616262801
7 0.617054666791643
8 2.39366130992061
9 0.320387755032806
};
\addplot [semithick, green, opacity=0.8, mark=*, mark size=4, mark options={solid}, only marks, forget plot]
table {%
0 0.276861151414258
1 2.69967189409277
2 0.825551334549399
3 0.319324995391071
4 0.405361204925511
5 0.246673575824215
6 0.381556421052665
7 0.416807004738422
8 1.58545221778609
9 0.214579932036854
};
\end{axis}

\begin{axis}[
width=10.5cm,  % adjust as needed
height=9.7cm,
scale only axis,
axis y line=right,
hide x axis,
log basis y={10},
tick align=outside,
y grid style={darkgray176},
ylabel=\textcolor{orange}{Sinkhorn, $\ell_2$-Distance},
ylabel style={font=\LARGE},
ymin=0.0248455786863937, ymax=0.724893668947103,
ymode=log,
ytick pos=right,
ytick style={color=black},
yticklabel style={anchor=west, font=\Large}  % Instead of just {anchor=west}
]
\path [draw=orange, draw opacity=0.8, very thick]
(axis cs:0,0.0707867849795591)
--(axis cs:0,0.205327986074345);

\path [draw=orange, draw opacity=0.8, very thick]
(axis cs:1,0.0289628417009399)
--(axis cs:1,0.18864007329657);

\path [draw=orange, draw opacity=0.8, very thick]
(axis cs:2,0.0295982289881933)
--(axis cs:2,0.188642436549777);

\path [draw=orange, draw opacity=0.8, very thick]
(axis cs:3,0.0640699136115256)
--(axis cs:3,0.248941985978967);

\path [draw=orange, draw opacity=0.8, very thick]
(axis cs:4,0.038180808581057)
--(axis cs:4,0.185649536550045);

\path [draw=orange, draw opacity=0.8, very thick]
(axis cs:5,0.0341302229180222)
--(axis cs:5,0.177088288856404);

\path [draw=orange, draw opacity=0.8, very thick]
(axis cs:6,0.0570021466839881)
--(axis cs:6,0.219012134486721);

\path [draw=orange, draw opacity=0.8, very thick]
(axis cs:7,0.138016059285118)
--(axis cs:7,0.621845151696886);

\path [draw=orange, draw opacity=0.8, very thick]
(axis cs:8,0.0625437977058547)
--(axis cs:8,0.257805326510043);

\path [draw=orange, draw opacity=0.8, very thick]
(axis cs:9,0.0760199317619914)
--(axis cs:9,0.294187067165261);

\addplot [semithick, orange, opacity=0.8, mark=-, mark size=5, mark options={solid}, only marks]
table {%
0 0.0707867849795591
1 0.0289628417009399
2 0.0295982289881933
3 0.0640699136115256
4 0.038180808581057
5 0.0341302229180222
6 0.0570021466839881
7 0.138016059285118
8 0.0625437977058547
9 0.0760199317619914
};
\addplot [semithick, orange, opacity=0.8, mark=-, mark size=5, mark options={solid}, only marks]
table {%
0 0.205327986074345
1 0.18864007329657
2 0.188642436549777
3 0.248941985978967
4 0.185649536550045
5 0.177088288856404
6 0.219012134486721
7 0.621845151696886
8 0.257805326510043
9 0.294187067165261
};
\addplot [semithick, orange, opacity=0.8, mark=square, mark size=4, mark options={solid,fill opacity=0}, only marks]
table {%
0 0.138057385526952
1 0.108801457498755
2 0.109120332768985
3 0.156505949795246
4 0.111915172565551
5 0.105609255887213
6 0.138007140585354
7 0.379930605491002
8 0.160174562107949
9 0.185103499463626
};
\end{axis}

\end{tikzpicture}
  }
    \captionof{figure}{This diagram illustrates the impact of varying latent dimensions on the posterior samples. The  circular dots represent the variational formulation of the forward KL divergence and the hollow squares correspond to the Sinkhorn divergence.  The x-axis shows latent dimension and the y-axis shows $\ell_2$-distance between true and generated end-effector positions. 
    }
    \label{obs:1}
  \end{minipage}\hfill
  \begin{minipage}{0.48\textwidth}
    \centering
    \resizebox{\linewidth}{!}{%
    \begin{tikzpicture}[scale=0.4]

\definecolor{crimson2143940}{RGB}{214,39,40}
\definecolor{darkgray176}{RGB}{176,176,176}
\definecolor{darkgreen228022}{RGB}{22,80,22}
\definecolor{darkorange25512714}{RGB}{255,127,14}
\definecolor{darkslategray155990}{RGB}{15,59,90}
\definecolor{forestgreen4416044}{RGB}{44,160,44}
\definecolor{lightgray204}{RGB}{204,204,204}
\definecolor{maroon1071920}{RGB}{107,19,20}
\definecolor{saddlebrown127637}{RGB}{127,63,7}
\definecolor{steelblue31119180}{RGB}{31,119,180}

\begin{axis}[
legend cell align={left},
legend columns=2,
legend style={fill opacity=0.8, draw opacity=1, text opacity=1, draw=lightgray204},
tick align=outside,
tick pos=left,
x grid style={darkgray176},
xlabel={Training Time (mins)},
xmin=0.344824008336143, xmax=0.523028865503886,
xtick style={color=black},
y grid style={darkgray176},
ylabel={$\ell_2$-Distance},
ymin=0.01, ymax=0.91,
ytick style={color=black},
title = {\scriptsize{Effect of Entropy Regularization parameter}}
]
\path [draw=steelblue31119180, draw opacity=0.8, very thick]
(axis cs:0.473028865503886,0.0622954262154443)
--(axis cs:0.473028865503886,0.249765211627597);

\path [draw=darkorange25512714, draw opacity=0.8, very thick]
(axis cs:0.454964831329527,0.0579206794500351)
--(axis cs:0.454964831329527,0.241661194534529);

\path [draw=forestgreen4416044, draw opacity=0.8, very thick]
(axis cs:0.435489033706605,0.0201806484588555)
--(axis cs:0.435489033706605,0.127420417432274);

\path [draw=crimson2143940, draw opacity=0.8, very thick]
(axis cs:0.400297844220722,0.0129211258497976)
--(axis cs:0.400297844220722,0.050723251576225);

\path [draw=steelblue31119180, draw opacity=0.8, very thick]
(axis cs:0.41825940381913,0.0622943593632607)
--(axis cs:0.41825940381913,0.249755858310631);

\path [draw=darkorange25512714, draw opacity=0.8, very thick]
(axis cs:0.406715494110471,0.0580104291439057)
--(axis cs:0.406715494110471,0.250879304040046);

\path [draw=forestgreen4416044, draw opacity=0.8, very thick]
(axis cs:0.397094922595554,0.836325191759637)
--(axis cs:0.397094922595554,0.890563970076896);

\path [draw=crimson2143940, draw opacity=0.8, very thick]
(axis cs:0.394824008336143,0.85375086377774)
--(axis cs:0.394824008336143,0.902284753641912);

\addplot [semithick, steelblue31119180, opacity=0.8, mark=-, mark size=5, mark options={solid}, only marks]
table {%
0.473028865503886 0.0622954262154443
};
\addlegendentry{$0.01$}
\addplot [semithick, steelblue31119180, opacity=0.8, mark=-, mark size=5, mark options={solid}, only marks, forget plot]
table {%
0.473028865503886 0.249765211627597
};
\addplot [semithick, darkorange25512714, opacity=0.8, mark=-, mark size=5, mark options={solid}, only marks]
table {%
0.454964831329527 0.0579206794500351
};
\addlegendentry{$0.1$}
\addplot [semithick, darkorange25512714, opacity=0.8, mark=-, mark size=5, mark options={solid}, only marks, forget plot]
table {%
0.454964831329527 0.241661194534529
};
\addplot [semithick, forestgreen4416044, opacity=0.8, mark=-, mark size=5, mark options={solid}, only marks]
table {%
0.435489033706605 0.0201806484588555
};
\addlegendentry{$1$}
\addplot [semithick, forestgreen4416044, opacity=0.8, mark=-, mark size=5, mark options={solid}, only marks, forget plot]
table {%
0.435489033706605 0.127420417432274
};
\addplot [semithick, crimson2143940, opacity=0.8, mark=-, mark size=5, mark options={solid}, only marks]
table {%
0.400297844220722 0.0129211258497976
};
\addlegendentry{$10$}
\addplot [semithick, crimson2143940, opacity=0.8, mark=-, mark size=5, mark options={solid}, only marks, forget plot]
table {%
0.400297844220722 0.050723251576225
};
\addplot [semithick, steelblue31119180, opacity=0.8, mark=-, mark size=5, mark options={solid}, only marks, forget plot]
table {%
0.41825940381913 0.0622943593632607
};
\addplot [semithick, steelblue31119180, opacity=0.8, mark=-, mark size=5, mark options={solid}, only marks, forget plot]
table {%
0.41825940381913 0.249755858310631
};
\addplot [semithick, darkorange25512714, opacity=0.8, mark=-, mark size=5, mark options={solid}, only marks, forget plot]
table {%
0.406715494110471 0.0580104291439057
};
\addplot [semithick, darkorange25512714, opacity=0.8, mark=-, mark size=5, mark options={solid}, only marks, forget plot]
table {%
0.406715494110471 0.250879304040046
};
\addplot [semithick, forestgreen4416044, opacity=0.8, mark=-, mark size=5, mark options={solid}, only marks, forget plot]
table {%
0.397094922595554 0.836325191759637
};
\addplot [semithick, forestgreen4416044, opacity=0.8, mark=-, mark size=5, mark options={solid}, only marks, forget plot]
table {%
0.397094922595554 0.890563970076896
};
\addplot [semithick, crimson2143940, opacity=0.8, mark=-, mark size=5, mark options={solid}, only marks, forget plot]
table {%
0.394824008336143 0.85375086377774
};
\addplot [semithick, crimson2143940, opacity=0.8, mark=-, mark size=5, mark options={solid}, only marks, forget plot]
table {%
0.394824008336143 0.902284753641912
};
\addplot [semithick, darkslategray155990, opacity=0.8, mark=*, mark size=4, mark options={solid}, only marks, forget plot]
table {%
0.473028865503886 0.156030318921521
};
\addplot [semithick, saddlebrown127637, opacity=0.8, mark=*, mark size=4, mark options={solid}, only marks, forget plot]
table {%
0.454964831329527 0.149790936992282
};
\addplot [semithick, darkgreen228022, opacity=0.8, mark=*, mark size=4, mark options={solid}, only marks, forget plot]
table {%
0.435489033706605 0.0738005329455648
};
\addplot [semithick, maroon1071920, opacity=0.8, mark=*, mark size=4, mark options={solid}, only marks, forget plot]
table {%
0.400297844220722 0.0318221887130113
};
\addplot [semithick, darkslategray155990, opacity=0.8, mark=square, mark size=4, mark options={solid,fill opacity=0}, only marks, forget plot]
table {%
0.41825940381913 0.156025108836946
};
\addplot [semithick, saddlebrown127637, opacity=0.8, mark=square, mark size=4, mark options={solid,fill opacity=0}, only marks, forget plot]
table {%
0.406715494110471 0.154444866591976
};
\addplot [semithick, darkgreen228022, opacity=0.8, mark=square, mark size=4, mark options={solid,fill opacity=0}, only marks, forget plot]
table {%
0.397094922595554 0.863444580918267
};
\addplot [semithick, maroon1071920, opacity=0.8, mark=square, mark size=4, mark options={solid,fill opacity=0}, only marks, forget plot]
table {%
0.394824008336143 0.878017808709826
};
\end{axis}

\end{tikzpicture}
  }
    \captionof{figure}{This diagram illustrates the effect of entropy regularization on the Wasserstein distance. The circular dots represent the debiased Sinkhorn divergence, while the hollow squares indicate the Sinkhorn approximation of the 
    $W_2$ distance. The x-axis shows training time and the y-axis shows the $\ell_2$-distance between true and generated end-effector positions. 
    }
    \label{fig:obs_4}
  \end{minipage}
\end{figure}

To investigate the impact of latent dimension size on training and inference, we train models with different latent dimensions (which can be tuned by appropriately changing $\dx$ by padding). The loss functions employed are the Sinkhorn divergence and the variational formulation of  the forward \ac{KL} divergence.
Also, similar to previous examples, the \ac{IK} problem described in Section~\ref{sec:observation-1-prior} is used for this case study.  
 As illustrated in Figure~\ref{obs:1}, we observe an
 approximately local $U$-shaped relationship between latent dimension and inference performance, where 
        the overall trend is  multimodal. 
        Additionally, training with the Sinkhorn divergence exhibits greater stability compared to the forward \ac{KL} divergence.
While in this case study we focus on the dimension of latent variable, the work in~\citet{hagemann2021stabilizing} investigated the number of modes of $P_Z$, demonstrating that parameterizing $P_Z$ in dependence on the labels $Y$ is helpful  for obtaining stable \acp{INN} with reasonable Lipschitz constants.

\subsubsection{Effect of entropic regularization on Sinkhorn divergence and Wasserstein distance approximation}
\label{sec:observation-4}

To demonstrate the effect of entropic regularization on the debiased Sinkhorn divergence and the Sinkhorn approximation of the Wasserstein distance, we vary the entropic regularization parameter across experiments. The model is trained for the inverse kinematics problem same training configuration described in Section~\ref{sec:observation-1-prior}.  As illustrated in Figure~\ref{fig:obs_4}, a decrease in entropic regularization tends to improves sample quality but increases training time.  Additionally, training with the Sinkhorn divergence exhibits greater stability compared to the Sinkhorn approximation of the Wasserstein distance.

\subsubsection{Effect of support set of the latent distribution and the data distribution}
\label{sec:observation-5}

We investigate the effect of the support set of the latent distribution when training with the variational formulation of the KL divergence and the Sinkhorn divergence. In this experiment, the data distribution is fixed as \( \mathcal{U}(0,1) \). The latent distributions considered for comparison are \( \mathcal{U}(0,1) \), \( \mathcal{U}(3,4) \), \( \mathcal{U}(5,6)  \), \( \mathcal{U}(10,11)\) and \( \mathcal{U} (15,16) \). 
 As illustrated in Figure~\ref{support} and Table~\ref{table_obs6}, the quality of the generated samples deteriorates as the support of the latent distribution becomes increasingly disjoint from that of the true distribution. This effect is observed for both cost functions; however, the model trained using the \( f \)-divergence formulation of the KL divergence exhibits a more pronounced degradation compared to the model trained with the Sinkhorn divergence. 
Thus, our empirical results support the hypothesis that models trained with the Sinkhorn divergence are more robust to mismatches between the supports of the latent and prior distributions.
\begin{figure}[!ht]
    \centering
    \input{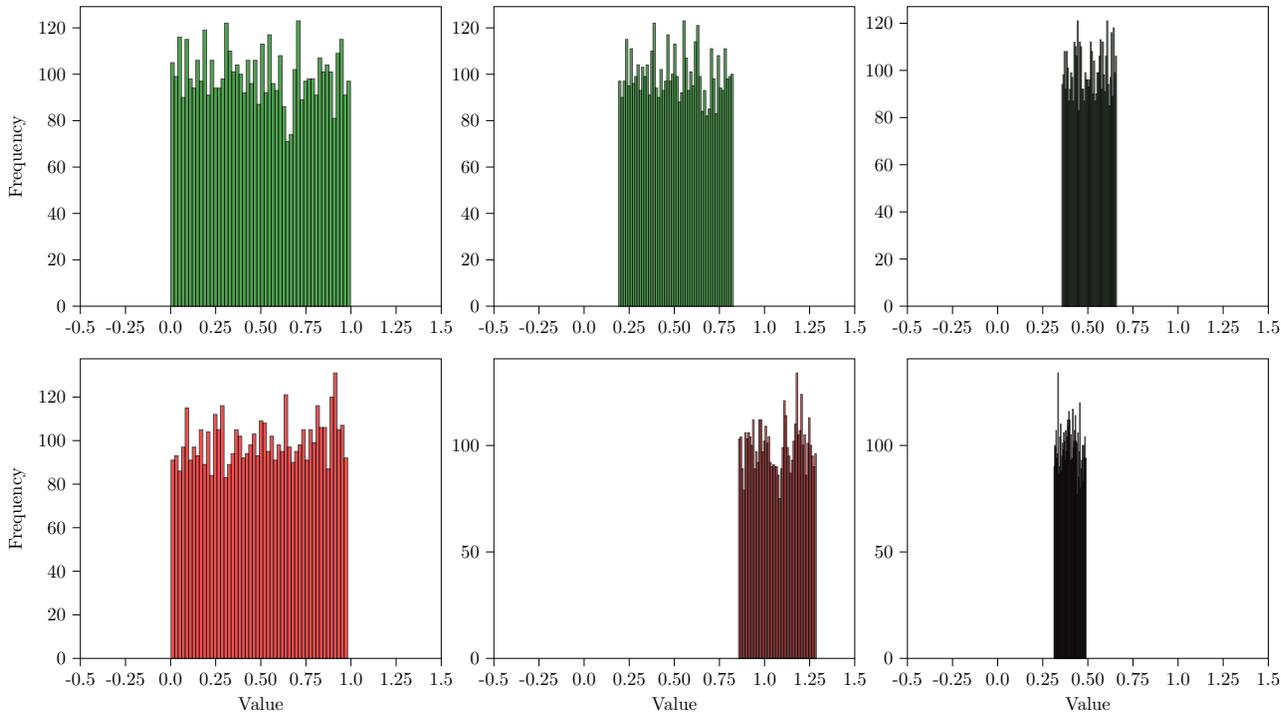}
    \caption{
\textbf{Top Row:} Samples from NF models trained using the Sinkhorn divergence. The left image corresponds to a latent distribution \( \mathcal{U}(0,1) \) used to simulate the data, the middle image uses $\mathcal{U}(3,4)$, and the right image uses $\mathcal{U}(5,6)$. 
\textbf{Bottom Row:} Samples from NF models trained using an $f$-divergence approximation of the forward KL divergence. The latent distributions mirror those used in the Sinkhorn divergence experiments.
}

    \label{support}
\end{figure}

\begin{table}[ht!]
  \centering 
  \small
  \renewcommand{\arraystretch}{1.25}
  \begin{tabular}{|l|c|c|c|}
    \hline
    \diagbox[width=9em,innerleftsep=0cm,innerrightsep=0.2cm]{~~~$P_Z$}{~~Divergence} & KL Divergence & Sinkhorn Divergence  \\
    \hline
    \rowcolor{Turquoise!12!white} 
    $\mathcal{U}(0,1)$ & 0.003 & 0.001 \\
    \hline
    $\mathcal{U}(3,4)$ & 1.22 & 0.24 \\
    \hline
    \rowcolor{Turquoise!12!white} 
    $\mathcal{U}(5,6)$ &  1.36 & 0.91  \\
    \hline
    $\mathcal{U}(10,11)$ & 3.07 & 1.28 \\
    \hline
    \rowcolor{Turquoise!12!white} 
    $\mathcal{U}(15,16)$ & 1.43 & 1.33  \\
    \hline
    
  \end{tabular}
  \vspace{2mm}
  \caption{This table shows the \ac{MMD} between samples from the true distribution $\mathcal{U}(0,1)$ and those generated by the trained model.
}
  \label{table_obs6}
\end{table}

\subsubsection{Effect of moments of $X$ on the performance of \ac{NF}}
\label{sec:observation-6}
In this experiment, we investigate how the number of finite moments of the $P_X$ affects the performance of the \ac{NF}. To ensure that the moment of $X$ can be adjusted, we assumed $X$ has a Pareto distribution~\citep{arnold2014pareto}. That is, the dataset was generated by sampling from a Pareto distribution characterized by a shape parameter $\alpha$ and a scale parameter $x_m$. In this case the probability density function of $X$ is given by
\begin{equation}
    f(X; x_m, \alpha) = 
    \begin{cases}
        \dfrac{\alpha x_m^{\alpha}}{X^{\alpha + 1}}, & X \geq x_m, \\
        0, & X < x_m,
    \end{cases}
\end{equation}
where $\alpha > 0$ and $x_m > 0$. Also, the \ac{NF} trained with the variational formulation of KL-Divergence. The latent space was defined by a $10$-dimensional random vector drawn from a standard normal distribution. 

Here, the parameter $\alpha$ is varied from 1 to 10, and,  as illustrated in Fig.~\ref{pareto_dist},  the Wasserstein-1 distance between the true and generated distributions decreased. This is consistent with our theoretical results in~\Cref{subsec:variational-nf} that as the number of finite moment of $X$ increases, the $W_1$ distance between the true and estimated distributions decreases.

\begin{figure}[!ht]
    \centering
    \begin{tikzpicture}
\begin{semilogyaxis}[
    width=12cm,
    height=7cm,
    xlabel={$\alpha$},
    ylabel={${W}_1$ Distance},
    xlabel style={font=\large},
    ylabel style={font=\large},
    xmin=0.5, xmax=10.5,
    ymin=0.04, ymax=40,
    grid=none,
    legend pos=north east,
    legend style={font=\small, draw=black, fill=white},
    legend cell align=left,
    mark size=2.5pt,
    xticklabel style={font=\scriptsize},
    yticklabel style={font=\scriptsize},
    major tick length=2pt,
]
\addplot[blue!70, thick, mark=*, solid] coordinates {
    (1, 28.5)
    (2, 0.74)
    (3, 0.106)
    (4, 0.105)
    (5, 0.134)
    (6, 0.168)
    (7, 0.153)
    (8, 0.071)
    (9, 0.069)
    (10, 0.054)
};
\addlegendentry{config 1}
\addplot[red!70, very thick,  mark=triangle, mark size = 4pt] coordinates {
    (1, 9.1)
    (2, 0.58)
    (3, 0.398)
    (4, 0.301)
    (5, 0.224)
    (6, 0.184)
    (7, 0.161)
    (8, 0.132)
    (9, 0.103)
    (10, 0.081)
};
\addlegendentry{config 2}
\addplot[cyan!70!blue, very thick,  mark=square, mark size=3pt] coordinates {
    (1, 1.26)
    (2, 0.58)
    (3, 0.416)
    (4, 0.316)
    (5, 0.245)
    (6, 0.206)
    (7, 0.181)
    (8, 0.146)
    (9, 0.117)
    (10, 0.092)
};
\addlegendentry{config 3}
\end{semilogyaxis}
\end{tikzpicture}
    \caption{This figure illustrates the Wasserstein-1 distance for increasing values of $\alpha$ in the Pareto distribution, which corresponds to increasing number of finite moments. The y-axis is set to a logarithmic (base 10) scale. Configurations 1–3 specify \acp{INN} with {2, 6, 8} coupling layers and subnetworks of hidden sizes {128, 64, 128}, respectively. 
    }
    \label{pareto_dist}
\end{figure}
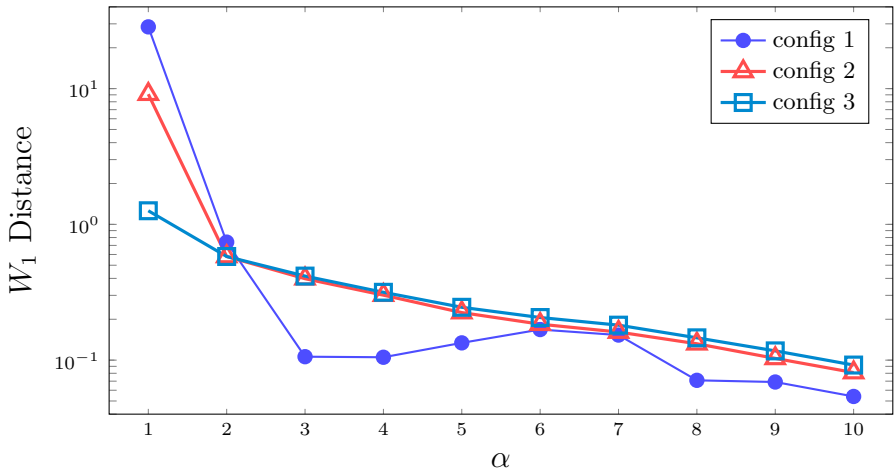

\subsection{Geoacoustic Inversion}
\label{subsec:geoacoustic_inversion}
To demonstrate the practical relevance of our approach,  we employ the insights gained from the previous  exploratory experiments to the \ac{GI} setting of the SWellEx-96 experiment~\citep{yardim2010geoacoustic,MeyGem:J21} conducted off the coast of San Diego near Point Loma. This experimental setting  is one of the most used, documented, and understood studies in the undersea acoustics community.\footnote{see \url{http://swellex96.ucsd.edu/}} Here, synthetic data corresponding to this experiment is used to simplify the task for the initial application of an \ac{INN}-based framework in ocean acoustics.

\begin{figure}[t]
    \centering      \includegraphics[scale=0.65,draft=false]{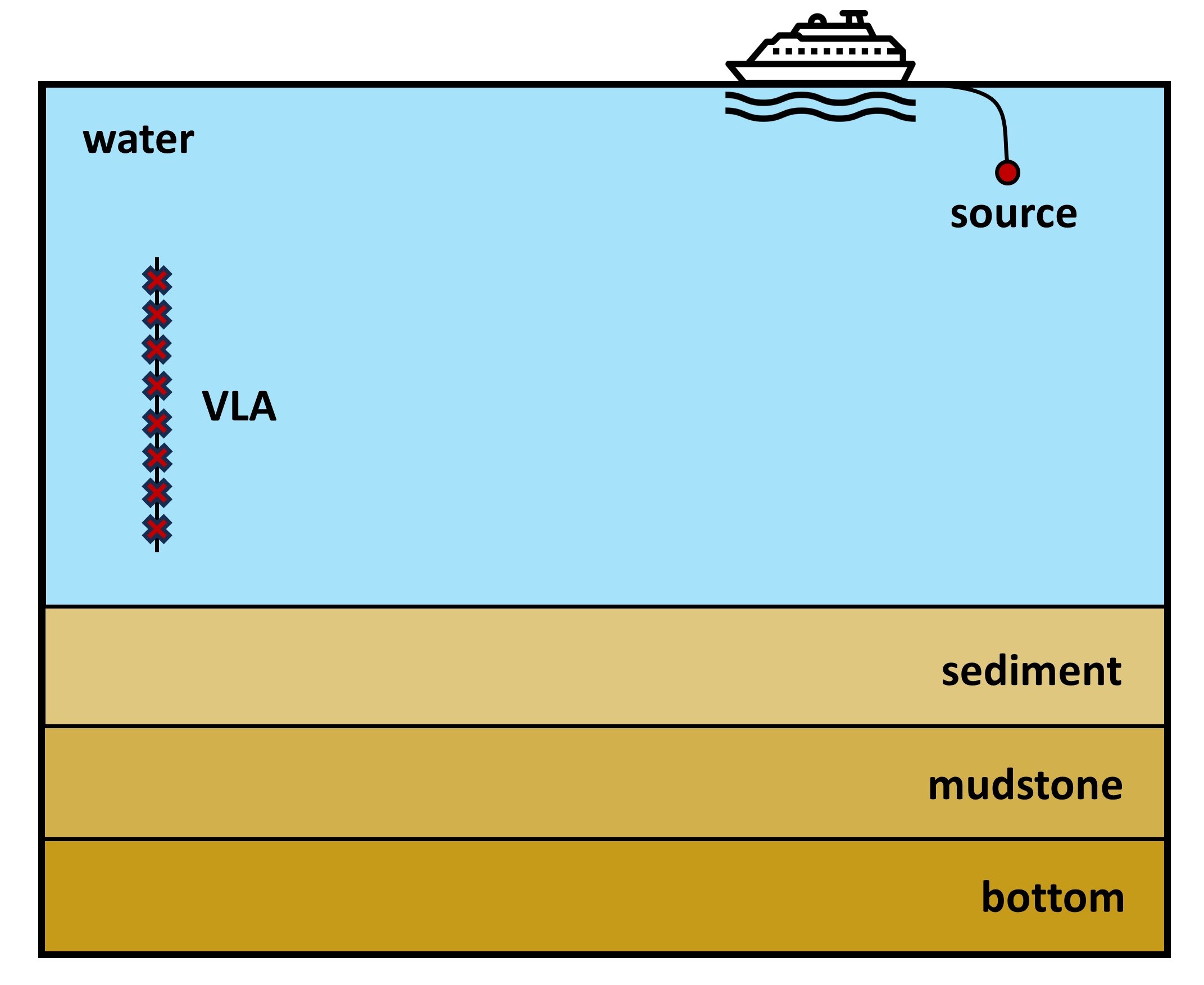}
    \caption{The SWellEx-96 experiment environment. The acoustic source is towed by a research vessel and transmits signals at various frequencies. The acoustic sensor consists of a vertical line array (VLA). Based on the measurements collected at the VLA, the objective is to estimate posterior distributions over parameters of interest (e.g. water depth, sound speed at the water-sediment interface, source range and depth, etc.).}
 \label{SWellEx-96-diagram}
\end{figure} 

 As depicted in Figure~\ref{SWellEx-96-diagram}, the data is collected via a vertical line array (VLA). 
The specification of the 21 hydrophones of the VLA and sound speed profile (SSP) in the water column is provided in the SWellEx-96 documentation. The SSP and sediment parameters are considered to be range-independent. Water depth refers to the depth of the water at the array. The source is towed by a research vessel which  consists of a comb signal comprising  frequencies of $49,79,112,148,201,283,\ {\mathrm{and}}\ 388\ \mathrm{Hz}$. While in the SWellEx-96 experiment the position of the source changes with time, for this task we consider the instant when the source depth is $60$ m and the distance (or range) between the source and the VLA is \mbox{$3$ km}.

The sediment layer is modeled with the following properties.
The seabed consists initially of a sediment layer that is $23.5$ meters thick, with a density of $1.76$ g/$\mbox{cm}^3$, and an attenuation of $0.2$ dB/kmHz. The sound speed at the bottom of this layer is assumed to be 
$1593$ m/s. The second layer is mudstone that is $800$ meters thick, possessing a density of $2.06$ g/$\mbox{cm}^3$, and an attenuation of $0.06$ dB/kmHz. The top and bottom sound speeds of this layer are $1881$ m/s and $3245$ m/s respectively. The description of the geoacoustic model of the SWellEx-96 experiment is complemented by a half-space featuring a density of $2.66$ g/$\mbox{cm}^3$, an attenuation of $0.020$ dB/kmHz, and a sound speed of $5200$ m/s.

Based on the measurements at the VLA, the objective of this task is to infer the posterior distribution over the water depth as well as the sound speed at the water-sediment interface. For this task, we assume all the quantities above to be known. The unknown parameters $m_1$ (the water depth) and $m_2$ (the sound speed at the water-sediment interface) follow a uniform prior in $[200.5, 236.5]$ m and $[1532, 1592]$ m/s, i.e. \mbox{$m_1\sim \mathcal{U}([200.5, 236.5])$} and $m_2\sim \mathcal{U}([1532, 1592])$, where $\mathcal{U}(\Omega)$ denotes a uniform distribution in \mbox{the domain $\Omega$.}

The received pressure {\bf y} on each hydrophone and for each frequency is a function of unknown parameters ${\bf m}$ (e.g. water depth, sound speed at the water-sediment interface, etc.) and additive noise $\boldsymbol{\epsilon}$ as follows
\begin{align}
\label{eq:GI-forwardmodel}
    {\bf y}= s({\bf m}, \boldsymbol{\epsilon})=F({\bf m}) + \boldsymbol{\epsilon},\qquad\qquad\boldsymbol{\epsilon} \sim \mathcal{N}(\boldsymbol{0}, \boldsymbol{\Sigma})
\end{align}
where $\boldsymbol{\Sigma}$ is the covariance matrix of data noise. Here, $s({\bf m}, \boldsymbol{\epsilon})$ is a known forward model that, assuming an additive noise model, can be rewritten as $F({\bf m}) + \boldsymbol{\epsilon}$, where $F({\bf m})$ represents the undersea acoustic model~\citep{jensen2011computational}. The SWellEx-96 experiment setup involves a complicated environment and no closed-form analytical solution is available for $F({\bf m})$. In this case, $F({\bf m})$ can only be evaluated numerically, and we use the normal-modes program KRAKEN~\citep{porter1992kraken} for this purpose. The signal-to-noise ratio is $15$ dB.  Inspired by \citet{zhang2021bayesian}, for the invertible architectures, we include data noise $\boldsymbol{\epsilon}$ as additional model parameters to be learned.
In this context, the input of the network is obtained by augmenting the unknown parameters {\bf m} with additive noise $\boldsymbol{\epsilon}$.

The pressures received on the hydrophones are considered in the frequency domain, and hence they can be complex numbers.
While the invertible architectures can be constructed to address complex numbers, in this case study, we stack the real and imaginary parts of the pressure field at the network's output. That is, the pressure $y=\mbox{Re}\{y\}+i \, \mbox{Im} \, \{y\}$ will be represented as $\big[\mbox{Re}\{y\}, \mbox{Im} \, \{y\}\big]$ at the network's output.

 This case study is analyzed noise and signal data collected from five hydrophones. Each hydrophone captured complex measurements across seven frequencies, which we processed by separating into real and imaginary components and then concatenating into a single vector. This processing approach resulted in input data tensors containing 70 signal measurements. We combined the input parameters with noise measurements, creating tensors of sizes 72.

\begin{figure}[!ht]
    \centering
    
    \begin{subfigure}[b]{0.3\textwidth}
        \centering
        % This file was created with tikzplotlib v0.10.1.
\begin{tikzpicture}[scale=0.62]

\definecolor{darkblue}{RGB}{0,0,139}
\definecolor{darkgray176}{RGB}{176,176,176}

\begin{axis}[
tick align=outside,
tick pos=left,
x grid style={darkgray176},
xlabel={Water Depth},
xmin=199, xmax=237,
xtick style={color=black},
y grid style={darkgray176},
ylabel={Probability Density},
ymin=0, ymax=0.245700012892485,
ytick style={color=black},
yticklabel style={/pgf/number format/fixed},
title={NLL}
]
\draw[draw=black,fill=darkblue] (axis cs:207.457946777344,0) rectangle (axis cs:208.369430541992,0.017000000923872);
\draw[draw=black,fill=darkblue] (axis cs:208.369445800781,0) rectangle (axis cs:209.280944824219,0.0830000042915344);
\draw[draw=black,fill=darkblue] (axis cs:209.280944824219,0) rectangle (axis cs:210.192428588867,0.189000010490417);
\draw[draw=black,fill=darkblue] (axis cs:210.192428588867,0) rectangle (axis cs:211.103927612305,0.224000006914139);
\draw[draw=black,fill=darkblue] (axis cs:211.103912353516,0) rectangle (axis cs:212.015396118164,0.234000012278557);
\draw[draw=black,fill=darkblue] (axis cs:212.015411376953,0) rectangle (axis cs:212.926895141602,0.156000003218651);
\draw[draw=black,fill=darkblue] (axis cs:212.926895141602,0) rectangle (axis cs:213.838394165039,0.0650000050663948);
\draw[draw=black,fill=darkblue] (axis cs:213.83837890625,0) rectangle (axis cs:214.749862670898,0.0270000007003546);
\draw[draw=black,fill=darkblue] (axis cs:214.749877929688,0) rectangle (axis cs:215.661376953125,0.00400000018998981);
\draw[draw=black,fill=darkblue] (axis cs:215.661376953125,0) rectangle (axis cs:216.572860717773,0.00100000004749745);
\addplot [red, dashed]
table {%
221.588272094727 -6.93889390390723e-18
221.588272094727 0.245700012892485
};
\end{axis}

\end{tikzpicture}
    \end{subfigure}
    \hfill
    \begin{subfigure}[b]{0.3\textwidth}
        \centering
        % This file was created with tikzplotlib v0.10.1.
\begin{tikzpicture}[scale=0.62]

\definecolor{darkblue}{RGB}{0,0,139}
\definecolor{darkgray176}{RGB}{176,176,176}

\begin{axis}[
tick align=outside,
tick pos=left,
x grid style={darkgray176},
xlabel={Water Depth},
xmin=199, xmax=237,
xtick style={color=black},
y grid style={darkgray176},
ylabel={Probability Density},
ymin=0, ymax=0.259350009262562,
ytick style={color=black},
yticklabel style={/pgf/number format/fixed},
title={KLD}
]
\draw[draw=black,fill=darkblue] (axis cs:219.741149902344,0) rectangle (axis cs:219.855331420898,0.017000000923872);
\draw[draw=black,fill=darkblue] (axis cs:219.855316162109,0) rectangle (axis cs:219.969497680664,0.0610000044107437);
\draw[draw=black,fill=darkblue] (axis cs:219.969512939453,0) rectangle (axis cs:220.083679199219,0.16100001335144);
\draw[draw=black,fill=darkblue] (axis cs:220.083679199219,0) rectangle (axis cs:220.197860717773,0.234000012278557);
\draw[draw=black,fill=darkblue] (axis cs:220.197845458984,0) rectangle (axis cs:220.312026977539,0.247000008821487);
\draw[draw=black,fill=darkblue] (axis cs:220.312042236328,0) rectangle (axis cs:220.426223754883,0.160000011324883);
\draw[draw=black,fill=darkblue] (axis cs:220.426208496094,0) rectangle (axis cs:220.540390014648,0.0720000043511391);
\draw[draw=black,fill=darkblue] (axis cs:220.540405273438,0) rectangle (axis cs:220.654571533203,0.034000001847744);
\draw[draw=black,fill=darkblue] (axis cs:220.654571533203,0) rectangle (axis cs:220.768753051758,0.00900000054389238);
\draw[draw=black,fill=darkblue] (axis cs:220.768737792969,0) rectangle (axis cs:220.882919311523,0.00500000035390258);
\addplot [red, dashed]
table {%
221.588272094727 -6.93889390390723e-18
221.588272094727 0.259350009262562
};
\end{axis}

\end{tikzpicture}

       \end{subfigure}
    \hfill
    \begin{subfigure}[b]{0.3\textwidth}
        \centering
        % This file was created with tikzplotlib v0.10.1.
\begin{tikzpicture}[scale=0.62]

\definecolor{darkblue}{RGB}{0,0,139}
\definecolor{darkgray176}{RGB}{176,176,176}

\begin{axis}[
tick align=outside,
tick pos=left,
x grid style={darkgray176},
xlabel={Water Depth},
xmin=199, xmax=237,
xtick style={color=black},
y grid style={darkgray176},
ylabel={Probability Density},
ymin=0, ymax=0.284550029039383,
ytick style={color=black},
yticklabel style={/pgf/number format/fixed},
title={JSD}
]
\draw[draw=black,fill=darkblue] (axis cs:219.657669067383,0) rectangle (axis cs:219.848129272461,0.00300000002607703);
\draw[draw=black,fill=darkblue] (axis cs:219.848129272461,0) rectangle (axis cs:220.038589477539,0.0410000011324883);
\draw[draw=black,fill=darkblue] (axis cs:220.038589477539,0) rectangle (axis cs:220.229049682617,0.129000008106232);
\draw[draw=black,fill=darkblue] (axis cs:220.229049682617,0) rectangle (axis cs:220.419509887695,0.242000013589859);
\draw[draw=black,fill=darkblue] (axis cs:220.419494628906,0) rectangle (axis cs:220.609970092773,0.271000027656555);
\draw[draw=black,fill=darkblue] (axis cs:220.609985351562,0) rectangle (axis cs:220.800445556641,0.181000009179115);
\draw[draw=black,fill=darkblue] (axis cs:220.800445556641,0) rectangle (axis cs:220.990905761719,0.0890000015497208);
\draw[draw=black,fill=darkblue] (axis cs:220.990905761719,0) rectangle (axis cs:221.181365966797,0.0360000021755695);
\draw[draw=black,fill=darkblue] (axis cs:221.181365966797,0) rectangle (axis cs:221.371826171875,0.00700000021606684);
\draw[draw=black,fill=darkblue] (axis cs:221.371826171875,0) rectangle (axis cs:221.562286376953,0.00100000004749745);
\addplot [red, dashed]
table {%
221.588272094727 -6.93889390390723e-18
221.588272094727 0.284550029039383
};
\end{axis}
\end{tikzpicture}
       
    \end{subfigure}

    \vspace{0.4cm} %

    \begin{subfigure}[b]{0.3\textwidth}
        \centering
        % This file was created with tikzplotlib v0.10.1.
\begin{tikzpicture}[scale=0.62]

\definecolor{darkblue}{RGB}{0,0,139}
\definecolor{darkgray176}{RGB}{176,176,176}

\begin{axis}[
tick align=outside,
tick pos=left,
x grid style={darkgray176},
xlabel={Water Depth},
xmin=199, xmax=237,
xtick style={color=black},
y grid style={darkgray176},
ylabel={Probability Density},
ymin=0, ymax=0.33180001527071,
ytick style={color=black},
title={NLL}
]
\draw[draw=black,fill=darkblue] (axis cs:221.002502441406,0) rectangle (axis cs:221.155349731445,0.00100000004749745);
\draw[draw=black,fill=darkblue] (axis cs:221.155334472656,0) rectangle (axis cs:221.308197021484,0.0020000000949949);
\draw[draw=black,fill=darkblue] (axis cs:221.308197021484,0) rectangle (axis cs:221.461044311523,0.00400000018998981);
\draw[draw=black,fill=darkblue] (axis cs:221.461044311523,0) rectangle (axis cs:221.613906860352,0.0190000012516975);
\draw[draw=black,fill=darkblue] (axis cs:221.613922119141,0) rectangle (axis cs:221.76676940918,0.0520000010728836);
\draw[draw=black,fill=darkblue] (axis cs:221.766754150391,0) rectangle (axis cs:221.91960144043,0.109000004827976);
\draw[draw=black,fill=darkblue] (axis cs:221.91960144043,0) rectangle (axis cs:222.072463989258,0.234000012278557);
\draw[draw=black,fill=darkblue] (axis cs:222.072479248047,0) rectangle (axis cs:222.225326538086,0.316000014543533);
\draw[draw=black,fill=darkblue] (axis cs:222.225311279297,0) rectangle (axis cs:222.378173828125,0.220000013709068);
\draw[draw=black,fill=darkblue] (axis cs:222.378173828125,0) rectangle (axis cs:222.531021118164,0.0430000014603138);
\addplot [red, dashed]
table {%
221.588272094727 0
221.588272094727 0.33180001527071
};
\end{axis}

\end{tikzpicture}
       
    \end{subfigure}
    \hfill
    \begin{subfigure}[b]{0.3\textwidth}
        \centering
        % This file was created with tikzplotlib v0.10.1.
\begin{tikzpicture}[scale=0.62]

\definecolor{darkblue}{RGB}{0,0,139}
\definecolor{darkgray176}{RGB}{176,176,176}

\begin{axis}[
tick align=outside,
tick pos=left,
x grid style={darkgray176},
xlabel={Water Depth},
xmin=199, xmax=237,
xtick style={color=black},
y grid style={darkgray176},
ylabel={Probability Density},
ymin=0, ymax=0.235200007259846,
ytick style={color=black},
yticklabel style={/pgf/number format/fixed},
title={KLD}
]
\draw[draw=black,fill=darkblue] (axis cs:217.664337158203,0) rectangle (axis cs:217.729797363281,0.00400000018998981);
\draw[draw=black,fill=darkblue] (axis cs:217.729797363281,0) rectangle (axis cs:217.79524230957,0.0110000008717179);
\draw[draw=black,fill=darkblue] (axis cs:217.79524230957,0) rectangle (axis cs:217.860702514648,0.0320000015199184);
\draw[draw=black,fill=darkblue] (axis cs:217.860702514648,0) rectangle (axis cs:217.926162719727,0.0430000014603138);
\draw[draw=black,fill=darkblue] (axis cs:217.926177978516,0) rectangle (axis cs:217.991622924805,0.0880000069737434);
\draw[draw=black,fill=darkblue] (axis cs:217.991607666016,0) rectangle (axis cs:218.057067871094,0.128000006079674);
\draw[draw=black,fill=darkblue] (axis cs:218.057067871094,0) rectangle (axis cs:218.122528076172,0.173000007867813);
\draw[draw=black,fill=darkblue] (axis cs:218.122528076172,0) rectangle (axis cs:218.18798828125,0.18800000846386);
\draw[draw=black,fill=darkblue] (axis cs:218.18798828125,0) rectangle (axis cs:218.253433227539,0.224000006914139);
\draw[draw=black,fill=darkblue] (axis cs:218.253433227539,0) rectangle (axis cs:218.318893432617,0.109000004827976);
\addplot [red, dashed]
table {%
221.588272094727 0
221.588272094727 0.235200007259846
};
\end{axis}

\end{tikzpicture}
       
    \end{subfigure}
    \hfill
    \begin{subfigure}[b]{0.284\textwidth}
        \centering
        % This file was created with tikzplotlib v0.10.1.
\begin{tikzpicture}[scale=0.62]

\definecolor{darkblue}{RGB}{0,0,139}
\definecolor{darkgray176}{RGB}{176,176,176}

\begin{axis}[
tick align=outside,
tick pos=left,
x grid style={darkgray176},
xlabel={Water Depth},
xmin=199, xmax=237,
xtick style={color=black},
y grid style={darkgray176},
ylabel={Probability Density},
ymin=0, ymax=0.329700011014938,
ytick style={color=black},
title={JSD}
]
\draw[draw=black,fill=darkblue] (axis cs:217.600769042969,0) rectangle (axis cs:217.654006958008,0.00500000035390258);
\draw[draw=black,fill=darkblue] (axis cs:217.654006958008,0) rectangle (axis cs:217.707229614258,0.0140000004321337);
\draw[draw=black,fill=darkblue] (axis cs:217.707244873047,0) rectangle (axis cs:217.760482788086,0.034000001847744);
\draw[draw=black,fill=darkblue] (axis cs:217.760467529297,0) rectangle (axis cs:217.813705444336,0.0440000034868717);
\draw[draw=black,fill=darkblue] (axis cs:217.813720703125,0) rectangle (axis cs:217.866958618164,0.0930000022053719);
\draw[draw=black,fill=darkblue] (axis cs:217.866943359375,0) rectangle (axis cs:217.920166015625,0.114000007510185);
\draw[draw=black,fill=darkblue] (axis cs:217.920166015625,0) rectangle (axis cs:217.973403930664,0.146999999880791);
\draw[draw=black,fill=darkblue] (axis cs:217.973419189453,0) rectangle (axis cs:218.026657104492,0.186000004410744);
\draw[draw=black,fill=darkblue] (axis cs:218.026641845703,0) rectangle (axis cs:218.079864501953,0.314000010490417);
\draw[draw=black,fill=darkblue] (axis cs:218.079864501953,0) rectangle (axis cs:218.133102416992,0.0490000024437904);
\addplot [red, dashed]
table {%
221.588272094727 0
221.588272094727 0.329700011014938
};
\end{axis}

\end{tikzpicture}
       
    \end{subfigure}

    \caption{Posterior comparisons between models trained with different loss functions for the \ac{GI} dataset, using signal measurements from \textbf{five} hydrophones. The inferred parameter is \textbf{water depth}, with the red dotted line indicating the true value. The \ac{INN} configuration for the upper row uses Configuration \ref{def:configuration1} whereas the configuration used in the bottom row is Configuration \ref{def:configuration2}. The y-axis shows un-normalized probability scores.}
    \label{fig:water_depth}
\end{figure}
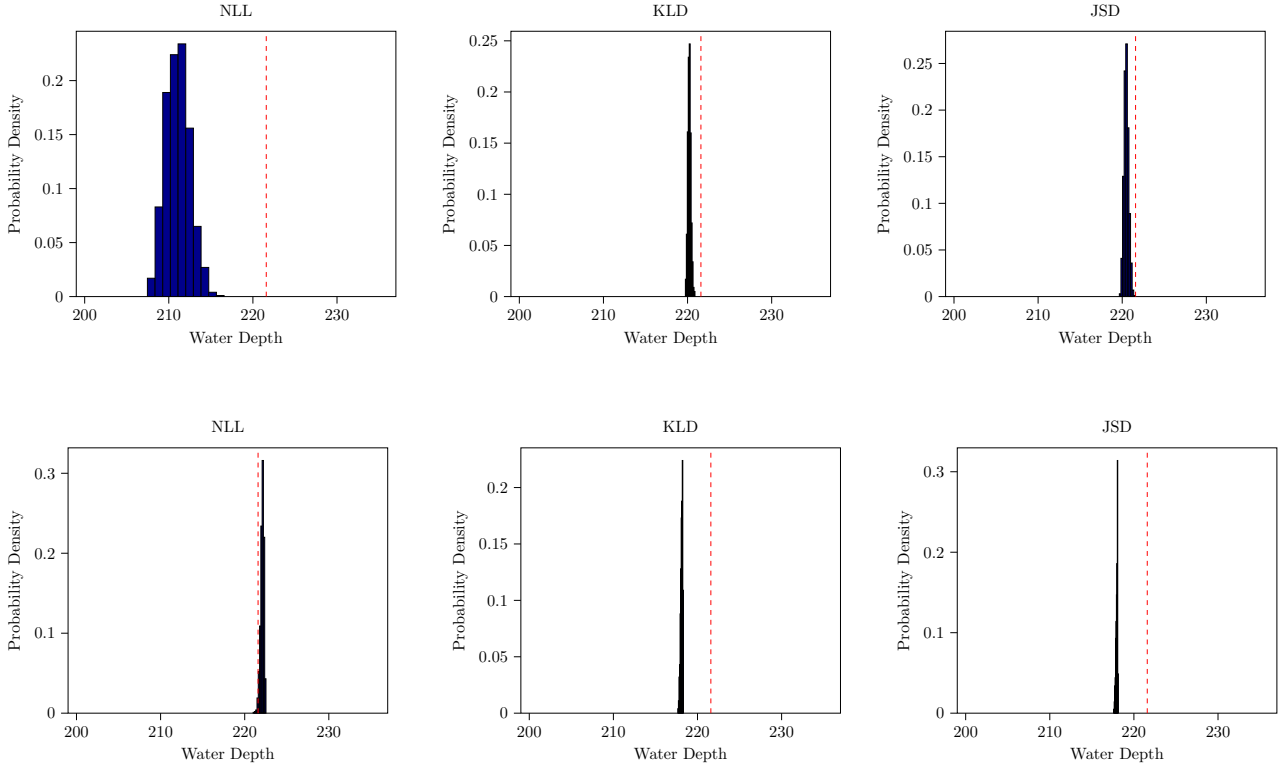

\noindent
For the \ac{GI} dataset, training is conducted using a batch size of 128 over 10 epochs with an 80/20 train--validation split. 
The optimization settings vary by objective: the \ac{INN} trained with the \ac{NLL} loss uses a learning rate of 
$3\times10^{-4}$, while the \ac{INN} and critic network $V_\omega$ trained with the f-divergence objective use a 
learning rate of $2\times10^{-4}$ together with $L_2$ regularization of $2\times10^{-5}$. All models employ the Adam optimizer with betas $(0.8,\,0.9)$. 
\vspace{2mm}
\noindent
Two different coupling-based configuration are considered. 
\begin{configuration}
\label[definition]{def:configuration1}
corresponds to a model with sub-networks of hidden size 64, a latent dimension of 18, and 3 coupling layers.
\end{configuration}

\begin{configuration}
\label[definition]{def:configuration2}
corresponds to a model with sub-networks of hidden size 403, a latent dimension of 2, and 4 coupling layers.
\end{configuration}

\noindent
The results are presented in Figs~\ref{fig:water_depth} and \ref{fig:sound_speed}, and in Tables  \ref{tab:water_depth}, \ref{tab:sound_speed} and \ref{table:training_time}. The histograms, in Figs~\ref{fig:water_depth} and \ref{fig:sound_speed}, are based on 1,000 samples drawn from the trained \ac{INN}. Also, while training times are compared in Table~\ref{table:training_time} the inference time is also provided in its caption.

\begin{figure}[!ht]
    \centering
    
    \begin{subfigure}[b]{0.3\textwidth}
        \centering
        % This file was created with tikzplotlib v0.10.1.
\begin{tikzpicture}[scale=0.62]

\definecolor{darkblue}{RGB}{0,0,139}
\definecolor{darkgray176}{RGB}{176,176,176}

\begin{axis}[
tick align=outside,
tick pos=left,
x grid style={darkgray176},
xlabel={Sound Speed},
xmin=1492, xmax=1592,
xtick style={color=black},
y grid style={darkgray176},
ylabel={Probability Density},
ymin=0, ymax=0.253050012141466,
ytick style={color=black}, 
yticklabel style={/pgf/number format/fixed},
title = {NLL}
]
\draw[draw=black,fill=darkblue] (axis cs:1562.42724609375,0) rectangle (axis cs:1562.97009277344,0.00100000004749745);
\draw[draw=black,fill=darkblue] (axis cs:1562.97021484375,0) rectangle (axis cs:1563.51318359375,0.00500000035390258);
\draw[draw=black,fill=darkblue] (axis cs:1563.51318359375,0) rectangle (axis cs:1564.05615234375,0.0310000013560057);
\draw[draw=black,fill=darkblue] (axis cs:1564.05615234375,0) rectangle (axis cs:1564.59899902344,0.0860000029206276);
\draw[draw=black,fill=darkblue] (axis cs:1564.59887695312,0) rectangle (axis cs:1565.14172363281,0.169000014662743);
\draw[draw=black,fill=darkblue] (axis cs:1565.14184570312,0) rectangle (axis cs:1565.68481445312,0.241000011563301);
\draw[draw=black,fill=darkblue] (axis cs:1565.68481445312,0) rectangle (axis cs:1566.22778320312,0.222000017762184);
\draw[draw=black,fill=darkblue] (axis cs:1566.22778320312,0) rectangle (axis cs:1566.77062988281,0.141000002622604);
\draw[draw=black,fill=darkblue] (axis cs:1566.7705078125,0) rectangle (axis cs:1567.31335449219,0.0790000036358833);
\draw[draw=black,fill=darkblue] (axis cs:1567.3134765625,0) rectangle (axis cs:1567.8564453125,0.025000000372529);
\addplot [red, dashed]
table {%
1554.46752929688 -6.93889390390723e-18
1554.46752929688 0.253050012141466
};
\end{axis}

\end{tikzpicture}
        
    \end{subfigure}
    \hfill
    \begin{subfigure}[b]{0.3\textwidth}
        \centering
        % This file was created with tikzplotlib v0.10.1.
\begin{tikzpicture}[scale=0.62]

\definecolor{darkblue}{RGB}{0,0,139}
\definecolor{darkgray176}{RGB}{176,176,176}

\begin{axis}[
tick align=outside,
tick pos=left,
x grid style={darkgray176},
xlabel={Sound Speed},
xmin=1492, xmax=1592,
xtick style={color=black},
y grid style={darkgray176},
ylabel={Probability Density},
ymin=0, ymax=0.255150016397238,
ytick style={color=black}, 
yticklabel style={/pgf/number format/fixed},
title={KLD}
]
\draw[draw=black,fill=darkblue] (axis cs:1552.28991699219,0) rectangle (axis cs:1552.84216308594,0.00400000018998981);
\draw[draw=black,fill=darkblue] (axis cs:1552.84216308594,0) rectangle (axis cs:1553.39440917969,0.0190000012516975);
\draw[draw=black,fill=darkblue] (axis cs:1553.39453125,0) rectangle (axis cs:1553.94689941406,0.0910000056028366);
\draw[draw=black,fill=darkblue] (axis cs:1553.94677734375,0) rectangle (axis cs:1554.4990234375,0.146000012755394);
\draw[draw=black,fill=darkblue] (axis cs:1554.4990234375,0) rectangle (axis cs:1555.05126953125,0.243000015616417);
\draw[draw=black,fill=darkblue] (axis cs:1555.05126953125,0) rectangle (axis cs:1555.603515625,0.22800001502037);
\draw[draw=black,fill=darkblue] (axis cs:1555.603515625,0) rectangle (axis cs:1556.15576171875,0.179000005125999);
\draw[draw=black,fill=darkblue] (axis cs:1556.15576171875,0) rectangle (axis cs:1556.70812988281,0.0650000050663948);
\draw[draw=black,fill=darkblue] (axis cs:1556.70812988281,0) rectangle (axis cs:1557.26037597656,0.0210000015795231);
\draw[draw=black,fill=darkblue] (axis cs:1557.26037597656,0) rectangle (axis cs:1557.81262207031,0.00400000018998981);
\addplot [red, dashed]
table {%
1554.46752929688 -6.93889390390723e-18
1554.46752929688 0.255150016397238
};
\end{axis}

\end{tikzpicture}
        
    \end{subfigure}
    \hfill
    \begin{subfigure}[b]{0.3\textwidth}
        \centering
        % This file was created with tikzplotlib v0.10.1.
\begin{tikzpicture}[scale=0.62]

\definecolor{darkblue}{RGB}{0,0,139}
\definecolor{darkgray176}{RGB}{176,176,176}

\begin{axis}[
tick align=outside,
tick pos=left,
x grid style={darkgray176},
xlabel={Sound Speed},
xmin=1492, xmax=1592,
xtick style={color=black},
y grid style={darkgray176},
ylabel={Probability Density},
ymin=0, ymax=0.232050016522408,
ytick style={color=black},
yticklabel style={/pgf/number format/fixed},
title={JSD}
]
\draw[draw=black,fill=darkblue] (axis cs:1552.18200683594,0) rectangle (axis cs:1552.67175292969,0.00600000005215406);
\draw[draw=black,fill=darkblue] (axis cs:1552.67175292969,0) rectangle (axis cs:1553.16149902344,0.0329999998211861);
\draw[draw=black,fill=darkblue] (axis cs:1553.16149902344,0) rectangle (axis cs:1553.65124511719,0.0570000037550926);
\draw[draw=black,fill=darkblue] (axis cs:1553.65124511719,0) rectangle (axis cs:1554.14099121094,0.118000008165836);
\draw[draw=black,fill=darkblue] (axis cs:1554.14111328125,0) rectangle (axis cs:1554.63098144531,0.190000012516975);
\draw[draw=black,fill=darkblue] (axis cs:1554.630859375,0) rectangle (axis cs:1555.12060546875,0.221000015735626);
\draw[draw=black,fill=darkblue] (axis cs:1555.12060546875,0) rectangle (axis cs:1555.6103515625,0.185000002384186);
\draw[draw=black,fill=darkblue] (axis cs:1555.6103515625,0) rectangle (axis cs:1556.10009765625,0.122000008821487);
\draw[draw=black,fill=darkblue] (axis cs:1556.10009765625,0) rectangle (axis cs:1556.58984375,0.0530000030994415);
\draw[draw=black,fill=darkblue] (axis cs:1556.58984375,0) rectangle (axis cs:1557.07958984375,0.0150000005960464);
\addplot [red, dashed]
table {%
1554.46752929688 -6.93889390390723e-18
1554.46752929688 0.232050016522408
};
\end{axis}

\end{tikzpicture}
        
    \end{subfigure}

    \vspace{0.4cm} %

    \begin{subfigure}[b]{0.3\textwidth}
        \centering
        % This file was created with tikzplotlib v0.10.1.
\begin{tikzpicture}[scale=0.62]

\definecolor{darkblue}{RGB}{0,0,139}
\definecolor{darkgray176}{RGB}{176,176,176}

\begin{axis}[
tick align=outside,
tick pos=left,
x grid style={darkgray176},
xlabel={Sound Speed},
xmin=1492, xmax=1592,
xtick style={color=black},
y grid style={darkgray176},
ylabel={Probability Density},
ymin=0, ymax=0.276150012016296,
ytick style={color=black},
yticklabel style={/pgf/number format/fixed},
title={NLL}
]
\draw[draw=black,fill=darkblue] (axis cs:1552.78247070312,0) rectangle (axis cs:1552.99670410156,0.0020000000949949);
\draw[draw=black,fill=darkblue] (axis cs:1552.99682617188,0) rectangle (axis cs:1553.21105957031,0.0130000002682209);
\draw[draw=black,fill=darkblue] (axis cs:1553.2109375,0) rectangle (axis cs:1553.42517089844,0.0430000014603138);
\draw[draw=black,fill=darkblue] (axis cs:1553.42529296875,0) rectangle (axis cs:1553.63952636719,0.113000005483627);
\draw[draw=black,fill=darkblue] (axis cs:1553.63940429688,0) rectangle (axis cs:1553.85363769531,0.207000017166138);
\draw[draw=black,fill=darkblue] (axis cs:1553.85375976562,0) rectangle (axis cs:1554.06799316406,0.263000011444092);
\draw[draw=black,fill=darkblue] (axis cs:1554.06787109375,0) rectangle (axis cs:1554.28210449219,0.212000012397766);
\draw[draw=black,fill=darkblue] (axis cs:1554.2822265625,0) rectangle (axis cs:1554.49645996094,0.106000006198883);
\draw[draw=black,fill=darkblue] (axis cs:1554.49633789062,0) rectangle (axis cs:1554.71057128906,0.0329999998211861);
\draw[draw=black,fill=darkblue] (axis cs:1554.71069335938,0) rectangle (axis cs:1554.92492675781,0.00800000037997961);
\addplot [red, dashed]
table {%
1554.46752929688 -6.93889390390723e-18
1554.46752929688 0.276150012016296
};
\end{axis}

\end{tikzpicture}
        
    \end{subfigure}
    \hfill
    \begin{subfigure}[b]{0.3\textwidth}
        \centering
        % This file was created with tikzplotlib v0.10.1.
\begin{tikzpicture}[scale=0.62]

\definecolor{darkblue}{RGB}{0,0,139}
\definecolor{darkgray176}{RGB}{176,176,176}

\begin{axis}[
tick align=outside,
tick pos=left,
x grid style={darkgray176},
xlabel={Sound Speed},
xmin=1492, xmax=1592,
xtick style={color=black},
y grid style={darkgray176},
ylabel={Probability Density},
ymin=0, ymax=0.435750022530556,
ytick style={color=black}, 
title={KLD}
]
\draw[draw=black,fill=darkblue] (axis cs:1550.8505859375,0) rectangle (axis cs:1551.05712890625,0.0020000000949949);
\draw[draw=black,fill=darkblue] (axis cs:1551.05712890625,0) rectangle (axis cs:1551.26379394531,0.00600000005215406);
\draw[draw=black,fill=darkblue] (axis cs:1551.26391601562,0) rectangle (axis cs:1551.47058105469,0.00600000005215406);
\draw[draw=black,fill=darkblue] (axis cs:1551.47045898438,0) rectangle (axis cs:1551.67700195312,0.0140000004321337);
\draw[draw=black,fill=darkblue] (axis cs:1551.67700195312,0) rectangle (axis cs:1551.88354492188,0.0290000010281801);
\draw[draw=black,fill=darkblue] (axis cs:1551.88354492188,0) rectangle (axis cs:1552.09020996094,0.0530000030994415);
\draw[draw=black,fill=darkblue] (axis cs:1552.09033203125,0) rectangle (axis cs:1552.29699707031,0.0690000057220459);
\draw[draw=black,fill=darkblue] (axis cs:1552.296875,0) rectangle (axis cs:1552.50341796875,0.144000008702278);
\draw[draw=black,fill=darkblue] (axis cs:1552.50341796875,0) rectangle (axis cs:1552.7099609375,0.262000024318695);
\draw[draw=black,fill=darkblue] (axis cs:1552.7099609375,0) rectangle (axis cs:1552.91662597656,0.415000021457672);
\addplot [red, dashed]
table {%
1554.46752929688 -6.93889390390723e-18
1554.46752929688 0.435750022530556
};
\end{axis}

\end{tikzpicture}
        
    \end{subfigure}
    \hfill
    \begin{subfigure}[b]{0.3\textwidth}
        \centering
        % This file was created with tikzplotlib v0.10.1.
\begin{tikzpicture}[scale=0.62]

\definecolor{darkblue}{RGB}{0,0,139}
\definecolor{darkgray176}{RGB}{176,176,176}

\begin{axis}[
tick align=outside,
tick pos=left,
x grid style={darkgray176},
xlabel={Water Depth},
xmin=199, xmax=237,
xtick style={color=black},
y grid style={darkgray176},
ylabel={Probability Density},
ymin=0, ymax=0.329700011014938,
ytick style={color=black},
title={JSD}
]
\draw[draw=black,fill=darkblue] (axis cs:217.600769042969,0) rectangle (axis cs:217.654006958008,0.00500000035390258);
\draw[draw=black,fill=darkblue] (axis cs:217.654006958008,0) rectangle (axis cs:217.707229614258,0.0140000004321337);
\draw[draw=black,fill=darkblue] (axis cs:217.707244873047,0) rectangle (axis cs:217.760482788086,0.034000001847744);
\draw[draw=black,fill=darkblue] (axis cs:217.760467529297,0) rectangle (axis cs:217.813705444336,0.0440000034868717);
\draw[draw=black,fill=darkblue] (axis cs:217.813720703125,0) rectangle (axis cs:217.866958618164,0.0930000022053719);
\draw[draw=black,fill=darkblue] (axis cs:217.866943359375,0) rectangle (axis cs:217.920166015625,0.114000007510185);
\draw[draw=black,fill=darkblue] (axis cs:217.920166015625,0) rectangle (axis cs:217.973403930664,0.146999999880791);
\draw[draw=black,fill=darkblue] (axis cs:217.973419189453,0) rectangle (axis cs:218.026657104492,0.186000004410744);
\draw[draw=black,fill=darkblue] (axis cs:218.026641845703,0) rectangle (axis cs:218.079864501953,0.314000010490417);
\draw[draw=black,fill=darkblue] (axis cs:218.079864501953,0) rectangle (axis cs:218.133102416992,0.0490000024437904);
\addplot [red, dashed]
table {%
221.588272094727 0
221.588272094727 0.329700011014938
};
\end{axis}

\end{tikzpicture}
        
    \end{subfigure}

    \caption{Posterior comparisons between models trained with different loss functions for the \ac{GI} dataset, using signal measurements from \textbf{five} hydrophones. The inferred parameter is \textbf{sound speed}, with the red dotted line indicating the true value. The \ac{INN} configuration for the upper row uses Configuration \ref{def:configuration1} whereas the configuration used in the bottom row is Configuration \ref{def:configuration2}.}
    \label{fig:sound_speed}
\end{figure}
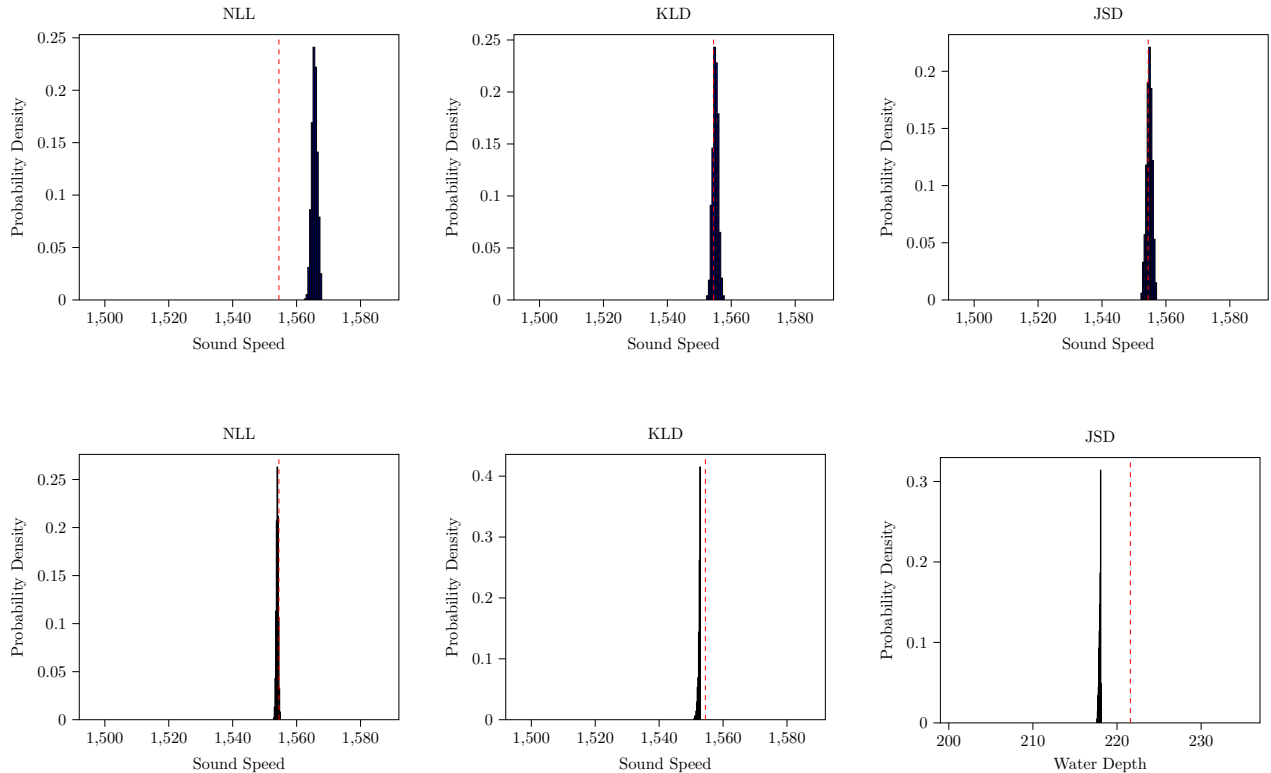

\begin{table}[ht!]
  \centering
  \begin{subtable}[t]{0.48\textwidth}
    \centering
    \small
    \renewcommand{\arraystretch}{1.25}
    \begin{tabular}{|l|c|c|c|}
      \hline
      \diagbox[width=9em,innerleftsep=0cm,innerrightsep=0.2cm]{~Method}{Epochs} & $1$ & $5$ & $10$ \\
      \hline
      NLL & 6.19 & 22.90 & 22.48 \\
      \hline
      \rowcolor{Turquoise!12!white}
      KL Divergence & 7.06 & 3.46 & 3.27 \\
      \hline
      JS Divergence & 6.63 & 3.21 & 3.26 \\
      \hline
    \end{tabular}
    \caption{Configuration \ref{def:configuration1}}
    
  \end{subtable}\hfill
  \begin{subtable}[t]{0.48\textwidth}
    \centering
    \small
    \renewcommand{\arraystretch}{1.25}
    \begin{tabular}{|l|c|c|c|}
      \hline
      \diagbox[width=9em,innerleftsep=0cm,innerrightsep=0.2cm]{~Method}{Epochs} & $1$ & $5$ & $10$ \\
      \hline
      NLL & 25.16 & 2.32 & 0.75 \\
      \hline
      \rowcolor{Turquoise!12!white}
      KL Divergence & 10.56 &  10.66 & 11.56 \\
      \hline
      JS Divergence & 11.20 & 11.04 & 12.26 \\
      \hline
    \end{tabular}
    \caption{Configuration \ref{def:configuration2}}
    
  \end{subtable}

  \caption{Comparison of the $L_2$ loss for water depth using two different model configurations. The parameter being recovered is water depth.} 
  \label{tab:water_depth}
\end{table}

\begin{table}[ht!]
  \centering
  \begin{subtable}[t]{0.48\textwidth}
    \centering
    \small
    \renewcommand{\arraystretch}{1.25}
    \begin{tabular}{|l|c|c|c|}
      \hline
      \diagbox[width=9em,innerleftsep=0cm,innerrightsep=0.2cm]{~Method}{Epochs} & $1$ & $5$ & $10$ \\
      \hline
      NLL & 365.26 & 211.71 & 193.99 \\
      \hline
      \rowcolor{Turquoise!12!white}
      KL Divergence &  31.36 & 0.92 & 2.24\\
      \hline
      JS Divergence & 27.40 & 0.61 & 7.26 \\
      \hline
    \end{tabular}
    \caption{Configuration \ref{def:configuration1}}
    
  \end{subtable}\hfill
  \begin{subtable}[t]{0.48\textwidth}
    \centering
    \small
    \renewcommand{\arraystretch}{1.25}
    \begin{tabular}{|l|c|c|c|}
      \hline
      \diagbox[width=9em,innerleftsep=0cm,innerrightsep=0.2cm]{~Method}{Epochs} & $1$ & $5$ & $10$ \\
      \hline
      NLL & 9.33 & 0.73 & 1.91 \\
      \hline
      \rowcolor{Turquoise!12!white}
      KL Divergence &  9.26 & 1.25 & 6.31 \\
      \hline
      JS Divergence & 8.73 & 2.29 & 10.24 \\
      \hline
    \end{tabular}
    \caption{Configuration \ref{def:configuration2}}
      \end{subtable}

  \caption{Comparison of the $L_2$ loss for sound speed using two different model configurations introduced. The parameter being recovered is sound speed.}
  \label{tab:sound_speed}
\end{table}

\begin{table}[ht!]
  \centering
  \begin{subtable}[t]{0.48\textwidth}
    \centering
    \small
    \renewcommand{\arraystretch}{1.25}
    \begin{tabular}{|l|c|c|c|}
      \hline
      \diagbox[width=9em,innerleftsep=0cm,innerrightsep=0.2cm]{~Method}{Epochs} & $1$ & $5$ & $10$ \\
      \hline
      NLL & 4.87 & 16.78 & 30.50 \\
      \hline
      \rowcolor{Turquoise!12!white}
      KL Divergence &  1.37 & 5.44 & 10.91\\
      \hline
      JS Divergence & 1.34 & 5.58 & 11.08 \\
      \hline
    \end{tabular}
    \caption{Configuration \ref{def:configuration1}}
   
  \end{subtable}\hfill
  \begin{subtable}[t]{0.48\textwidth}
    \centering
    \small
    \renewcommand{\arraystretch}{1.25}
    \begin{tabular}{|l|c|c|c|}
      \hline
      \diagbox[width=9em,innerleftsep=0cm,innerrightsep=0.2cm]{~Method}{Epochs} & $1$ & $5$ & $10$ \\
      \hline
      NLL & 4.22 & 14.19 & 26.02 \\
      \hline
      \rowcolor{Turquoise!12!white}
      KL Divergence &  1.71 &  6.76 & 13.36 \\
      \hline
      JS Divergence & 1.86 & 6.64 & 12.45 \\
      \hline
    \end{tabular}
    \caption{Configuration \ref{def:configuration2}}
   
  \end{subtable}

  \caption{Comparison of \textbf{training times} for two model configurations. The configuration used for the left table achieves an \textbf{inference time} of 0.0060 seconds per 1000 samples, while the configuration used for the right table achieves 0.0036 seconds for the same number of samples.}
  \label{table:training_time}
\end{table}

For \ac{GI}, 
standard likelihood-evaluation-based MCMC is computationally intensive, as each likelihood evaluation necessitates solving the forward model with an acoustic propagation simulator KRAKEN~\citep{porter1992kraken}. 
Alternately, a pre-trained \ac{INN} facilitates rapid, near-real-time posterior inference by eliminating the need for repeated forward model evaluations. This speedup (inference time of 0.0060 seconds per 1000 samples) comes from shifting computation offline into training. Posterior sampling  is often computationally challenging, especially in high dimensions~\citep{montanari2023posterior}, and there is no universally best configuration for these problems.
As demonstrated in our case studies, \acp{INN} can be sensitive to architectural and optimization choices; therefore, automated hyperparameter search methods such as the tree-structured Parzen estimator~\citep{bergstra2011algorithms} can be useful in practice.

\vspace{5mm}

\section{Conclusion}
 In this work, we have presented a unified theoretical and practical framework that connects \acp{INN} and \acp{NF} through a common variational formulation, motivated by related works of~\citet{nowozin2016f, zhang2019variational, grover2018flow}. By analyzing these architectures under a shared perspective, we derived new theoretical guarantees on the approximation quality of both posterior and generative distributions under weaker, more realistic assumptions than those used in prior studies. 
 Our findings bridge a key gap in the literature by providing a principled understanding of when and why these models succeed in both generative and inverse problem settings. 
 Beyond the theoretical results, our empirical investigations yield general design principles that can guide practitioners in the effective implementation of \ac{INN}- and \ac{NF}-based systems. Finally, the application of our framework to a ocean-acoustic inversion task demonstrates its practical value and robustness.

\subsection*{Acknowledgements}
M. J. Khojasteh acknowledges support from the Gleason Endowment at RIT and also thanks Prof. Florian Meyer and Dr. Mohsen Sadr for their constructive discussions and suggestions.
 K. Youcef-Toumi acknowledges the support from the Center for Complex Engineering Systems (CCES) at King Abdulaziz City for Science and Technology (KACST) and Massachusetts Institute of Technology (MIT).

\newpage

\bibliographystyle{abbrvnat} %
\bibliography{references}  %

\newpage

\appendix

\section{Additional Background}
\label{appendix:background}

\begin{definition}[Rademacher Complexity]
    \label[definition]{def:rademacher-complexity} Let $\calF$ denote a class of real-valued functions over the domain $\mathbb{R}^d$,  and let $X^n \overset{i.i.d.}{\sim} P_X$ denote an \iid dataset. With $\epsilon^n$ denoting $n$ \iid Rademacher random variables~(i.e., $\pm 1$ w.p. $1/2$ each), we define the Rademacher complexity of the function class $\calF$ as 
    \begin{align}
        \RadComp_n(\calF, X^n) = \mathbb{E}_{\epsilon^n}\lb \sup_{f \in \calF} \frac{1}{n} \sum_{i=1}^n \epsilon_i f(X_i) \middle \vert X^n \rb
    \end{align}
    We will denote the expected Rademacher complexity~(with respect to the data $X^n$) as $\RadComp_n(\calF) = \mathbb{E}_{X^n}[\RadComp_n(\calF, X^n)]$. 
\end{definition}

We now recall a result about vector contraction for Rademacher complexities from~\cite{maurer2016vector}. 
\begin{fact}
    \label{fact:maurer-contraction} For some $d \geq 1$, let $\phi_i:\mathbb{R}^d \to \mathbb{R}$ denote a collection of $J$-Lipschitz functions in $\ell_2$ norm with $\phi_i(\boldsymbol{0})=0$. Then, for a given set of points $\{x_1, \ldots, x_n\}$ in some domain $\calX$, and a function class $\calH$ consisting of $h:\calX \to \mathbb{R}^d$, we have the following: 
    \begin{align}
        \mathbb{E}_{\epsilon^n} \lb \sup_{h \in \calH} \frac 1 n \sum_{i=1}^n  \epsilon_i \phi_i(h(x_i)) \rb \leq  {2 J}  \mathbb{E}_{\sigma^n}\lb \sup_{h \in \calH} \frac 1n  \sum_{i=1}^n \langle \sigma_i, h(x_i) \rangle \rb. 
    \end{align}
    Above, we assume that $(\epsilon_1, \ldots, \epsilon_n)$ are \iid Rademacher, and  each $(\sigma_1, \ldots, \sigma_n)$ represents  \iid $\{-1, 1\}^d$-valued random vectors with \iid Rademacher coordinates. 
\end{fact}
Finally, we now recall a version of a well-known concentration inequality that we will use to obtain high probability uniform deviation bounds. 
\begin{fact}[McDiarmid's Inequality]
    \label{fact:mcdiarmid}
    Let $\{X_i: 1 \leq i \leq n\}$ denote a stream of independent $\calX$-valued observations, and let $h:\calX^n \to \mathbb{R}$ be a function satisfying a bounded difference property: 
    \begin{align}
        |h(x_1, \ldots, x_{i-1}, x_i, x_{i+1}, \ldots, x_n) -  
        h(x_1, \ldots, x_{i-1}, x_i', x_{i+1}, \ldots, x_n)| \leq c, 
    \end{align}
    for all $i \in [n]$,  $x^n \in \calX^n$ and $x'_i \in \calX$. Then we have the following for all $\epsilon>0$: 
    \begin{align}
        \mathbb{P}\lp h(X^n) - \EE[h(X^n)] >n \epsilon \rp \leq \exp \lp - \frac{2n \epsilon^2}{c^2} \rp.  
    \end{align}
\end{fact}

\subsection{Discussion of Invertible Neural Architectures}
\label{sec:background}
This section outlines the key factors that ensure practical invertibility in \acp{INN}, which are formally defined in Section~\ref{sec:introduction}, including architectural structures, bidirectional training, and padding schemes.

\noindent\textbf{Model Architecture:}
A key aspect in designing \acp{INN} is constructing architectures that are both expressive and invertible.
One of the most popular design is constructed by concatenating  affine coupling
blocks~\cite{kingma2018glow,dinh2016density,dinh2014nice}.
In this case, invertible architecture $T$  is defined by a sequence of reversible blocks,  where %
each block consists of two complementary affine coupling layers. The the block’s input $\mathbf{u} \in \mathbb{R}^{d_{\mathbf{u}}}$                   is split into $\mathbf{u}_1 \in \mathbb{R}^{d_{\mathbf{u}_1}}$ and $\mathbf{u}_2 \in \mathbb{R}^{d_{\mathbf{u}} - d_{\mathbf{u}_1}}
$.  Each transformation stage applies a learned mapping to a subset of the input features, while the remaining features are left unchanged, ensuring invertibility of the overall transformation.  In other words, we have 
\begin{align}
    \begin{bmatrix}
\mathbf{v}_1\\[\smallskipamount]
\mathbf{v}_2
\end{bmatrix} = 
\begin{bmatrix}
\mathbf{u}_1 \odot \exp\scalebox{1.1}{$($}s_1(\mathbf{u}_2)\scalebox{1.1}{$)$} + t_1(\mathbf{u}_2) \\[\smallskipamount]
\mathbf{u}_2
\end{bmatrix}
,\qquad\quad %
    \begin{bmatrix}
\mathbf{o}_1\\[\smallskipamount]
\mathbf{o}_2
\end{bmatrix} = 
\begin{bmatrix}
\mathbf{v}_1 \\[\smallskipamount]
\mathbf{v}_2 \odot \exp\scalebox{1.1}{$($}s_2(\mathbf{v}_1)\scalebox{1.1}{$)$} + t_2(\mathbf{v}_1) 
\end{bmatrix}
\end{align}
The mappings  $s_i$    and   $t_i$  for $i=1,2$ are arbitrarily  neural network. Each block can be inverted, %
that is, given the output $\mathbf{o} = [\mathbf{o}_1, \mathbf{o}_2]
$                         the inverse can be calculated as 
\begin{align}
    \mathbf{u}_2 = \big(\mathbf{o}_2 - t_2(\mathbf{o}_1)\big) \oslash \exp \scalebox{1.1}{$($} s_2(\mathbf{o}_1) \scalebox{1.1}{$)$}
\end{align}
and
\begin{align}
    \mathbf{u}_1 = \big(\mathbf{o}_1 - t_1(\mathbf{u}_2)\big) \oslash  \exp \scalebox{1.1}{$($} s_1(\mathbf{u}_2) \scalebox{1.1}{$)$}
\end{align}
Note that when the coupling block is inverted, the subnetworks    $s_i$    and   $t_i$     need not to be invertible
they are only evaluated in the forward direction.
Also, this structure yields a triangular Jacobian, making the log-determinant
computationally tractable. %

Another popular approach to achieve invertibility is using residual connections (iResNet models)~\cite{gomez2017reversible,jacobsen2018revnet,behrmann2019invertible}. In this way, the mapping $T$ is constructed by the successive composition of residual blocks, $T =  (I + F_M) \circ (I + F_{M-1})\circ \cdots \circ (I + F_1)$, where $``\circ"$ denotes the composition operation. A sufficient condition for invertibility of each block is that the residual sub-network has a Lipschitz constant
less than one. However, unlike coupling-based methods that offer exact solutions for the inverse and the log-determinant of the Jacobian, they need to be approximated in the case of  iResNet models~\cite{chen2019residual}.  

In summary, coupling layers provide exact inversion and efficient Jacobian computation, but their fixed partitioning limits flexibility and can affect stability. Invertible residual networks support more flexible and expressive architectures, needing only a Lipschitz constraint for stability. However, they have less efficient inversion and Jacobian computation than coupling layers.

\noindent\textbf{Bi-Directional Training:}
Invertible networks allow applying losses in both the input and output domains, since the mapping can be evaluated in both directions. During training, one may alternate between forward and inverse passes, accumulating gradients from both directions. This bidirectional loss formulation can improve the effectiveness and stability of training~\cite{ardizzone2018analyzing}.In the forward direction, as discussed in~\eqref{eq:forwardcostINN}, the \ac{SL} $L_\yy(T)$ and \ac{USL} $L_\zz(T)$ will be minimized. Similarly, in the backward direction   an \ac{USL} $L_\xx(T)$ can be considered. 
For  the \ac{SL} we consider mean square error, and the \ac{USL} is discussed in the next section.

\noindent\textbf{Padding:}
Padding adjusts input dimensionality to meet a model's structural requirements~\cite{ardizzone2018analyzing}. For \acp{INN}, it resolves mismatches by adding  entries enabling bijective mappings. Common strategies like zero-padding and repetition-padding keep forward and inverse passes dimensionally consistent and can support stable learning.
For instance, if zero-padding is used on either side of the network, additional loss terms, called reconstruction loss, are required to ensure no information is encoded in the padding dimensions and forcing these dimensions to remain inactive.

\subsection{Negative Log-Likelihood (NLL)}
\label{sec:NLL}

For a  map $T(X;\theta)\mapsto [Y,Z]$ parameterized by $\theta$, and assuming $Y$ and $Z$ are independent, the change-of-variables formula implies that the density $P_{T^{-1}(Y,Z)}$  is
\begin{align}
\label{eq:change_var_4_KL}
P_{T^{-1}(Y,Z)}= P_{T_{\yy}(X)} \, P_{T_{\zz}(X)} \cdot  \left|\det\left(\Jac_{X \,\mapsto\,[Y,Z]}(X)\right)\right|
\end{align}
where $\Jac_{X \,\mapsto\,[Y, Z]}(X;\theta)$ denotes the Jacobian of the map $T$ parameterized by $\theta$. As described next, this expression can be used to define an unsupervised training loss.
In particular, we aim to minimize the forward KL divergence between the true posterior distribution $P_{X|Y}$ and $P_{T^{-1}(Y,Z)}$, given by
\begin{align}
\mathcal{L}(\theta) &= D_{\text{KL}}\big(P_{X|Y}, P_{T^{-1}(Y,Z)}\big)\notag
= - \mathbb{E}_{P_{X|Y}}\big[\log P_{T^{-1}(Y,Z)}\big] + \text{const.}\notag\\
&= - \mathbb{E}_{P_{X|Y}}\big[
\log P_{T_{\yy}(X)} + 
\log P_{T_{\zz}(X)} + \log\left|\mbox{det}\big(\Jac_{X \,\mapsto\,[Z,Y]}(X)\big)\right|
\big] + \text{const.}\label{eq:kl_div}
\end{align}
The empirical approximation of the above loss is as follows.
\begin{align}
\label{eq:KL_forward_app}
\hat{\mathcal{L}}(\theta) = - \frac{1}{n}\sum_{i=1}^n \Big( 
\log P_{T_{\yy}(X_i)} + 
\log P_{T_{\zz}(X_i)} + \log\left|\mbox{det}\big(\Jac_{X \,\mapsto\,[Y,Z]}(X_i)\big)\right|
+ \text{const.}\Big) 
\end{align}

\vspace{1mm}
As we can see, minimizing the above Monte Carlo approximation of the KL divergence is equivalent to maximizing likelihood (or minimizing negative log-likelihood). Assuming $P_{Z}$ is standard Gaussian and $P_{Y}$ is a multivariate normal distribution around $\mathbf{y}_\mathrm{gt}$  and a small standard deviation $\sigma$, \ac{NLL} loss in Eq.~(\ref{eq:KL_forward_app}) becomes
\begin{align}
\mathcal{L}_\mathrm{NLL}(\theta) &= \frac{1}{n}\sum_{i=1}^n \left( 
\frac{1}{2}\,\frac{\big(T_{\yy}(X_i) - y_{\text{gt}}\big)^2}{\sigma^2} + \frac{1}{2}\, T_{\zz}(X_i)^2 - \log\left|\mbox{det}\big(\Jac_{X \,\mapsto\,[Y,Z]}(X_i)\big)\right|\right) \label{eq:nll_new}
\end{align}
 Note that we started with the backward loss between $P_{X|Y}$ and $P_{T^{-1}(Y,Z)}$ and end up with a supervised loss and unsupervised loss in the forward direction.  Also, the Gaussian assumption on
 $Y$ can be restrictive and as discussed in Section~\ref{sec:related-work}, we eliminate the need for such distributional assumptions by relying solely on samples and leveraging the variational formulation of f-divergences, which yields a lower bound on the true value which is discussed more in Section~\ref{sec:observation-2-f-div}.
\section{Variational normalizing flow~(V-NF)}
\label{subsec:variational-nf}
In this section, we introduce a class of normalizing flows trained using a variational objective. To simplify our presentation, we focus on the concrete example of the variational form of relative entropy, although, similar arguments can be developed for a larger class of $f$-divergences. Our discussion in this section closely follows the structure of~\Cref{subsec:variational-inn} for~\acp{INN}, and in fact some of the results we present can be directly inferred from the results of~\Cref{subsec:variational-inn}. However, we include all the details to keep this section self-contained and independent of~\Cref{subsec:variational-inn}.

Let $\calX = \mathbb{R}^d$ denote the observation space and let $\calZ = \mathbb{R}^d$ denote the latent space in which the latent variable $Z \sim P_Z$ takes its values. Let $\calT:\calX \to \calZ$ denote a class of {  diffeomorphisms\footnote{Our theoretical guarantees in~\Cref{theorem:nf-estimation-error} are also valid for homeomorphisms, but due to the specific requirements during the training of \ac{NF}, here we state the definition based on diffeomorphisms.}}, and for any $T \in \calT$, define the model distribution  on $\calX$ as 
\begin{align}
    Q_T = (T^{-1})_{\#}P_Z \quad \iff \quad Q_T(E) = P_Z(\{z \in \calZ:T^{-1}(z) \in E\}), \; \text{for any measurable } E \subset \calX. 
\end{align}
In other words, for $Z \sim P_Z$, the model $T$ represents a random variable $X_T  = T^{-1}(Z) \sim Q_T$ with density $q_T(x) = p_Z(T(x)) \mathrm{det} \Jac T(x)$, with $\Jac T$ denoting the associated Jacobian. 

\paragraph{Training procedure.} As mentioned in~\Cref{sec:general-framework},  we consider a variational objective for training \acp{NF}, and in particular, in this section we focus on the concrete case of forward relative entropy loss by using a function class $\calG_n$~(critic or witness function class that we allow to grow with $n$).  Formally, the population loss that will be used to train a model is 
\begin{align}
    \LNF(T) =  \sup_{g \in \calG_n} \lbr \mathbb{E}_{X \sim P_X}[\log g(X)] - \mathbb{E}_{X_T \sim Q_T}[g(X_T)]  \rbr + 1.  
\end{align}
Here we have used the variational definition of relative entropy presented in~\cite{nguyen2010estimating}, and by definition,  $\LNF(T) \leq \dkl(P_X \parallel Q_T)$ with equality if the true likelihood ratio $dP_X/dQ_T$ is contained in $\calG_n$.
Given $n$ \iid draws $X_1, \ldots, X_n$ from the data distribution $P_X$, and with $m_n$ independent latent random variables $Z_1, \ldots, Z_{m_n} \overset{\iid}{\sim} P_Z$, we define the empirical loss associated with $T \in \calT$ as 
\begin{align}
    \LhatNF_n(T) = \sup_{g \in \calG_n} \lbr \frac{1}{n} \sum_{i=1}^n \log g(X_i) - \frac{1}{m} \sum_{i=1}^{m_n} g(T^{-1}(Z_i)) \rbr + 1. 
\end{align}
Observe that here we use $n$ to denote the size of the true dataset that we are trying to model, while $m_n$ denotes the number of latent data points we use. In general $m \equiv m_n$ can be much larger than the $n$, as the latent random variables are much easier to generate~(a common choice is to select $P_Z$ to be a known multivariate Gaussian distribution).  
We can now define the empirical \ac{NF} model $\That_n$ as 
\begin{align}
    \That_n  \in \argmin_{T \in \calT} \LhatNF_n(T).  \label{eq:erm-NF}
\end{align}
The performance of the \ac{ERM} estimator with a variational objective defined above is governed by two factors: the sampling error, and the alproximation error due to the restriction of the critic class to $\calG_n$. Classical results from empirical process theory can allow us to control the difference between $\LNF$ and $\LhatNF$ in terms of certain notions of capacity of the class $\calG_n$. Informally, smaller $\calG_n$ will lead to much better finite sample approximation. On the other hand, reducing the difference $\dkl(P_X \parallel Q_T)$, that we refer to as the \emph{variational gap}, requires a larger $\calG_n$.  This represents the key challenge in our analysis: finding the right tradeoff for the size of the critic class $\calG_n$ that simultaneously controls these two terms. We first present a general result that assumes the existence of such a sequence of $\{\calG_n:n \geq 1\}$ that satisfies the conditions stated in~\Cref{assump:NF-theorem}. Then, in~\Cref{subsubsec:interpretable-conditions-nf}, we identify sufficient that are easier to verify for~\Cref{assump:NF-theorem} to hold, and finally in~\Cref{subsubsec:variational-nf-example}, we discuss a concrete realistic example satisfying these conditions in~\Cref{subsubsec:variational-nf-example}. 

\begin{assumption}
\label{assump:NF-theorem}
To analyze the \ac{NF} model $\That_n$ defined in~\eqref{eq:erm-NF}, we require the following assumptions: 
\begin{itemize}
    \item \textbf{(NF1): Realizability.} For the $P_X, P_Z$ under consideration, there exists a $T^* \in \calT$ with $Q_{T^*} = P_X$, which implies that $\LNF(T^*) = 0$ assuming that $\boldsymbol{1} \in \calG_n$~(note that we allow the critic class $\calG_n$ to grow with $n$, keeping $\calT$ fixed). 
    \item \textbf{(NF2): Uniform Convergence.} There exists a deterministic sequence $r_n \to 0$ and a confidence level $\delta \in (0, 1)$, such that with probability $1-\delta$, we have 
    \begin{align}
        \sup_{T \in \calT} |\widehat{L}_{NF}(T) - L_{NF}(T)| \leq r_n. 
    \end{align}
    In other words, we assume that the function classes $(\calG_n, \calT)$ are small enough to ensure uniform learnability. 
    
    \item \textbf{(NF3): Variational Approximation Gap.} $P_X \ll Q_T$ for all $T \in \calT$, and there exists a deterministic sequence $\eta_n$ converging to $0$ as $n \to \infty$, such that 
    \begin{align}
        \sup_{T \in \calT} \dkl(P_X \parallel Q_T) -  L_{NF}(T) \leq \eta_n. 
    \end{align}
    In other words, we assume that the capacity of the critic class $\calG_n$ grows with $n$ to approximate the likelihood ratios of all distributions modeled by elements of $\calT$ and the true data distribution $P_X$. 
    \item \textbf{(NF4): Moment Bounds.} There exists  positive constants $R$ and $a$, such that 
    \begin{align}
        \max \lbr \mathbb{E}[\|X\|^{1+a}], \; \sup_{T \in \calT} \mathbb{E}[\|T^{-1}(Z)\|^{1+a}] \rbr \leq R. 
    \end{align}
\end{itemize}
\end{assumption}
With these assumptions in place, we can now state the main result of this section.
\begin{theorem}
    \label{theorem:nf-estimation-error}
    Under~\Cref{assump:NF-theorem}, the empirical risk minimizer~$\That_n$ defined in~\eqref{eq:erm-NF} satisfies the following with probability at least $1-\delta$: 
    \begin{align}
        W_1(P_X, Q_{\That_n}) 
        \lesssim (r_n+\eta_n)^{\frac a {2(1 + a)}},  
    \end{align}
    where $\lesssim$ suppresses a multiplicative constant that depends on $a$, and $W_1$ is the $1$-Wasserstein metric.  
\end{theorem}
The proof of this statement is given in~\Cref{proof:nf-estimation-error}. This theorem obtains a H\"older-type  transfer from the original empirical variational objective to the Wasserstein error. In particular, it says that the quality of approximation achieved by the \ac{NF} model $\That_n$ depends on the sum of uniform empirical-process deviation $r_n$ and the variational gap $\eta_n$. The exponent $a/(1+a)$ encodes the tail behavior from~\Cref{assump:NF-theorem}: distributions with lighter tails~(or larger $a$) yield a stronger, near-$\sqrt{\cdot}$ relation between $W_1$ and $r_n + \eta_n$, while heavier tails cause a weaker transformation. 

\begin{remark}
    \label{remark:extension-to-other-f-divergences} 
    The proof of~\Cref{theorem:nf-estimation-error} proceeds in three steps. (i) On the high-probability event from~\NFconvergence, ERM and realizability assumption imply that $\LNF(\That_n) \leq 2r_n$. Combining this with the variational gap assumption~\NFvariational yields $\dkl(P_X \parallel Q_{\That_n}) \leq 2r_n + \eta_n$. (ii) The next step is to use Pinsker's inequality to get $\|P_X - Q_{\That_n}\|_{TV} \leq \sqrt{(1/2) (2r_n + \eta_n)}$. (iii) Finally, we employ the truncation argument formalized in~\Cref{lemma:truncation} along with the bounded $(1+a)$-moment assumption~\NFmoments to get the required bound in $W_1$ metric. 
    
    As this outline suggests, this exact argument goes through for any distance or divergence  that dominates the total variation metric via a Pinsker-type inequality. Hence,~\Cref{theorem:nf-estimation-error} can be stated more generally for such a class divergence measures, which includes the Jensen Shannon,  Hellinger, and chi-squared divergence. 
\end{remark}

\subsection{Verifiable Sufficient Conditions for~\Cref{assump:NF-theorem}}
\label{subsubsec:interpretable-conditions-nf}
In this section, we work under the realizability assumption~\NFrealizability, and derive sufficient conditions on $(\calT, \calG_n)$ for the other conditions to hold, such that the $W_1$ metric between the model distribution and the true distribution converges to zero as the sample size $n$ increases. 

We begin with the uniform convergence condition~\textbf{(NF2)}, and identify sufficient conditions in terms of the Rademacher complexities of the function classes involved. 
\begin{proposition}
    \label[proposition]{prop:nf2-uniform-convergence-verification} 
    Let $(\epsilon_i)_{i \geq 1}$ denote \iid Rademacher random variables, and with $\calH_n$ denoting   the function class $\{\log g: g \in \calG_n\}$, define  the expected Rademacher complexities
    \begin{align}
        &\mathfrak{R}_n(\calH_n) = \mathbb{E}_{\epsilon^n, X^n}\lb \sup_{h \in \calH_n} \frac{1}{n} \sum_{i=1}^n \epsilon_i h(X_i)   \rb, \qtext{and}  
        \mathfrak{R}_{m_n}(\calH_n, \calT) = \mathbb{E}_{\epsilon^{m_n}, Z^{m_n}}\lb \sup_{h, T} \frac{1}{m_n} \sum_{i=1}^n \epsilon_i h(T^{-1}(Z_i))  \rb. 
    \end{align}
    If $\|h\|_\infty \leq b_n$ for all $h \in \calH_n$,  for any $\delta >0$, we have the following: 
    \begin{align}
        \sup_{T \in \calT} \lv \LhatNF_n(T) - \LNF(T) \rv \leq 2 \mathfrak{R}_n(\calH_n) + 2 e^{b_n}  \mathfrak{R}_{m_n}\lp \calH_n, \calT \rp + (\frac{b_n}{\sqrt{n}} + \frac{e^{b_n}}{\sqrt{m_n}}) \sqrt{2 \log(4/\delta)}  
    \end{align}
    Thus, to verify~\textbf{(NF2)}, we need to identify uniformly bounded $\calH_n$~(equivalently $\calG_n$) whose expected Rademacher complexity can be appropriately controlled to drive the above term to $0$ with $n$. 
\end{proposition}
This result follows from the standard symmetrization via Rademacher variables argument, and we present the details in~\Cref{proof:nf2-uniform-convergence-verification}. 
We next observe that if the model class $\calT$ satisfies a uniform Lipschitz property, then the uniform moment control assumption~\textbf{(NF4)} follows easily. 
\begin{proposition}
    \label[proposition]{prop:nf4-moment-control}
    Suppose the function class $\calT$ satisfies the following uniform affine growth property: 
    \begin{align}
        \sup_{T \in \calT}  \|T^{-1}(z)\| \leq J_0 + J_1 \|z\| \qtext{for all} z \in \calZ. 
    \end{align}
    If $\mathbb{E}_{P_Z}[\|Z\|^{1+a}] = M_Z<\infty$ for an $a>0$, we have the following under the realizability assumption $(NF1)$:
    \begin{align}
        \max\lbr \mathbb{E}[\|X\|^{1+a}], \, \sup_{T \in \calT} \mathbb{E}\lb \|T^{-1}(Z)\|^{1+a} \rb \rbr \leq 2^a\lp J_0^{1+a} + J_1^{1+a} M_Z \rp \eqcolon R. 
    \end{align}
\end{proposition}
The proof of this result is exactly the same as that of~\Cref{prop:inn3-uniform-moment-bound}, and we omit it to avoid repetition. 
In other words, assuming realizability, a sufficient condition for \textbf{(NF4)} is if the latent variable $Z$ has finite $(1+a)$ moments, and the model class is globally Lipschitz and uniformly bounded at the origin. 
Finally, we proceed to the condition of vanishing variational approximation gap~\textbf{(NF3)}. 
\begin{proposition}
    \label[proposition]{prop:nf3-variational-gap} Suppose $\dkl(P_X \parallel Q_T)< \infty$ for all $T \in \calT$, and let $v_T = dP_X/dQ_T$ and $\ell_T = \log v_T$ denote the likelihood ratio and log-likelihood ratio (resp.) of $P_X$ and $Q_T$. 
    Fix some sequence $(K_n)_{n \geq 1}$ such that $K_n \overset{n \to \infty}{\longrightarrow} \infty$, and define the following uniform tail expectation terms, with $X \sim P_X$, $X' \sim Q_T$, and $\calH_n = \{\log g: g \in \calG_n\}$:  
    \begin{align}
        &\tau_1 \equiv \tau_1(K_n, \calT, \calH_n) = \sup_{T \in \calT} \sup_{h \in \calH_n} \mathbb{E}_{P_X}[\lv \ell_T(X) - h(X) \rv \boldsymbol{1}_{\|X\|_\infty > K_n}], \\
        \text{and} \quad &\tau_2 \equiv \tau_2(K_n, \calT, \calH_n) = \sup_{T \in \calT} \sup_{h \in \calH_n} \mathbb{E}_{Q_T}\lb \lv e^{h(X')} - e^{\ell_T(X')}   \rv \boldsymbol{1}_{\|X'\|_\infty > K_n} \rb. 
    \end{align}
    Next, introduce a uniform approximation error term computed over a restricted domain: 
    \begin{align}
        &\delta_n \equiv \delta_n(K_n, \calT, \calH_n) = \sup_{T \in \calT} \inf_{h \in \calH_n} \sup_{x:\|x\|_\infty \leq K_n} |h(x) - \ell_T(x)|. 
    \end{align}
    Suppose there exists a sequence of $\{K_n: n \geq 1\}$ with  $K_n \uparrow \infty$, such that 
    \begin{align}
        \lim_{n \to \infty} \lbr     \delta_n + \tau_1 + \tau_2 \rbr = 0.  \label{eq:nf3-sufficient-condition}
    \end{align}
    Then the variational gap $\eta_n$ also converges to zero; that is, if \eqref{eq:nf3-sufficient-condition} holds, then 
    \begin{align}
         \eta_n &= \sup_{T \in \calT} \lbr \dkl(P_X \parallel Q_T) - \LNF(T) \rbr  \;\stackrel{n \to \infty}{\longrightarrow}\;  0. 
    \end{align}
\end{proposition}
The proof of this result is in~\Cref{proof:nf3-variational-gap}.

\subsection{V-NF Example}
\label{subsubsec:variational-nf-example}

We now illustrate that the assumptions required by~\Cref{theorem:nf-estimation-error} are satisfied by practically useful models,  by constructing specific $(\calT, \calG_n)$, and verifying this assumptions. 
Our model class is an iResNet flow with uniformly controlled Jacobians that ensures certain useful global regularity properties. Our critic class is induced by a Gaussian \ac{RKHS} restricted to a growing cube $C_{K_n} = [-K_n, K_n]^d$, which allows a neat separation between (i) approximation on a compact set, and (ii) tail control outside the compact set, as needed by~\Cref{prop:nf3-variational-gap}. 
\begin{definition}
    \label[definition]{def:calT-nf} We will work with i-ResNets of the form $T =  (I_d + F_M) \circ (I_d + F_{M-1})\circ \cdots \circ (I_d + F_1)$, where $``\circ"$ denotes the composition operation, and 
    \begin{align}
        &F_j(x) = W_{j,2} \tanh\lp W_{j,1}x + b_{j,1} \rp + b_{j,2},  \\
        \qtext{with} &\max\lbr \|W_{j,1} \|_{op}, \|W_{j,2}\|_{op}\rbr \leq s \in (0, 1), \qtext{and} \max \{\|b_{j,1}\|, \; \|b_{j,2}\|\} \leq B, \quad \;\forall j \in [M]. 
    \end{align}
    Each $W_{j,1}$ lies in $\mathbb{R}^{d_j} \times \mathbb{R}^d$ while $W_{j,2}$ lies in $\mathbb{R}^{d} \times \mathbb{R}^{d_j}$, with $d_j$ denoting the dimension of the hidden layer in block $j$. Similarly, $b_{j,1} \in \mathbb{R}^{d_j}$ and $b_{j,2} \in \mathbb{R}^{d}$. Let $H$ denote the maximum value of the hidden layer dimension; that is, $H = \max_{1 \leq j \leq M} d_j$, $\|\cdot\|_{op}$ denotes the operator norm, and $\|\cdot\|$ denotes the $\ell_2$ norm. 
\end{definition}
Next, we introduce the critic function class $\calG_n$. 
\begin{definition}
\label[definition]{def:calG-nf}
Let $k_n(x,y) \equiv k_{\gamma_n}(x, y) = \exp \lp - \gamma_n \|x-y\|^2 \rp$ denote a Gaussian kernel with scale parameter $\gamma_n>0$, and let $\calH_{k_n}$ denote the  reproducing kernel Hilbert space~(RKHS) associated with this kernel. Then, we defile the critic class $\calG_n$ as 
\begin{align}
    \calG_n = \{e^h: h \in \calH_n\}, \qtext{where} \calH_n = \{h \boldsymbol{1}_{[-K_n,K_n]^d}:  h \in \calH_{k_n}, \;\text{with} \; \|h\|_{k_n}\leq b_n \} \qtext{for some} b_n, K_n > 0. 
\end{align}
\sloppy By the reproducing property, it follows that $\|h\|_{k_{\gamma_n}} \leq b_n$ implies $\|h\|_\infty = \sup_{ x\in \calX} \langle h, k_{\gamma_n}(x, \cdot) \rangle_{k_{\gamma_n}} \leq \sup_{x \in \calX} \sqrt{k_{\gamma_n}(x,x)} \|h\| \leq b_n$. As a result, every $g \in \calG_n$ satisfies 
    \begin{align}
       0 < e^{- b_n} \leq g(x) \leq e^{b_n}.  
    \end{align}
\end{definition}
As mentioned above, we work with the i-ResNet model class $\calT$ in~\Cref{def:calT-nf} since this satisfies a global Lipschitz continuity condition that allows us to work with unbounded latent random variables $Z \sim N(0, I_d)$. This global Lipschitz property turns out to be useful in establishing that the $(\calT, \calG_n)$ pair of~\Cref{def:calT-nf} and~\Cref{def:calG-nf} satisfy the conditions required by~\Cref{assump:NF-theorem}. We can also ensure these assumptions are satisfied for the coupling-based architecture~\cite{dinh2016density} by restricting the latent distribution $P_Z$ to be supported on a compact domain. We now state the main result of this section. 
\begin{theorem}
    \label{theorem:nf3-example}
    Let $(\calT, \calG_n)$ be as in the two definitions above, and define the parameters 
    \begin{align}
        K_n &= M (s \sqrt{H} + B) + \sqrt{d} + \sqrt{2 \log n}, \\
        \gamma_n &= K_n^{2 + \epsilon} \quad \text{for some fixed } \epsilon>0, \\
        b_n &= C_b \gamma_n^{d/4} K_n^{2 + d/2} = C_b K_n^{2 + d + \frac{d}{4} \epsilon}, \\
        m_n &= \lceil e^{4 b_n} \rceil. 
    \end{align}
    Then, under the realizability assumption \textbf{(NF1)}, and with these choices of the parameters, the ERM model~$\That_n$ defined in~\eqref{eq:erm-NF} satisfies 
    \begin{align}
        W_1(P_X, Q_{\That_n}) = o(1), \quad \text{w.p. at least } 1-\delta.  \label{eq:nf-theorem-main-result}
    \end{align}
    In other words, with these parameters, the models in~\Cref{def:calT-nf} satisfy the sufficient conditions for \textbf{(NF2)}-\textbf{(NF4)} to hold, as derived in~\Cref{subsubsec:interpretable-conditions-nf}. 
\end{theorem}
The proof of this result is in~\Cref{proof:nf3-variational-gap-example}.

\subsection{Proof of~\Cref{theorem:nf-estimation-error}}
\label{proof:nf-estimation-error}

Introduce the uniform approximation event $\calE_n = \{\sup_{T \in \calT}|\widehat{L}_{NF}(T) - L_{NF}(T)| \leq r_n\}$, and note that under the assumption \textbf{(NF2)}, we have $\mathbb{P}(\calE_n) \geq 1-\delta$. For the rest of this proof, we will work under the event $\calE_n$. 

By definition of ERM model, we knlw that for any $T \in \calT$, we have $\widehat{L}_{NF}(\That_n) \leq \widehat{L}_{NF}(T)$. This implies the following chain: 
\begin{align}
    L_{NF}(\That_n) & \leq \widehat{L}_{NF}(\That_n) + r_n && (\text{event } \calE_n) \\
    & \leq \widehat{L}_{NF}(T^*) + r_n && (\text{ERM property}) \\
    & \leq (L_{NF}(T^*) + r_n) + r_n && (\text{event } \calE_n) \\
    & = 2r_n && (\text{assumption \textbf{(NF1)}}). 
\end{align}
Now, by the variational gap assumption~\textbf{(NF3)}, we know that 
\begin{align}
    \left\lvert \dkl(P_X \parallel Q_{\That_n}) - L_{NF}(\That_n) \right\rvert &\leq \sup_{T \in \calT} \left\lvert \dkl(P_X \parallel Q_T) - L_{NF}(T) \right\rvert \leq \eta_n. 
\end{align}
Thus, combining the two displays above, we obtain the following bound on the relative entropy between the true and the \ac{NF} distributions, under the $1-\delta$ probability event $\calE_n$: 
\begin{align}
    \dkl(P_X \parallel P_{\That_n}) \leq \eta_n + 2r_n \quad \stackrel{\text{Pinsker's}}{\implies} \quad  
    \|P_X - P_{\That_n}\|_{TV} \leq \sqrt{\frac{1}{2} (2r_n + \eta_n)}. \label{eq:nf-proof-1}
\end{align}
We can now conclude the proof with the following chain: 
\begin{align}
    W_1(P_X, P_{\That_n}) \lesssim \lp \|P_X - P_{\That_n}\|_{TV}\rp^{\frac{a}{1+a}} \lesssim (2r_n + \eta_n)^{\frac{a}{2(1+a)}}. 
\end{align}
Here the first inequality~(modulo constants) is due to~\Cref{lemma:truncation}, and the second inequality is by~\eqref{eq:nf-proof-1}.

\subsection{Proof of~\Cref{prop:nf2-uniform-convergence-verification}} 
\label{proof:nf2-uniform-convergence-verification}
This follows from the standard uniform convergence results for bounded function classes. Observe that 
\begin{align}
    \sup_{T \in \calT} |\LhatNF(T) - \LNF(T)| &\leq \sup_{h \in \calH_n} |(\empPX_n - \PX) h| + \sup_{T \in \calT} \sup_{h \in \calH_n} |(\empPZ_{m_n} - \PZ) e^h \circ T^{-1}| \\
    & \coloneqq \rmD_X + \rmD_Z. 
\end{align}
Note that above we have used the notation $\empPX_n h = \frac{1}{n} \sum_{i=1}^n h(X_i)$, $\PX h = \mathbb{E}_{X \sim P_X}[h(X)]$, $\empPZ_{m_n} e^h\circ T^{-1} = \frac{1}{m_n} \sum_{i=1}^{m_n} e^{h(T^{-1}(Z_i))}$ and $\PZ e^h \circ T^{-1} = \mathbb{E}_{Z \sim P_Z}[e^{h(T^{-1}(Z))}]$.  

Now observe that by the classical symmetrization technique, we have the following upper bound on the expected value of $\rmD_X$~\cite[Lemma 26.2]{shalev2014understanding}:  
\begin{align}
    \mathbb{E}[\rmD_X] \leq 2 \RadComp_n(\calH_n) = 2 \mathbb{E}\lb \RadComp_n(\calH_n, X^n) \rb %
\end{align}
Here, $\RadComp_n$ denotes the Rademacher complexity whose definition is recalled in~\Cref{def:rademacher-complexity} in~\Cref{appendix:background}. 
Since we have assumed that $\|h\|_{\infty} \leq b_b$, for all $h \in \calH_n$, an application of the bounded difference concentration inequality~(recalled in~\Cref{fact:mcdiarmid} in~\Cref{appendix:background}) immediately implies with some constant $C>0$: 
\begin{align}
    \mathbb{P}_{X^n}\lp \rmD_X \geq \mathbb{E}[\rmD_X] +  C b_n\sqrt{\frac{\log(2/\delta)}{n}} \rp 
    &\leq \mathbb{P}_{X^n}\lp \rmD_X \geq 2 \RadComp_n(\calH_n) +  C b_n\sqrt{\frac{\log(2/\delta)}{n}} \rp  
    \leq  \frac{\delta}{2}. \label{eq:nf2-proof-3}
\end{align}

A similar argument works for the term $\rmD_Z$. In particular, we know that the function class $\{e^{h \circ T^{-1}}: h \in \calH_n, T \in \calT\}$ is uniformly bounded from above by $e^{b_n}$. Thus, we can again get a high probability upper bound using the expected Rademacher complexity and the bounded difference inequality: 
\begin{align}
    \mathbb{P}_{Z^{m_n}}\lp \rmD_Z \geq \mathbb{E}[\rmD_Z] +  C e^{b_n}\sqrt{\frac{\log(2/\delta)}{m_n}} \rp &\leq \mathbb{P}_{Z^n}\lp \rmD_Z \geq 2\RadComp_{m_n}(\calG_n \circ \calT^{-1}) +  C e^{b_n}\sqrt{\frac{\log(2/\delta)}{m_n}} \rp  \leq  \frac{\delta}{2}. \label{eq:nf2-proof-4}
\end{align}
Together,~\eqref{eq:nf2-proof-4} and~\eqref{eq:nf2-proof-3}, along with the observation that $\RadComp_{m_n}(\calG_n \circ \calT^{-1}) \leq e^{b_n} \RadComp_{m_n}(\calH_n, \calT)$ by the contraction result, imply the required uniform convergence bound with probability at least $1-\delta$: 
\begin{align}
    \rmD_X + \rmD_Z \leq 2 \RadComp_n(\calH_n)  + 2\RadComp_{m_n}(\calH_n, \calT) + C\sqrt{\log(2/\delta)} \lp \frac{b_n}{\sqrt{n}} + \frac{e^{b_n}}{\sqrt{m_n}} \rp \eqcolon r_n. 
\end{align}
Thus, in order to establish a uniform convergence guarantee, it suffices to find a sequence of $\{(n, m_n, b_n, \calH_n): n \geq 1\}$, such that the term $r_n$ in  RHS above converges to $0$. \hfill %
\subsection{Proof of~\Cref{prop:nf3-variational-gap}}
\label{proof:nf3-variational-gap}
Throughout this proof, we assume that $X \sim P_X$ and $X' \sim Q_T$ for $T \in \calT$, where $Q_T$ represents the distribution of $T^{-1}(Z)$ for some latent variable $Z \sim P_Z$. Our argument below does not rely on how $X'$ is generated, or on any properties of $P_Z$.  
Now, observe that for any $T \in \calT$, the relative entropy $\dkl(P_X \parallel Q_T)$ is equal to $\mathbb{E}_{P}[\ell_T(X)]$, which implies that 
\begin{align}
    \Delta_n(T) &\coloneqq \dkl(P_X \parallel Q_T) - \sup_{h \in \calH_n} \lbr \mathbb{E}_{P_X}[h(X)] - \mathbb{E}_{Q_T}[e^{h(X')}] + 1 \rbr \\
    &=  \inf_{h \in \calH_n} \lbr \mathbb{E}_{P_X}[\ell_T(X) - h(X)] + \mathbb{E}_{Q_T}[e^{h(X')} - e^{\ell_T(X')}]  \rbr,  \label{eq:nf3-proof-1}
\end{align}
where the second equality uses the fact that $E_{Q_T}[e^{\ell_T}] = 1$~(recall that $\ell_T$ is the log-likelihood ratio between $P_X$ and $Q_T$). Now, for some $K_n \to \infty$, define the events $E_n = \{\|X\|_\infty \leq K_n\}$ and $F_n = \{\|X'\|_\infty \leq K_n\}$, and observe that 
\begin{align}
    \Delta_n(T) &\leq  \inf_{h \in \calH_n} \lbr \mathbb{E}_{P_X}[\lp\ell_T(X) - h(X) \rp \boldsymbol{1}_{E_n}] + \mathbb{E}_{Q_T}[\lp e^{h(X')} - e^{\ell_T(X')} \rp \boldsymbol{1}_{F_n}]  \rbr + \tau_1 + \tau_2,  \label{eq:nf3-proof-2}
\end{align}
where the two terms $\tau_1$ and $\tau_2$ are defined as 
\begin{align}
    \tau_1(K_n, \calT, \calH_n) &= \sup_{T \in \calT} \sup_{h \in \calH_n} \mathbb{E}_{P_X}[ \lv \ell_T(X) - h(X) \rv \boldsymbol{1}_{\|X\|_\infty > K_n}], \\
    \text{and} \quad 
    \tau_2(K_n, \calT, \calH_n) &= \sup_{T \in \calT} \sup_{h \in \calH_n} \mathbb{E}_{Q_T}\lb \lv e^{h(X')} - e^{\ell_T(X')} \rv\boldsymbol{1}_{\|X'\|_\infty > K_n} \rb.
\end{align}
It remains to analyze the first term in the RHS of~\eqref{eq:nf3-proof-2}. We proceed by observing that 
\begin{align}
    \lv e^{\ell_T} - e^{h} \rv = e^{\ell_T} \lv e^{h-\ell_T} - 1 \rv   \leq e^{\ell_T} \lp |\ell_T - h| + \frac{e^{|\ell_T-h|}}{2} |\ell_T-h|^2 \rp, 
\end{align}
where the inequality uses a second-order Taylor approximation of $e^x$ around $0$. More specifically, we use $|e^u - 1| \leq |u| \leq e^{|u|} \tfrac{u^2}{2}$ with $u \leftarrow \ell_T-h$.  
Now, let us introduce the following terms: 
\begin{align}
    e^{b_n} \coloneqq \sup_{h \in \calH_n} \sup_{x: \|x\|_\infty \leq K_n} e^{h(x)}, \qtext{and} \delta_n =  \sup_{T \in \calT} \inf_{h \in \calH_n} \sup_{x: \|x\|_\infty\leq  K_n} |\ell_T(x) - h(x)|. \label{eq:nf3-proof-3}
\end{align}
Using these definitions in~\eqref{eq:nf3-proof-2}, we obtain 
\begin{align}
    \Delta_n(T) &\leq \mathbb{E}_{P_X}[\delta_n \boldsymbol{1}_{E_n}] + \mathbb{E}_{Q_T}\lb e^{\ell_T(X')} \lp \delta_n + e^{\delta_n} \frac{\delta_n^2}{2} \rp \boldsymbol{1}_{F_n} \rb + \tau_1 + \tau_2 \\
    & \leq \delta_n + \lp  \delta_n + e^{\delta_n} \frac{\delta_n^2}{2}\rp \mathbb{E}_{Q_T}[e^{\ell_T}] + \tau_1 + \tau_2 \\
    & = \delta_n \lp 2 + e^{\delta_n} \delta_n/2 \rp + \tau_1 + \tau_2. 
\end{align}
Since the RHS is independent of $T$, taking a supremum over the function class $\calT$, gives us the required 
\begin{align}
    \eta_n \;=\; \sup_{T \in \calT} \Delta_n(T) \;\leq\; \lp 2 + e^{\delta_n} \delta_n/2 \rp \delta_n  + \tau_1 + \tau_2, 
\end{align}
which converges to $0$ if $\delta_n, \tau_1, \tau_2$ converge to $0$. 
This concludes the proof. \hfill %

\subsection{Proof of~\Cref{theorem:nf3-example}}
\label{proof:nf3-variational-gap-example}

To prove this result, we need to verify that the sufficient conditions obtained in~\Cref{prop:nf2-uniform-convergence-verification}~(uniform convergence),~\Cref{prop:nf4-moment-control}~(moment bounds), and~\Cref{prop:nf3-variational-gap}~(variational gap), are satisfied by our specific choices in~\Cref{def:calT-nf} and~\Cref{def:calG-nf}. We verify these conditions in~\Cref{proof:nf-theorem-uniform-convergence},~\Cref{proof:nf-theorem-moment}, and~\Cref{proof:nf-theorem-variational-gap} respectively. Before proceeding to these steps, we first establish certain properties of the function class $\calT$ that will be used often.

Suppose $P_Z$ is a Gaussian distribution with identity covariance, and density $p_Z(z) \propto \exp (- \|z\|^2/2)$, which implies that $\psi(z) = \log p_Z(z) = -\|z\|^2/2 + \mathrm{const}$. 
Under the realizability assumption, there exists $T^* \in \calT$ such that $P_X = Q_{T^*}$.  Every $T \in \calT$ can be represented as 
\begin{align}
    T = G_M \circ G_{M-1} \circ \ldots \circ G_1, \qtext{where} G_j(x) = x + F_j(x), \; \; F_j(x) = W_{j,2} \tanh (W_{j,1}x + b_{j,1}) + b_{j,2}. 
\end{align}
By assumption, we have $\|W_{j,i}\|_{op} \leq s \in (0, 1)$ and $\|b_{j,i}\|\leq B$ for all $j, i$, where $\|\cdot\|$ denotes the $\ell_2$ norm, and $\|\cdot\|_{op}$ denotes the operator norm induced by $\|\cdot\|$.

\subsubsection{Verification of the Uniform Convergence Assumption}
\label{proof:nf-theorem-uniform-convergence}
To simplify the notation, introduce the function classes $\calS_n = \{g \circ T^{-1}: g \in \calG_n, T \in \calT\}$ and $\calU_n = \{(1/b_n)h \circ T^{-1}: h \in \calH_n, T \in \calT\}$. Recall that $\calH_n$ consists of functions of the form $\tilde{h} \boldsymbol{1}_{[-K_n, K_n]^d}$ for all $\tilde{h}$ lying in the RKHS of a Gaussian kernel $k_n$~(denoted by $\calH_{k_n}$), with $\|\tilde{h}\|_{k_n} \leq b_n$. 
We first consider the term $\RadComp_{m_n}(\calS_n)$, since the bound on the other term follows similarly.  
Observe that $\calS_n = \{e^{b_n u}: u \in \calU_n\}$ by definition, leading to the following inequality via the contraction lemma~\cite[Lemma 26.9]{shalev2014understanding}: 
\begin{align}
    \mathfrak{R}_{m_n}(\calS_n, Z^{m_n}) \leq e^{b_n} \mathfrak{R}_{m_n}\lp \calU_n, Z^{m_n} \rp. \label{eq:nf-theorem-proof-1}
\end{align}
It remains for us to control the Rademacher complexity of the composition class $\calU_n$. For any $z^{m_n} = (z_1, \ldots, z_{m_n})$, introduce the terms $\calI \equiv \calI(z^n, K_n) = \{i \in [m_n]: z_i \in [-K_n, K_n]^d\}$ and $s_n \equiv s_n(z^{m_n}, K_n) = |\calI|$. For $z^{m_n}$ such that $s_n >0$, observe that
\begin{align}
    b_n\mathfrak{R}_{m_n}(\calU_n, z^{m_n}) &= \mathbb{E}_{\boldsymbol{\sigma}}\lb  \sup_{T\in \calT} \sup_{h \in \calH_n}\; \frac{1}{m_n} \sum_{i=1}^{m_n} \sigma_i h(T^{-1}(z_i)) \rb \\
    &= \mathbb{E}_{\boldsymbol{\sigma}}\lb  \sup_{T\in \calT} \sup_{\htilde: \|\htilde\|_{k_n} \leq b_n}\;  \frac{1}{m_n} \sum_{i \in \calI} \sigma_i \htilde(T^{-1}(z_i)) \rb \\
    &= \mathbb{E}_{\boldsymbol{\sigma}}\lb  \sup_{T\in \calT} \sup_{\htilde: \|\htilde\|_{k_n} \leq b_n}\; \left \langle \htilde, \, \frac{1}{m_n}  \sum_{i \in \calI} \sigma_i k_n(T^{-1}(z_i), \cdot) \right \rangle_{k_n} \rb  && (\text{reproducing property})\\ 
    &= \mathbb{E}_{\boldsymbol{\sigma}}\lb  \sup_{T\in \calT} {b_n } \left\lVert \frac{1}{m_n} \sum_{i \in \calI} \sigma_ik_n(T^{-1}(z_i), \cdot ) \right\rVert_{k_n} \rb  && \lp \text{optimal } h \propto \frac{b_n}{m_n}\sum_i k_n(T^{-1}(z_i), \cdot)\rp\\
    & = \mathbb{E}_{\boldsymbol{\sigma}}\lb \frac{b_n}{m_n} \sqrt{\sup_{T\in \calT}  \sum_{i \in \calI} \sum_{j \in \calI} \sigma_i \sigma_j k(T^{-1}(z_i), T^{-1}(z_j)) }\rb   \\
    &\leq  \frac{b_n}{m_n} \sqrt{\mathbb{E}_{\boldsymbol{\sigma}}\lb   \sum_{i=1}^{m_n} \sum_{j=1}^{m_n} \sigma_i \sigma_j  \rb } && \lp \text{Jensen's + } \sup_{x,x'} k(x,x') \leq 1 \rp \\
    & = \frac{b_n \sqrt{s_n}}{m_n}. 
\end{align}
As a result, with an \iid sample $Z^{m_n}$, the expected Rademacher complexity satisfies 
\begin{align}
    b_n \RadComp_{m_n}(\calU_n) &=  b_n\mathbb{E}_{Z^{m_n}}\lb \RadComp_{m_n}(\calU_n, Z^{m_n})\rb  \leq \frac{b_n}{m_n} \mathbb{E}_{Z^{m_n}} \lb \sqrt{\sum_{i=1}^{m_n} \boldsymbol{1}_{Z_i \in [-K_n, K_n]^d} }\rb \\ 
    & \leq b_n \sqrt{ \frac{ P_Z([-K_n, K_n]^d) }{m_n} } \leq \frac{b_n}{\sqrt{m_n}}. 
\end{align}
An exactly analogous argument implies $\RadComp_n(\calH_n) \leq b_n/\sqrt{n}$. Combining these two bounds, we see that the uniform convergence rate obtained in~\Cref{prop:nf2-uniform-convergence-verification} reduces to the following (up to leading constants): 
\begin{align}
    r_n \lesssim \frac{b_n}{\sqrt{n}} + \frac{b_n e^{b_n}}{\sqrt{m_n}}. \label{eq:nf-theorem-proof-2}
\end{align}
This can be made to converge to $0$ with $n$, by selecting $b_n = o(\sqrt{n})$ and $m_n = \Omega(e^{(2+\epsilon)b_n})$ for any $\epsilon>0$. Observe that the choices of $b_n, m_n$ in~\Cref{theorem:nf3-example} satisfy these conditions, as we discuss further at the end of this section.  \hfill %

\subsubsection{Verification of the Moment Assumption}
\label{proof:nf-theorem-moment}

\paragraph{Lipschitz property.} Observe that $\tanh$ is $1$-Lipschitz as its derivative is $\operatorname{sech}^2 \in [0,1]$, which implies 
\begin{align}
    \mathrm{Lip}(F_j) \leq \|W_{j,2}\|_{op} \|W_{j,1}\|_{op} \leq s^2 \eqcolon \rho \in (0, 1).  \label{eq:nf-theorem-proof-3}
\end{align}
Hence, for every block $G_j = I_d + F_j$, we have 
\begin{align}
(1-\rho) \|x-x'\| \leq \|G_j(x) - G_j(x')\| \leq (1+ \rho) \|x-x'\|    
\end{align}
Since $T$ is the composition of $M$ such blocks, we immediately obtain 
\begin{align}
 (1-\rho)^M \|x - x'\| \leq   \|T(x) - T(x')\| \leq (1+\rho)^M \|x - x'\|. 
\end{align}
This allows us to conclude that every $T \in \calT$ is bi-Lipschitz with 
\begin{align}
    \mathrm{Lip}(T) \leq (1+s^2)^{M}, \qtext{and} \mathrm{Lip}(T^{-1}) \leq (1-s^2)^{-M}. 
\end{align}
Since these constants are independent of $T$, they are also uniformly valid over the class $\calT$: 
\begin{align}
    \sup_{T \in \calT} \max \lbr \mathrm{Lip}(T), \mathrm{Lip}(T^{-1}) \rbr \leq \max \lbr (1-s^2)^{-M}, \, (1+s^2)^M \rbr \eqcolon J_1. \label{eq:nf-theorem-proof-8}
\end{align}

\paragraph{Values at $x=0$.} Fix any $T = G_M \circ \ldots \circ G_1$ in $\calT$, and consider $x_0 = 0$, $x_j = G_j(x_{j-1}) = x_{j-1} + F_j(x_{j-1})$ for $j \in [M]$. Observe that by~\eqref{eq:nf-theorem-proof-3}, we have for any $x$ and $j$: 
\begin{align}
    \|F_j(x)\| \leq \|F_j(0)\| + s^2 \|x\|. 
\end{align}
Now, $\|F_j(0)\|$ can be bounded as 
\begin{align}
    \|F_j(0)\| = \|W_{j,2} \tanh(b_{j,1}) + b_{j,2}\| \leq \|W_{j,2}\|_{op} \lp \|\tanh(b_{j,1})\| + \|b_{j,2}\|\rp. 
\end{align}
We know that $\|W_{j,2}\|_{op} \leq s$ and $\|b_{j,i}\| \leq B$ for $i =1, 2$. Furthermore, let $H$ denote the maximum dimension of the hidden layer in $M$ blocks; that is, $H =\max_{j\in [M]} d_j$. Then, $\|\tanh(b_{j,1})\| \leq \sqrt{d_j} \tanh(B) \leq \sqrt{H}$. These facts imply that 
\begin{align}
    \|F_j(0)\|\leq C_0 \coloneqq s \sqrt{H} + B. \label{eq:nf-theorem-proof-4}
\end{align}
This implies that for any $j \in [M]$, due to the Lipschitz property, we have 
\begin{align}
    \|x_j\| &= \|G_j (x_{j-1})\| =  \|x_{j-1} + F_j(x_{j-1})\| \leq \|x_{j-1}\| +  \|F_j(x_{j-1})\| \leq \|x_{j-1}\| +  \|F_j(0)\| + s^2 \|x_{j-1}\| \\
    & \leq C_0 + (1+ s^2) \|x_{j-1}\|. 
\end{align}
Applying this inequality iteratively, with $x_0=0$ and $x_M = T(x_0)$, we have 
\begin{align}
    \|T(0)\| \leq C_0 \sum_{i=0}^{M-1} (1+s^2)^{i} = C_0 \lp \frac{(1+s^2)^M -1}{s^2} \rp.  \label{eq:nf-theorem-prooof-5}
\end{align}
Finally, for any $T$, let $y_T$ denote $T^{-1}(0)$, and observe that 
\begin{align}
\|y_T\| =\| T^{-1} \circ T (y_T - 0)\| \leq \mathrm{Lip}(T^{-1}) \|T(y_T) - T(0)\| = \mathrm{Lip}(T^{-1}) \| T(0)\|. \label{eq:nf-theorem-prooof-6} 
\end{align}
Together,~\eqref{eq:nf-theorem-prooof-5} and~\eqref{eq:nf-theorem-prooof-6} imply 
\begin{align}
    \sup_{T \in \calT} \max \lbr \|T(0)\|, \, \|T^{-1}(0)\| \rbr \leq J_1 (s\sqrt{H} + B) \lp \frac{ (1+s^2)^M - 1}{s^2} \rp \eqcolon J_0. \label{eq:nf-theorem-prooof-7}
\end{align}

Since $Z \sim N(0, I_d)$ has finite moments of all orders and $\|T^{-1}(z)\| \leq J_0 + J_1\|z\|$ uniformly in $T$, we have that $\sup_{T \in \calT}\mathbb{E}[\|T^{-1}(Z)\|^a] < \infty$ for all $a >0$.
Under the realizability assumption, $X \stackrel{d}{=} (T^*)^{-1}(Z)$, hence this also implies $\EE\|X\|^{1+a} < \infty$. 
Thus, with an arbitrary $a>0$, and  $M_Z = 2^{a/2} \Gamma( (a+d)/2)/\Gamma( d/2)$, we get 
\begin{align}
    R = 2^a\lp J_0^{1+a} + J_1^{1+a} 2^{a/2} \Gamma( (a+d)/2)/\Gamma( d/2) \rp.  \label{eq:nf-theorem-proof-9}
\end{align}
Recall that $J_0$ was defined in~\eqref{eq:nf-theorem-prooof-7}, and $J_1$ in~\eqref{eq:nf-theorem-proof-8}.  This completes the proof. \hfill %

\subsubsection{Verification of the Variational Gap Assumption}
\label{proof:nf-theorem-variational-gap}

To verify this assumption for the model introduced in~\Cref{def:calT-nf}, we will first obtain the Jacobian of the functions in $\calT$, which will then allow us to obtain a closed form expression for the log likelihood ratio $dP_X/dQ_T$ and characterize its behavior on a cube $C_{K_n} = [-K_n, K_n]^d$. 

\begin{lemma}
\label[lemma]{lemma:nf-Jacobian} 
Consider any $T = G_M \circ G_{M-1} \circ \cdots \circ G_1  \in \calT$ of~\Cref{def:calT-nf}, where each block $G_j(x) = (I_d + F_j)(x) = x + W_{j,2} \tanh(W_{j,1} x +  b_{j,1}) + b_{j,2}$. Then, we have the following: 
\begin{align}
    &\Jac G_j(x) = I_d + W_{j,2} D_j(x) W_{j,1}, \qtext{where} D_j(x) = \mathrm{diag}(\operatorname{sech}^2(W_{j,1} x + b_{j,1})) \\
    & \|\Jac F_j(x)\|_{op} \leq s^2, \qtext{and} \|\Jac F_j(x) - \Jac F_j(y)\|_{op} \leq 2s^3 \|x -y\|. 
\end{align}
\end{lemma}
This result,  proved in~\Cref{proof:nf-Jacobian},  tells us that every $G_j$ is a diffeomorphism with singular values in $[1-s^2, 1+s^2]$, and hence, $\det \Jac G_j(x) > 0$ which implies that $\log \det \Jac G_j$ is well defined. We now obtain a closed-form expression for the log-likelihood ratio $\ell_T =\log dP_X/dQ_T$.

\begin{lemma}
    \label[lemma]{lemma:nf-likelihood-ratio} 
    The log-likelihood ratio function $\ell_T(x) = \log dP_X/dQ_T$ for any $T \in \calT$ is equal to 
    \begin{align}
         \ell_T(x) = \frac{1}{2} \lp \|T(x)\|^2 - \|T^*(x)\|^2 \rp + \log \det \Jac T^*(x) - \log \det \Jac T(x). 
    \end{align}
    Furthermore, when restricted to the domain $[-K, K]^d$, each $\ell_T$ is Lipschitz with constant $L_1 = \calO(K)$~(see~\eqref{eq:nf-likelihood-ratio-proof-4} for exact expression), which implies the following uniform bound 
    \begin{align}
        \sup_{T \in \calT} \sup_{x \in [-K, K]^d} |\ell_T(x)| \leq L_0 + L_1 \sqrt{d} K, 
    \end{align}
    for a universal constant $L_0$ stated in~\eqref{eq:nf-likelihood-ratio-proof-5}
\end{lemma}
This result, proved in~\Cref{proof:nf-likelihood-ratio}, establishes that when restricted to a box $[-K,K]^d$ in the domain $\calX = \mathbb{R}^d$, the likelihood ratio function class $\{\ell_T: T \in \calT\}$ is well behaved. With these results at hand, we can now proceed towards verifying the sufficient conditions for the variational gap to vanish as $n$ goes to $\infty$.

\paragraph{Approximation error $\delta_n$.} We begin by obtaining a characterization of the approximation error $\delta_n$. 

\begin{lemma}
    \label[lemma]{lemma:nf-approximation}
    Consider the given $\calH_n = \{\htilde \boldsymbol{1}_{[-K_n, K_n]^d}: \htilde \in \calH_{k_n}, \; \|\htilde\|_{k_n} \leq b_n\}$, where $\calH_{k_n}$ is the RKHS associated with the Gaussian kernel $k_n(x,y) = \exp(-\gamma_n \|x-y\|^2)$ for a scale parameter $\gamma_n$. For parameters $b_n$, $\gamma_n$, and  $K_n$, we have 
    \begin{align}
        b_n \gtrsim  \gamma_n^{d/4} K_n^{2 + d/2} \quad \implies \quad 
        \delta_n(K_n) \lesssim \frac{K_n}{\sqrt{\gamma_n}}.  \label{eq:delta-Kn-bound}
    \end{align}
    Here $a_n \lesssim b_n$ represents that there exists a constant $c$ that does not change with $n$, such that $a_n \leq c b_n$ for all $n$ large enough.
\end{lemma}
The proof of this statement is in~\Cref{proof:nf-approximation}.

\paragraph{Tail term $\tau_1$. } Let us consider the first tail term $\tau_1$, whose definition is recalled below:
\begin{align}
    \tau_1 \equiv \tau_1(K_n, \calT, \calH_n) = \sup_{T \in \calT} \sup_{h \in \calH_n} \mathbb{E}_{P_X}[\lv \ell_T(X) - h(X) \rv\, \boldsymbol{1}_{\|X\|_\infty > K_n}]. 
\end{align}
Fix an $h$ and $T$, and observe that 
\begin{align}
\mathbb{E}_{P_X}[|\ell_T(X) - h(X)|  \boldsymbol{1}_{\|X\|_\infty > K_n}] &\leq \mathbb{E}_{P_X}[|\ell_T(X)| \boldsymbol{1}_{\|X\|_\infty > K_n} ] +  \mathbb{E}_{P_X}[|h(X)| \boldsymbol{1}_{\|X\|_\infty > K_n}] \\
& \eqcolon \tau_{11}(T, h) + \tau_{12}(T, h). 
\end{align}
Now, from~\Cref{lemma:nf-likelihood-ratio}, we know that there exist constants $c_0, c_1, c_2$ independent of $n, T$, such that we have 
\begin{align}
    |\ell_T(x)| \leq c_0 + c_1 \|x\| + c_2 \|x\|^2 \qtext{for all} x \in \mathbb{R}^d. 
\end{align}
If the distribution $P_X$ is such that $\mathbb{E}_{P_X}[\|X\|^2] < \infty$, then, this implies that 
\begin{align}
    \tau_{11}(T, h) \leq \mathbb{E}_{P_X}[\lp c_0 + c_1 \|X\| + c_2 \|X\|^2\rp \boldsymbol{1}_{\|X\|_\infty > K_n}] \stackrel{n \to \infty}{\longrightarrow} 0. 
\end{align}
Thus, the only condition needed to control $\tau_{11}$ is that $K_n \to \infty$, and that $\mathbb{E}_{P_X}[\|X\|^2] < \infty$. 

Next, we look at the other term $\tau_{12}$. This is much easier to handle, since we know that $\|h\|_\infty \leq b_n$ for all $h \in \calH_n$. Hence, we have 
\begin{align}
    \tau_{12} \leq b_n \mathbb{P}_{P_X}(\|X\|_\infty > K_n) \leq b_n e^{-c K_n^2}, 
\end{align}
for some constant $c$~(depending on the Lipschitz constant of $T^*$). Hence, a $K_n =\Omega (\sqrt{\log b_n}$) is sufficient to drive this to $0$. 

\paragraph{Tail term $\tau_2$.} Finally, we consider the remaining tail term $\tau_2$, defined as 
\begin{align}
    \tau_2 \equiv \tau_2(K_n, \calT, \calH_n) = \sup_{T \in \calT} \sup_{h \in \calH_n} \mathbb{E}_{Q_T}\lb \lv e^{h(X')} - e^{\ell_T(X)}   \rv \boldsymbol{1}_{\|X'\|_\infty > K_n} \rb
\end{align}
For fixed $T \in \calT$ and $h \in \calH_n$, observe that since $h(x)=0$ for all $x:\|x\|_\infty > K_n$, we get 
\begin{align}
    \mathbb{E}_{Q_T}\lb \lv e^{h(X')} - e^{\ell_T(X)}   \rv\boldsymbol{1}_{\|X'\|_\infty > K_n}  \rb & = \mathbb{E}_{Q_T}\lb |1 - e^{\ell_T}|\boldsymbol{1}_{\|X'\|_\infty > K_n}  \rb \\
    & \leq \mathbb{P}_T(\|X'\|_\infty > K_n) + \mathbb{E}_{Q_T}\lb e^{\ell_T} \boldsymbol{1}_{\|X'\|_\infty > K_n} \rb \\
    & = \mathbb{P}(\|X'\|_\infty > K_n) + \mathbb{P}\lp  \|X\|_\infty > K_n \rp.  
\end{align}
Thus, we have 
\begin{align}
    \tau_2 \leq \sup_{T} \mathbb{P} \lp \|T^{-1}(Z)\|_\infty > K_n \rp + \mathbb{P}\lp \|(T^*)^{-1}Z\| >K_n \rp,
\end{align}
which goes to $0$ with $n$  under the $\calT$-uniform Lipschitz bounds derived in~\Cref{lemma:nf-Jacobian}. 

\subsubsection{Completing the Proof}
\label{subsubsec:completing-the-proof-nf}

To complete the proof, we will show an explicit choice of the free parameters that drives the $W_1$-metric between the true and the \ac{NF} distributions to zero. In particular, we observe the following: 
\begin{itemize}
    \item With $m_n = \lceil e^{4 b_n} \rceil$, we get $e^{b_n}/\sqrt{m_n} \leq e^{-b_n}$. Plugging this in~\eqref{eq:nf-theorem-proof-2} then gives us
    \begin{align}
        r_n \lesssim \frac{b_n}{\sqrt{n}} +  b_n e^{-b_n} \lesssim \frac{b_n}{\sqrt{n}}. 
    \end{align}
    \item Now, set $b_n = C_b \gamma^{d/4} K_n^{2 + d/2 + \epsilon}$, and $\gamma_n = K_n^{2 + \epsilon}$ for some $\epsilon>0$. This implies that 
    \begin{align}
        b_n \asymp K_n^{2 + d + \epsilon\lp 1 + \frac d 4 \rp} \quad \implies \quad  r_n \lesssim \frac{K_n^{2 + d + \epsilon\lp 1 + \frac d 4 \rp}}{\sqrt{n}}. 
    \end{align}
    \item The above choice of $b_n$ ensures that $b_n \gtrsim \gamma_n^{d/4} K_n^{2+d/2}$, and hence~\eqref{eq:delta-Kn-bound} applies to these parameter choices, which leads to 
    \begin{align}
        \delta_n(K_n) \lesssim \frac{K_n}{\sqrt{\gamma_n}} \lesssim K_n^{-\epsilon/2}. 
    \end{align}
    \item Due to the uniform Lipschitz condition on the elements of $\calT$ along with the fact that $Z \sim N(0, I_d)$, we can show that $\tau_1 + \tau_2 \lesssim \mathrm{poly(K_n)} e^{-c K_n^2}$ for some $c>0$. With the choice of $K_n \asymp \sqrt{\log n}$ as  stated in~\Cref{theorem:nf-estimation-error}, this reduces to $\tau_1 + \tau_2 \lesssim \mathrm{poly}(\log n) n^{-c}$. Thus, combining with the previous state, we get 
    \begin{align}
        \delta_n + \tau_1 + \tau_2 \lesssim (\log n)^{-\epsilon/4} + \mathrm{poly}(\log n) n^{-c}. 
    \end{align}
\end{itemize}
Thus, with the choice of the parameters in~\Cref{theorem:nf3-example}, the estimation error converges to $0$ with probability at least $1-\delta$.

\section{Deferred Proofs from~\Cref{subsec:variational-inn}}
\label{proof:finite-sample-bound}

\subsection{Proof of~\Cref{theorem:finite-sample}}
\label{proof:finite-sample}

     To simplify the notation, we will use $\PA$ to represent the conditional probability $P(\cdot \mid Y \in A)$. By the realizability assumption, we have $\PA_{Y, T^*_z(X)} = \PA_{YZ}$. 
     The starting point of the proof is to use the realizability and the uniform convergence assumptions to relate the population losses associated with the \ac{ERM} model $\That_n$ to the error terms $u_n$ and  $r_n$  from~\Cref{assump:INN-theorem}. For any $T \in \calT$, let $L(T) = L_\yy(T) + \lambda L_\zz(T)$, and $\Lhat(T) = \LhatY(T) + \lambda \LhatZ(T)$, and observe the following chain: 
     \begin{align}
         0 = L(T^*) \geq \Lhat(T^*) - u_n - \lambda r_n \geq \Lhat(\That_n) - u_n - \lambda r_n  \geq L(\That_n) - 2 (u_n + \lambda r_n).  \label{eq:erm-property-1}
     \end{align}
     The equality follows from the realizability assumption~\textbf{(INN1)}, while the first and third inequalities use the uniform convergence assumption~\textbf{(INN2)}, and the second inequality follows form the definition of $\That_n$ as the \ac{ERM} model. This simple argument allows us to control the individual components of $L(\That_n)$ as follows: 
     \begin{align}
         L_{\mathbf{y}}(\That_n) \leq 2(u_n + \lambda r_n), \qtext{and}  L_{\mathbf{z}}(\That_n) \leq 2 \bigg( r_n + \frac{u_n}{\lambda} \bigg). \label{eq:erm-property-2}
     \end{align}
     For the first inequality, we use the fact that $L_{\mathbf{z}}(T) \geq 0$ for all $T$ as the critic class $\calG_n$ is assumed to contain the $\boldsymbol{0}$ element, while the second inequality uses the fact that $L_{\mathbf{y}}(T) \geq 0$ by definition of the squared loss.

    The next step of the proof is to use the Lipschitz property of the inverse map $\That_n^{-1}$ to relate the $W_1$ metric between the posterior distributions $\PA_{\That_n(X)}$ and $\PA_{YZ}$. In particular, to further simplify the notation and use $\mu = \PA_{\That_n(X)}$, $\rho = \PA_{YZ}$, and let $F$ denote the inverse map $\That_n^{-1}:\calY \times \calZ \to \calX$.  Then, $F_{\#}\mu$ and $F_{\#} \rho$ denote the true and \ac{INN} posterior distributions on $\calX$  by definition~(recall that we use $F_{\#} \mu$ to represent the  pushforward of measure $\mu$ on to the image space of $F$; that is for any measurable $E \subset \calX$, we have $(F_{\#}\mu)(E) = \mu(F^{-1}(E))$). Let $\pi \in \Pi(\mu, \rho)$ denote any coupling on $(\calY \times \calZ)^2$ with marginals $\mu$ and $\rho$, and let $\pi_X \in \Pi(F_{\#}\mu, F_{\#}\rho)$ denote the coupling by pushing $\pi$ forward under $F \times F$~(recall that $F=\That_n^{-1}$); that is, $\pi_X(E_X \times G_X) = \pi\lp \{(y,z): F(y) \in E_X, \; F(z) \in G_X\}\rp$.  With these definitions, we immediately have 
    \begin{align}
        W_1(F_{\#}\mu, F_{\#}\rho) &= \inf_{\pi' \in \Pi\lp F_{\#}\mu, F_{\#} \rho \rp} \int_{\calX \times \calX} \|x - x'\| d\pi'(x, x') = \inf_{\substack{\pi_X = (F \times F)_{\#} \pi \\ \pi \in \Pi(\mu, \rho)}} \int_{\calX \times \calX} \|x - x'\| d\pi_X(x, x') \\
        & = \inf_{\pi \in \Pi(\mu, \rho)} \int_{\calX \times \calX} \|F(y,z) - F(y', z') \| d\pi((y,z), (y', z')) \\
        &  \leq J \inf_{\pi \in \Pi(\mu, \rho)}  \int_{(\calY\times \calZ)^2}  \|(y,z) - (y', z') \| d\pi((y,z), (y', z')) 
         = J W_1(\mu, \rho). 
    \end{align}
    Putting back the values $\mu \leftarrow \PA_{\That_n(X)}$, $\rho \leftarrow \PA_{YZ}$, and $F \leftarrow \That_n^{-1}$, the inequality chain above implies  
    \begin{align}
        W_1\bigg(\PA_X, \PA_{\hat{T}_n^{-1}(Y,Z)} \bigg) &\leq  J W_1\lp \PA_{\That_n(X)}, \PA_{YZ} \rp.  \label{eq:main-theorem-proof-0}
    \end{align}
    Let us introduce the notation: 
    \begin{align}
        \nu = \PA_{\That_n(X)}, \quad \gamma = \PA_{Y, \That_{n,z}(X)}, \qtext{and} \omega = \PA_{Y,Z}. \label{eq:main-proof-nu-mu-rho}
    \end{align}
    Using the fact that $W_1$ is a metric, we can use triangle inequality to get 
    \begin{align}
        W_1\lp \PA_{\That_n(X)}, \PA_{YZ} \rp = W_1(\nu, \omega) \leq W_1(\nu, \gamma) + W_1(\gamma, \omega).  \label{eq:main-theorem-proof-1}
    \end{align}
    To complete the proof, we need to obtain upper bounds on the two terms in the RHS above. To control the first term in~\eqref{eq:main-theorem-proof-1},  
    We begin by recalling the explicit definition of the first term
        \begin{align}
             W_1(\nu, \gamma) = \inf_{\pi \in \Pi(\nu, \gamma)} \int \|(y,z) - (y', z')\| d\pi((y,z), (y', z'))
        \end{align}
    Let us now construct the natural coupling $\pi$ between $\nu$ and $\gamma$  in the following steps: 
    
    \begin{itemize}
        \item Generate $(X, Y) \sim \PA_{XY}$. 
        \item Let $\nu$ denote the conditional distribution of $\That_n(X) = (\That_{n,y}(X), \That_{n,z}(X))$
        \item Let $\gamma$ denote the conditional distribution of $(Y, \That_{n,z}(X))$. Thus $\nu$ and $\gamma$ have a common second component. 
    \end{itemize}
    Using this particular coupling in the definition of $W_1(\nu, \gamma)$, we observe that 
    \begin{align}
        W_1(\nu, \gamma) &\leq \int \sqrt{\|\That_{n,y}(x) - y\|^2 + \|\That_{n,z}(x) - \That_{n,z}(x)\|^2 }\, d\PA_{XY}(x, y)  
         = \EA_{XY} \lb \|\That_{n,y}(X) - Y\|\rb \\
         &\leq \sqrt{\EA_{XY} \lb \|\That_{n,y}(X) - Y\|^2\rb} 
          \stackrel{\eqref{eq:main-theorem-proof-2}.1}{\leq} \frac{\sqrt{L_{\mathbf{y}}(\That_n)}}{\sqrt{P_Y(A)}} \leq \frac{\sqrt{2(u_n + \lambda r_n)}}{\sqrt{P_Y(A)}},  \label{eq:main-theorem-proof-2}
    \end{align}
    \sloppy where the last inequality follows from~\eqref{eq:erm-property-2}, and $\eqref{eq:main-theorem-proof-2}.1$ uses the fact that $L_{\mathbf{y}}(\That_n) = \mathbb{E}_{XY}[\|\That_{n,y}(X) - Y\|^2] = P_Y(A) \EA_{XY}[\|\That_{n,y}(X)-Y\|^2] + P_Y(A^c) \mathbb{E}_{XY}^{(A^c)}[\|\That_{n,y}(X)-Y\|^2] \geq P_Y(A) \EA_{XY}[\|\That_{n,y}(X)-Y\|^2]$.

    It remains for us to obtain a bound on the term~$W_1(\gamma, \omega)$. Introduce the conditional total variation~(TV) distance between these two measures, 
    \begin{align}
        \Delta_A &= \|\PA_{YZ} - \PA_{Y, \That_{n,z}(X)} \|_{TV} = \|\omega - \gamma\|_{TV},  \label{eq:conditional-tv-distance}
    \end{align}
    and observe that an application of the truncation argument stated in~\Cref{lemma:truncation} gives us 
    \begin{align}
        W_1(\gamma, \omega) \leq C_a  \lp \frac{2R}{P_Y(A)}\rp^{\frac{1}{1+a}}\Delta_A^{\frac{a}{1+a}}, \qtext{where} C_a = 2\lp a^{\frac{1}{1+a}} + a^{\frac{-a}{1+a}} \rp.  \label{eq:main-theorem-proof-3}
    \end{align}
    Next, we use the fact that the two distributions $\gamma$ and $\omega$ share the marginal distribution of $Y$, which can be used to show that 
    \begin{align}
        D_f\lp \PA_{Y, \That_{n,z}(X)} \parallel \PA_{YZ} \rp \stackrel{\eqref{eq:main-theorem-proof-5}.1}{\leq} \frac{1}{P_Y(A)} D_f\lp P_{Y, \That_{n,z}(X)} \parallel P_{YZ} \rp \stackrel{\eqref{eq:main-theorem-proof-5}.2}{\leq} \frac{2(r_n + u_n/\lambda) + \eta_n}{P_Y(A)}, \label{eq:main-theorem-proof-5}
    \end{align}
    where$~\eqref{eq:main-theorem-proof-5}.1$ is justified in~\Cref{lemma:condition-f-div}, and~$\eqref{eq:main-theorem-proof-5}.2$ follows  from the bound obtained in~\eqref{eq:erm-property-2}, along with an application of the variational gap assumption~\textbf{(INN4)} that relates $L_{\mathbf{z}}(\That_n)$ to $D_f\big( P_{Y, \That_{n,z}(X)} \parallel P_{YZ} \big)$.

    Next, we use the  Pinsker-type inequality to obtain 
    \begin{align}
        \Delta_A \leq c_f \sqrt{\frac{1}{2} D_f\lp \PA_{Y, \That_{n,z}(X)} \parallel \PA_{YZ} \rp} \leq c_f \sqrt{\frac{(2r_n + 2u_n/\lambda +  \eta_n)}{2 P_Y(A)}}. 
    \end{align}
    Plugging this into~\eqref{eq:main-theorem-proof-3}, we get 
    \begin{align}
        W_1(\gamma, \omega) \leq C_a 2^{\frac{2-a}{2(1+a)}}\lp P_Y(A)\rp^{-\frac{2+a}{2(1+a)}}   c_f^{\frac{a}{1+a}} (2r_n + 2 u_n/\lambda +  \eta_n)^{\frac{a}{2(1+a)}} \label{eq:main-theorem-proof-4}
    \end{align}
    
    Combining~\eqref{eq:main-theorem-proof-4} and~\eqref{eq:main-theorem-proof-2} with~\eqref{eq:main-theorem-proof-1} gives us the stated upper bound: 
    \begin{align}
        W_1(\PA_X, \PA_{\That_n^{-1}(Y,Z)}) \leq J\lp \sqrt{\frac{2(u_n + \lambda r_n)}{P_Y(A)}} + K_a (P_Y(A))^{-\frac{2+a}{2(1+a)}}  \lp 2r_n + 2u_n/\lambda +  \eta_n \rp^{\frac{a}{2(1+a)}} \rp,   \label{eq:main-theorem-proof-6}
    \end{align}
    where $K_a = 2^{\frac{4+a}{2+2a}}\lp a^{\frac{1}{1+a}} +a^{\frac{-a}{1+a}} \rp c_f^{\frac{a}{1+a}}$. For any finite $a>0$, the dominant term is $\lesssim (r_n + u_n + \eta_n)^{\frac{a}{2(1+a)}}$, as claimed in the statement of~\Cref{theorem:finite-sample}.  This completes the proof. \hfill %

\subsection{Proof of~\Cref{prop:inn2-uniform-convergence}}
\label{proof:inn2-uniform-convergence}

\subsubsection{Proof of~\eqref{eq:inn2-prop-Lz}}

The uniform convergence of the $\LhatZ(T)$ follows from the boundedness of the function classes involved, along with some standard symmetrization and concentration techniques.  In particular, we can show that with probability at least $1-\delta/2$, we have 
\begin{align}
    \sup_{T \in \calT} |\LhatZ(T) - L_{\mathbf{z}}(T)| &\leq 2 \RadComp_n^{(1)} + 2 \RadComp_n^{(2)} + (b_n + A_{2,n}) \sqrt{\frac{2\log(4/\delta)}{n}}. \label{eq:inn2-proof-0}
\end{align}
We first introduce some notation to simplify the expressions. For any $T$ and $g$, define 
\begin{align}
    &\alpha(g) = \frac{1}{n} \sum_{i=1}^n g(Y_i, Z_i) - \mathbb{E}[g(Y,Z)], \quad 
    \psi(T, g) = \frac 1n \sum_{i=1}^n f^*\lp g(Y_i, T_{\mathbf{z}}(X_i))\rp- \mathbb{E}[f^*(g(Y, T_{\mathbf{z}}(X)))],  \label{eq:alpha-psi-def}\\
   \text{and} \quad 
  & s(T, g) = \mathbb{E}[g(Y,Z)] - \mathbb{E}[f^*\lp g(Y, T_{\mathbf{z}}(X))\rp].  \label{eq:s-T-g}
\end{align}
With these terms, we can write 
\begin{align}
    \LhatZ(T) & = \sup_{g \in \calG_n} \big[ s(T, g) + \alpha(g) - \psi(T, g) \big], \qtext{and}
    L_{\mathbf{z}}(T)  = \sup_{g \in \calG_n}  s(T, g). 
\end{align}
This leads to the following chain: 
\begin{align}
    \sup_{T \in \calT} \left\lvert \LhatZ(T) - L_{\mathbf{z}}(T) \right\rvert & = \sup_{T \in \calT} \left\lvert  \sup_{g\in \calG_n} \lp s(T,g) + \alpha(g) - \psi(T, g) \rp - \sup_{g \in \calG_n} s(T,g) \right\rvert  \\ 
    & \leq \sup_{T \in \calT} \sup_{g \in \calG_n} \left\lvert \alpha(g) + \psi(T,g) \right \rvert  \\
     & \leq \sup_{g \in \calG_n}|\alpha(g)| +  \sup_{T \in \calT} \sup_{g \in \calG_n} \left\lvert \psi(T,g) \right \rvert.  \label{eq:inn-uniform-convergence-proof-1}
\end{align}
It remains to control the two terms in~\eqref{eq:inn-uniform-convergence-proof-1}. For the first term, standard symmetrization arguments imply 
\begin{align}
    \EE[\sup_{g \in \calG_n} |\alpha(g)| ] & \leq 2 \mathbb{E}_{Y^n, Z^n, \epsilon^n}\lb \sup_{g \in \calG_n} \frac 1n \sum_{i=1}^n \epsilon_i g(Y_i, Z_i) \rb = 2 \RadComp_n^{(1)}(\calG_n). 
\end{align}
Finally,  the uniform boundedness assumption on $\calG_n$ implies that the random variable $\sup_{g\in \calG_n} \frac 1n \sum_{i=1}^n \epsilon_i g(Y_i, Z_i)$ satisfies the bounded difference property. Hence, an application of McDiarmid's inequality leads to 
\begin{align}
    \mathbb{P}\lp \sup_{g \in \calG_n} |\alpha(\alpha)| \geq 2 \RadComp_n^{(1)}(\calG_n) + b_n \sqrt{\frac{2 \log(8/\delta)}{n}} \rp \leq \frac{\delta}{4}. \label{eq:inn-uniform-convergence-proof-2}
\end{align}
Next, to bound the second term in~\eqref{eq:inn-uniform-convergence-proof-1}, we use the shorthand $m_{T,g}(x,y) = f^*(g(y, T_{\mathbf{z}}(x))$, and let $\calM_n = \{h_{T,g}: g \in \calG_n, \; T \in \calT\}$ denote the associated function, and obtain the following~(dropping the subscript from $m_{T,g}$):
\begin{align}
    \EE[\sup_{T, g} |\psi(T,g)|] & =  \EE\lb \sup_{m \in \calM_n} \lv \frac 1n \sum_{i=1}^n m(X_i, Y_i) - \EE[m(X, Y)] \rv  \rb \leq 2 \EE_{X^n, Y^n, \epsilon^n}\lb \sup_{m \in \calM_n} \frac 1n \sum_{i=1}^n \epsilon_i m(X_i, Y_i) \rb, 
\end{align}
by another application of the symmetrization technique. Next, observe that by definition, $u \mapsto f^*(u)$ is Lipschitz with constant $A_{1,n}$, and thus using the vector contraction lemma~(\Cref{fact:maurer-contraction}), we obtain 
\begin{align}
    \EE_{X^n, Y^n, \epsilon^n}\lb \sup_{m \in \calM_n} \frac 1n \sum_{i=1}^n \epsilon_i m(X_i, Y_i) \rb & \leq 2 A_{1,n} \EE_{X^n, Y^n, \epsilon^n} \lb \sup_{T,g} \frac 1n \sum_{i=1}^n \epsilon_i g(Y_i, T_{\mathbf{z}}(X_i)) \rb = 2 \RadComp_n^{(2)}(\calG_n, \calT). 
\end{align}
Finally, observing that each $|f^*(g(Y_i, T_{\mathbf{z}}(X_i)))| \leq A_{2,n}$ by assumption, another application of McDiarmid's bounded difference inequality gives us the required concentration result: 
\begin{align}
    \mathbb{P}\lp \sup_{\substack{g \in \calG_n \\ T \in \calT}} |\psi_{T,g}| \geq 4 A_{1,n} \RadComp_n^{(2)}(\calG_n, \calT) + A_{2,n}\sqrt{\frac{2 \log(8/\delta)}{n}} \rp \leq \frac{\delta}{4}. \label{eq:inn-uniform-convergence-proof-3}
\end{align}
Combining~\eqref{eq:inn-uniform-convergence-proof-2} and~\eqref{eq:inn-uniform-convergence-proof-3} give us the required $1-\delta/2$ probability bound claimed in~\eqref{eq:inn2-prop-Lz}.

\subsubsection{Proof of~\eqref{eq:inn2-prop-Ly}}
The proof of~\eqref{eq:inn2-prop-Ly} requires a careful truncation idea, as the squared loss function can be unbounded. We begin with the simple decomposition for some $K_n>0$: 
\begin{align}
    \sup_{T}|\LhatY(T) - L_{\mathbf{y}}(T)| &= \underbrace{\sup_{T}|\LhatY^{K_n}(T) - L_{\mathbf{y}}^{K_n}(T)|}_{\coloneqq \mathrm{Term1}} + \underbrace{\sup_{T}|\LhatY(T) - \LhatY^{K_n}(T)|}_{\coloneqq \mathrm{Term2}} + \underbrace{\sup_{T}|L_{\mathbf{y}}^{K_n}(T) - L_{\mathbf{y}}(T)|}_{\coloneqq \mathrm{Term3}}.  \label{eq:inn2-proof-1}
\end{align}
We will bound the three terms in~\eqref{eq:inn2-proof-1} separately.  

\paragraph{Bound on~$\mathrm{Term1}$ in~\eqref{eq:inn2-proof-1}.} Let us introduce the function $\phi_y(u) = \|\Pi_{K_n}(u)-\Pi_{K_n}(y)\|^2$, and its centered version $\bar{\phi}_y(u) = \phi_y(u) - \phi_y(0)$. Then, observe the following: 
\begin{align}
    \mathbb{E}_{\epsilon^n}\bigg[ \sup_{T \in \calT} \frac{1}{n} \sum_{i=1}^n \epsilon_i \phi_{Y_i}(T_{\mathbf{y}}(X_i)) \bigg] &= \mathbb{E}_{\epsilon^n} \lb \sup_{T} \frac{1}{n} \sum_{i=1}^n \epsilon_i \bar{\phi}_{Y_i}(T_{\mathbf{y}}(X_i)) + \frac{1}{n} \sum_{i=1}^n \epsilon_i \phi_{Y_i}(0)\rb \\
    &= \mathbb{E}_{\epsilon^n} \lb \sup_{T} \frac{1}{n} \sum_{i=1}^n \epsilon_i \bar{\phi}_{Y_i}(T_{\mathbf{y}}(X_i)) \rb  + \mathbb{E}_{\epsilon^n} \lb \frac{1}{n} \sum_{i=1}^n \epsilon_i \phi_{Y_i}(0)\rb\\
    & = \mathbb{E}_{\epsilon^n} \lb \sup_{T} \frac{1}{n} \sum_{i=1}^n \epsilon_i \bar{\phi}_{Y_i}(T_{\mathbf{y}}(X_i)) \rb,   \label{eq:inn2-proof-2}
\end{align}
since each Rademacher random variable $\epsilon_i$ is independent of $Y_i$. Thus, we can conclude via the standard symmetrization argument that 
\begin{align}
    \mathbb{E}\lb \mathrm{Term1} \rb \leq 2 \mathbb{E} \lb \sup_{T \in \calT} \frac{1}{n} \sum_{i=1}^n \epsilon_i \phi_{Y_i}(T_{\mathbf{y}}(X_i)) \rb \leq 2 \mathbb{E} \lb \sup_{T \in \calT} \frac{1}{n} \sum_{i=1}^n \epsilon_i \bar{\phi}_{Y_i}(T_{\mathbf{y}}(X_i)) \rb. 
\end{align}
Now, we can verify that the functions $\phi_y(u)$ are $4K_n$ Lipschitz globally, which implies that we can apply the (scalar) contraction lemma to get 
\begin{align}
    \mathbb{E}_{\epsilon^n} \lb \sup_{T \in \calT} \frac{1}{n} \sum_{i=1}^n \epsilon_i \phi_{Y_i}(T_{\mathbf{y}}(X_i)) \rb \leq 4 K_n \mathbb{E}_{\epsilon^n} \lb \sup_{\substack{T \in \calT \\ \|u_i\|\leq 1}} \frac{1}{n} \sum_{i=1}^n \epsilon_i \langle u_i, T_{\mathbf{y}}(X_i)\rangle \rb = 4K_n \RadComp_n(\calF_{\yy}). 
\end{align}
Plugging this into~\eqref{eq:inn2-proof-2}, we get 
\begin{align}
    \mathbb{E}[\mathrm{Term1}] \leq 8K_n \RadComp_n(\calF_{\yy}). 
\end{align}
We can translate this expectation result into a high probability deviation bound via the standard bounded difference concentration inequality. In particular, by construction, the functions involved in $\mathrm{Term1}$ satisfy a bounded differences property with parameter $2K_n$, which implies the following 
\begin{align}
\mathbb{P}\lp \mathrm{Term1} \geq 8K_n \RadComp_n(\calF_{\yy}) + 4K_n^2\sqrt{2 \log\lp \tfrac{8}{\delta} \rp}{n} \rp \leq \frac{\delta}{4}, \label{eq:inn2-proof-3}
\end{align}
by an application of Mcdiarmid's inequality.

\paragraph{Bound on $\mathrm{Term2}$ in~\eqref{eq:inn2-proof-1}.} From the definitions of $\LhatY^{K_n}$ and $\LhatY$, we observe that 
\begin{align}
    \LhatY - \LhatY^{K_n} &\leq \frac{1}{n} \sum_{i=1}^n \|T_{\mathbf{y}}(X_i) - Y_i\|^2 \boldsymbol{1}_{\|T_{\mathbf{y}}(X_i)\| \vee \|Y_i\| > K_n} \\ 
    &\leq \frac{1}{n} \sum_{i=1}^n 2\|T_{\mathbf{y}}(X_i)\|^2 \boldsymbol{1}_{\|T_{\mathbf{y}}(X_i)\| > K_n} + 2\|Y_i\|^2 \boldsymbol{1}_{\|Y_i\| > K_n} \\
    & \leq  \frac{2}{n} \sum_{i=1}^n \frac{\|T_{\mathbf{y}}(X_i)\|^{2+\beta}}{K_n^\beta} + \frac{\|Y_i\|^{2+\beta}}{K_n^\beta}, 
\end{align}
where the last inequality uses the fact that $\boldsymbol{1}_{A > B} \leq (A/B)^\beta$. Thus, the assumption on the $(2+\beta)$ moment of $T_{\mathbf{y}}(X_i)$ and $Y_i$ together imply that 
\begin{align}
    \mathbb{E}[\sup_{T} |\LhatY(T) - \LhatY^{K_n}(T)| ] \leq \frac{4R_{\yy}}{K_n^{\beta}}. \label{eq:inn2-proof-4}
\end{align}
For a given $\delta>0$, we then apply Markov's inequality to conclude that 
\begin{align}
    \mathbb{P}\lp \sup_{T} |\LhatY(T) - \LhatY^{K_n}(T)|  > \frac{16R_{\yy}}{K_n^\beta \delta}\rp \leq \frac{\delta}{4}. \label{eq:inn2-proof-5}
\end{align}

\paragraph{Bound on $\mathrm{Term3}$ in~\eqref{eq:inn2-proof-1}.} The same argument that we used to obtain~\eqref{eq:inn2-proof-4} is also applicable to this term, and we can conclude that 
\begin{align}
    \sup_{T \in \calT} |L_{\yy}(T) - {L}^{K_n}_{\yy}(T)|  &\leq\sup_{T} \mathbb{E}[\|T_{\mathbf{y}}(X) - Y\|^2\boldsymbol{1}_{\|T_{\mathbf{y}}(X)\|^2 \vee \|Y\|^2 > K_n}]  \\ &\leq \sup_{T} 2 \lp \mathbb{E}[\|T_{\mathbf{y}}(X)\|^2 \boldsymbol{1}_{\|T_{\mathbf{y}}(X)\|^2 > K_n}] + \mathbb{E}[\| Y\|^2 \boldsymbol{1}_{\|Y\|^2 > K_n}] \rp 
     \leq \frac{4R_{\yy}}{K_n^{\beta}}. \label{eq:inn2-proof-6}
\end{align}
Combining~\eqref{eq:inn2-proof-3},~\eqref{eq:inn2-proof-5}, and~\eqref{eq:inn2-proof-6} with~\eqref{eq:inn2-proof-1} and~\eqref{eq:inn2-proof-0}, we get the required result. 

\subsection{Proof of~\Cref{prop:inn4-variational-gap}}
\label{proof:inn4-variational-gap}
The general outline of the proof is similar to that of~\Cref{prop:nf3-variational-gap}. For any $T \in \calT$, let $\gamma_T$ denote $D_f(P \parallel Q_T)$, and observe that for any $g \in \calG_n$, we have  
\begin{align}
    \Delta(T) &= D_f(P \parallel Q_T) - \sup_{g \in \calG_n} \lbr \mathbb{E}_P[ g] - \mathbb{E}_{Q_T}[f^*(g)]  \rbr \\
    & =\lp \mathbb{E}_P[g^*_T] - \mathbb{E}_{Q_T}[f^*(g^*_T)] \rp - \sup_{g \in \calG_n} \lbr \mathbb{E}_P[ g] - \mathbb{E}_{Q_T}[f^*(g)]  \rbr \\
    & \leq  \lp \mathbb{E}_P[g^*_T ] - \mathbb{E}_{Q_T}[f^*(g^*_T)] \rp - \lp \mathbb{E}_P[g] - \mathbb{E}_{Q_T}[f^*(g)]\rp.  
\end{align}
Next, we subtract and add the clipped version of $g^*_T$, denoted by $\bar{g}_T$, to get 
\begin{align}
    \Delta(T)  & \leq  \lp \mathbb{E}_P[g^*_T -\bar{g}_T] + \mathbb{E}_P[\bar{g}_T] \rp - \lp \mathbb{E}_{Q_T}[f^*(g^*_T) - f^*(\bar{g}_T)] + \mathbb{E}_{Q_T}[f^*(\bar{g}_T)] \rp - \lp \mathbb{E}_P[g] - \mathbb{E}_{Q_T}[f^*(g)]\rp  \\
    & = \lp \mathbb{E}_P[g^*_T - \bar{g}_T] + \mathbb{E}_{Q_T}[f^*(\bar{g}_T) - f^*(g^*_T)] \rp + \lp \mathbb{E}_P[\bar{g}_T - g]  + \mathbb{E}_{Q_T}[f^*(g) - f^*(\bar{g}_T)] \rp \\
    & \coloneqq  C_1(T) + C_2(T, g). 
\end{align}
The first term $C_1(T)$ is the ``clipping error'' and satisfies the inequality 
\begin{align}
    C_1(T) & \leq \mathbb{E}_P[ (|g^*_T| - b_n)^+] + A_{1,n} \mathbb{E}_{Q_T}[\lp |g^*_T| - b_n)^+ \rp]. 
\end{align}
In the above display, we have used the fact that $|g^*_T - \bar{g}_T| = (|g^*_T| - b_n)^+$ and that $f^*$ is $A_{1,n}$ Lipschitz on $[-b_n, b_n]$. 

To analyze the term $C_2(T,g)$, we will consider its behavior inside and outside the ball $B_{K_n}$, which we denote by $C_{21} \equiv C_{2,1}(T,g)$ and $C_{22} \equiv C_{22}(T,g)$ respectively.
\begin{align}
    &C_{21} = \mathbb{E}_P[(\bar{g}_T - g) \boldsymbol{1}_{B_{K_n}}] + \mathbb{E}_{Q_T}[ \big( f^*(g)- f^*(\bar{g}_T)\big) \boldsymbol{1}_{B_{K_n}}], \\
    \text{and} \quad 
    & C_{22} = \mathbb{E}_P[(\bar{g}_T - g) \boldsymbol{1}_{B_{K_n}^c}] + \mathbb{E}_{Q_T}[ \big( f^*(g)- f^*(\bar{g}_T)\big) \boldsymbol{1}_{B_{K_n}^c}].
\end{align}
To bound the term $C_{21}$, we again appeal to the fact that $f^*$ is $A_{1,n}$ Lipschitz on $[-b_n, b_n]$, and hence we have 
\begin{align}
    |C_{21}| \leq \mathbb{E}_P[|\bar{g}_T - g| \boldsymbol{1}_{B_{K_n}}] + A_{1,n}  \mathbb{E}_{Q_T}[|\bar{g}_T - g| \boldsymbol{1}_{B_{K_n}}]. 
\end{align}
By definition, we have $|\bar{g}_T| \leq b_n$ and $|f^*(\cdot)| \leq A_{2,n}$, and also $|g| \leq b_n$ due to the uniform boundedness of $\calG_n$. Together these facts imply 
\begin{align}
    |C_{22}| \leq  b_n \mathbb{P}_P\lp B_{K_n}^c \rp + A_{2,n} \mathbb{P}_{Q_T}\lp B_{K_n}^c \rp.  
\end{align}
Combining these, we get for any $T \in \calT$, and $g \in \calG_n$: 
\begin{align}
    \Delta(T) & \leq |C_1(T)| +  \mathbb{E}_P[|\bar{g}_T - g| \boldsymbol{1}_{B_{K_n}}] + A_{1,n}  \mathbb{E}_{Q_T}[|\bar{g}_T - g| \boldsymbol{1}_{B_{K_n}}] + b_n \mathbb{P}_P\lp B_{K_n}^c \rp + A_{2,n} \mathbb{P}_{Q_T}\lp B_{K_n}^c \rp. 
\end{align}
Taking the infimum over all $g \in \calG_n$, and supremum over all $T \in \calT$, we get
\begin{align}
    \eta_n \leq \tau_2(b_n) + (1+A_{1,n}) \delta_n(b_n, K_n) + (b_n + A_{2,n})  \tau_1(K_n). 
\end{align}
Hence, a sufficient condition for this to converge to zero is if $\tau_2 + A_{1,n} \delta_n + (b_n + A_{2,n}) \tau_1$  converges to $0$. 
This concludes the proof. \hfill %

\subsection{Proof of~\Cref{theorem:inn-example}}
\label{proof:inn-example}

The proof of this result follows closely the proof of~\Cref{theorem:nf-estimation-error} as given in~\Cref{proof:nf-estimation-error}. Since we use the same model class $\calT$, we can follow the exact argument we used in~\Cref{proof:nf-theorem-moment}  to show that 
\begin{align}
    &\sup_{T \in \calT} \lbr \mathrm{Lip}(T), \, \mathrm{Lip}(T^{-1}) \rbr \leq J \coloneqq \max \lbr (1+s^2)^{M}, \, (1-s^2)^{-M} \rbr,   \\
    \text{and} \quad 
    & \sup_{T \in \calT} \|T(\boldsymbol{0})\|_2 \leq J_0 \coloneqq \lp s \sqrt{H} + B \rp  \frac{(1+s^2)^M - 1}{s^2}. 
\end{align}
Recall that the terms $s, B, H$, and $M$ were introduced in~\Cref{def:inn-calT}. Additionally, since our latent variable $Z$ is a multivariate Gaussian, it follows from the realizability assumption that the conditions of~\Cref{prop:inn3-uniform-moment-bound} hold with all $\beta >0$.  To complete the proof, it remains to verify that~\Cref{prop:inn2-uniform-convergence} and~\Cref{prop:inn4-variational-gap} hold. 

\subsubsection{Verification of~\Cref{prop:inn2-uniform-convergence}}
To verify this result, we need to control the supervised and unsupervised loss terms, and in particular, show the existence of $u_n, r_n \to 0$ such that 
\begin{align}
    \sup_{T \in \calT} |\LhatY(T) - L_{\mathbf{y}}(T)| \leq u_n, \qtext{and} \sup_{T \in \calT} |\LhatZ(T) - L_{\mathbf{z}}(T)| \leq r_n. 
\end{align}
\Cref{prop:inn2-uniform-convergence} implies that it suffices to check the uniform control over $2 + \beta$ moment of $T_{\mathbf{y}}(X)$ over all $T \in \calT$, and that the Rademacher complexities $\RadComp_n(\calF_{\yy}), \RadComp_n^{(1)}(\calG_n)$, and $\RadComp_n^{(2)}(\calG_n, \calT)$ vanish with $n$. 

\paragraph{Finite $2 + \beta$ moment.} From the global Lipschitz property of $T$, there exist finite constants $J_0, J$ such that for all $T \equiv (T_{\mathbf{y}}, T_{\mathbf{z}}) \in \calT$, we have 
\begin{align}
    \|T_{\mathbf{y}}(x)\| \leq J_0 + J \|x\|, \qtext{for all} x\in \calX. 
\end{align}
Hence, for any $\beta > 0$, there exists a constant $C_\beta$ depending on $J_0, J$ and $\beta$, such that 
\begin{align}
    \|T_{\mathbf{y}}(X)\|^{2 + \beta} \leq C_\beta \lp 1 + \|X\|^{2+\beta} \rp \quad \implies \quad 
    \sup_{T \in \calT} \EE\|T_{\mathbf{y}}(X)\|^{2+\beta} \leq C_\beta \lp 1 + \EE[\|X\|^{2+\beta}] \rp < \infty, 
\end{align}
using the finite $2+\beta$ moment assumption on $X$. 

\paragraph{Rademacher complexity $\RadComp_n(\calF_{\yy})$.} For any $T \in \calT$ and $u: \|u\|\leq 1$, observe that 
\begin{align}
    \EE\lb \sup_{ \substack{T \in \calT \\ u: \|u\| \leq 1}} \frac 1n \sum_{i=1}^n \epsilon_i \langle u, T_{\mathbf{y}}(X_i) \rangle \rb  \leq 2 J  \EE\lb \sup_{v:\|v\|\leq 1} \frac 1n \sum_{i=1}^n \epsilon_i \langle v, X_i \rangle \rb \leq 2J \sup_{v:\|v\|=1} \|v\| \EE\lb \left\lVert \frac 1n \sum_{i=1}^n \epsilon_i X_i \right\rVert \rb,  
\end{align}
where the first inequality follows from the contraction lemma, and the fact that the map $x \mapsto \langle u, T_{\mathbf{y}}(x) \rangle$ is $J$-Lipschitz, and the second inequality follows from an application of Cauchy-Schwarz. Next, by using Jensen's inequality and the concavity of the map $a \mapsto \sqrt{a}$, we have 
\begin{align}
    \EE\lb \left\lVert \frac 1n \sum_{i=1}^n \epsilon_i X_i \right\rVert \rb & = \EE \lb \sqrt{ \left\lVert \frac 1n \sum_{i=1}^n \epsilon_i X_i \right\rVert^2 } \rb \leq    \sqrt{ \EE \lb\left\lVert \frac 1n \sum_{i=1}^n \epsilon_i X_i \right\rVert^2 \rb} = \sqrt{ \EE \lb \frac 1 {n^2} \sum_{i=1}^n \sum_{j=1}^n \epsilon_i \epsilon_j \langle X_i, X_j \rangle \rb}.  
\end{align}
Since $\epsilon_i \perp \epsilon_j$ for $i \neq j$, we get the bound 
\begin{align}
    \RadComp_n(\calF_{\yy}) \leq 2J \sqrt{\frac 1 {n^2} \sum_{i=1}^n \EE[\|X\|^2]} = \calO\lp \frac 1 {\sqrt{n}} \rp, 
\end{align}
which converges to $0$ with $n$.

\paragraph{Rademacher Complexities $\RadComp_n^{(1)}(\calG_n)$ and $\RadComp_n^{(2)}(\calG_n, \calT)$.} For the first term, observe that with $U_i = (Y_i, Z_i)$ and $w_i = \boldsymbol{1}_{[-K_n, K_n]^{\dx}}(U_i)$ and $g(U_i) = h(U_i) w_i$ for some $h \in \calH_k$: 
\begin{align}
    \RadComp_n^{(1)}(\calG_n) = \mathbb{E}_{\epsilon^n, U^n} \lb \sup_{h \in \calH_k: \|h\|_k \leq b_n} \frac{1}{n} \sum_{i=1}^n \epsilon_i \langle h, k(U_i, \cdot)\rangle_k w_i \rb 
    \leq \frac{b_n}{n} \EE_{ U^n} \lb \sqrt{\sum_{i=1}^n w_i^2 k(U_i, U_i)} \rb \leq \frac{b_n}{\sqrt{n}}, 
\end{align}
where the first inequality uses a standard argument for RKHSs presented in details in~\Cref{proof:nf-theorem-uniform-convergence}. 

Next, for any $T \in \calT$ and $g\in \calG_n$, introduce the terms $\Utilde_i = (Y_i, T_z(X_i))$ and $\wtilde_i = \boldsymbol{1}_{[-K_n, K_n]^{\dx}}(\Utilde_i)$, and observe that 
\begin{align}
    \RadComp_n^{(2)}(\calG_n, \calT)  \leq  \mathbb{E}_{\epsilon^n, \Utilde^n} \lb \sup_{T, g} \frac{1}{n} \sum_{i=1}^n \epsilon_i \circ g(\Utilde_i) \rb \lesssim \frac{A_{1,n} b_n}{\sqrt{n}}, 
\end{align}
where we used the  (scalar) contraction lemma for Rademacher complexities.

\subsubsection{Verification of~\Cref{prop:inn4-variational-gap}}
\label{proof:inn-example-variational-gap}

We now obtain the terms $\tau_1, \tau_2$ and $\delta_n$ from~\Cref{prop:inn4-variational-gap} for our specific example. Recall that we are using the notation $P \equiv P_{Y,Z}$ and $Q_T \equiv P_{Y, T_z(X)}$

First, we consider the term $\tau_1$, which is equal to $\sup_{T \in \calT} Q_T(B_{K_n}^c) + P(B_{K_n}^c) = \sup_{T \in \calT} Q_T\lp \{u: \|u\| > K_n\} \rp + P(\{u:\|u\|> K_n\})$. Equivalently, with the random vector $U_T = (Y, T_z(X))$ for some $T \in \calT$, we can handle $Q_T\lp \{u: \|u\| > K_n\} \rp = \mathbb{P}\lp \|U_T\|> K_n \rp$ simply by using the Markov's inequality as 
\begin{align}
\mathbb{P}\lp \|U_T \| > K_n \rp \leq \frac{\EE[\|U_T\|^{2 + \beta}]}{K_n^{2+\beta}}   \leq \frac{R}{K_n^{2+\beta}}, 
\end{align}
where the term $R$ follows from the uniform moment condition that we verified using~\Cref{prop:inn3-uniform-moment-bound}. An exact same argument also works for the term $P(B_{K_n}^c)$. Hence, $\tau_1$ converges to $0$ as long as $K_n \to \infty$ with $n$. 

Next, we look at the term $\tau_2$ which is defined as 
\begin{align}
    \tau_2 = \sup_{T \in \calT} \lbr \EE_P[(|g^*_T| - b_n)^+] + A_{1,n} \EE_{Q_T}\lb (|g^*_T|-b_n)^+ \rb \rbr, 
\end{align}
where $g^*_T$ is the optimal (unconstrained) critic function~(or witness function) in the definition of JS-divergence. A standard calculation shows that 
\begin{align}
    g^*_T = \log \lp \frac{2 v_T}{1 + v_T} \rp, \qtext{where} v_T = \frac{dP}{dQ_T}. 
\end{align}
Next, observe that we $g^*_T$ is uniformly  bounded from above by $\log 2$, and we can use the simple inequality 
\begin{align}
    |g^*_T| = \lv \log 2 +  \ell_T - \log(1+v_T) \rv \leq \log 2 + | \ell_T|, \qtext{where} \ell_T = \log v_T = \log \frac{dP}{dQ_T}.  
\end{align}
As a result, for any $b_n > \log 2$, we have 
\begin{align}
    \lp |g_T^*| - b_n \rp^+ & \leq \lp |\ell_T| - (b_n - \log 2) \rp^+ \leq |\ell_T| \boldsymbol{1}_{|\ell_T| > b_n - \log 2}. 
\end{align}
Now, following the exact argument for the \ac{NF} case, we know from~\Cref{lemma:nf-likelihood-ratio},  that there exist constants $c_0, c_1, c_2$ independent of $n, T$, such that we have 
\begin{align}
    |\ell_T(x)| \leq c_0 + c_1 \|x\| + c_2 \|x\|^2 \qtext{for all} x \in \mathbb{R}^d. 
\end{align}
If the distribution $P_X$ is such that $\mathbb{E}_{P_X}[\|X\|^2] < \infty$, then, this implies that 
\begin{align}
    \mathbb{E}_{P_X}[|\ell_T| \boldsymbol{1}_{|\ell_T| > b_n - \log 2}]  \leq \mathbb{E}_{P_X}[c_0 + c_1 \|X\| + c_2 \|X\|^2: \|X\|_{\infty} > b_n - \log 2] \stackrel{n \to \infty}{\longrightarrow} 0. 
\end{align}
Thus, the only condition needed to control the above term is that $b_n \to \infty$, and that $\mathbb{E}_{P_X}[\|X\|^2] < \infty$. The same argument also works for $\EE_{Q_T}[|\ell_T| \boldsymbol{1}_{|\ell_T| > b_n - \log 2}] \to 0$ due to the bounded $2+\beta$ moment condition. 

It remains to consider the approximation error term on the ball $B_{K_n}  = \{u \in \R^{\dx}: \|u\| \leq K_n\}$, defined as 
\begin{align}
    \delta_n \equiv \delta_n(b_n, K_n) &= \sup_{T} \inf_{g} \lbr \EE_P\lb |\bar{g}_T - g| \boldsymbol{1}_{B_{K_n}} \rb + \EE_{Q_T}\lb |\bar{g}_T - g| \boldsymbol{1}_{B_{K_n}} \rb \rbr \\
    & \leq 2 \sup_{T} \inf_{g} \|\bar{g}_T - g\|_{\infty, C_{K_n}}, 
\end{align}
where recall that $C_{K_n} = [-K_n, K_n]^{\dx}$, $\bar{g}_T$ is the clipped version of $g_T^*$ at level $b_n$, and $\|\bar{g}_T - g\|_{\infty, C_{K_n}}$ denotes the $\sup$-norm over the cube $C_{K_n}$. To complete the proof, it suffices to argue that for every $T \in \calT$, there exists a $g_{T,n} \in \calG_n$, such that  \begin{align}
    \|\bar{g}_T - g_{T,n}\|_{\infty, C_{K_n}} = \|{g}_T^* - g_{T,n}\|_{\infty, C_{K_n}} \lesssim \frac{K_n}{\sqrt{\gamma_n}}. 
\end{align}
This is ensured by~\Cref{lemma:nf-approximation}. 

\subsubsection{Completing the proof}
Recall that the statement of~\Cref{theorem:inn-example} makes the following choices of the various parameters: 
\begin{itemize}
    \item $K_n = M(s \sqrt{H} + B) + \sqrt{\dx} + \sqrt{2 \log n}$
    \item $\gamma_n = K_n^{2 + \epsilon}$ for some $\epsilon>0$
    \item $b_n = C_b \gamma_n^{\dx/4}K_n^{2 + \dx/2}$
    \item $A_{1,n} \leq 1$ and $A_{2,n} \leq b_n$
\end{itemize}

Now, observe that from~\eqref{eq:inn2-prop-Ly}, these choices imply 
\begin{align}
    &\sup_{T \in \calT} |\hat{L}_y(T) - L_{\mathbf{y}}(T)| \lesssim  \sqrt{\frac{\log n}{n} } + \frac{\log n}{\sqrt{n}} + \frac{1}{(\log n)^{2/\beta}} \; \stackrel{n\to\infty}{\longrightarrow} 0\\ 
    &\sup_{T \in \calT} |\hat{L}_z(T) - L_{\mathbf{z}}(T)| \lesssim \frac{b_n}{\sqrt{n}} + (b_n + A_{2,n}) \sqrt{\frac{\log(1/\delta)}{n}} \lesssim \frac{(\log n)^{\dx\lp \frac 12 + \frac{\epsilon}{8} \rp + 1} }{\sqrt{n}}\; \stackrel{n\to\infty}{\longrightarrow} 0. 
\end{align}
This ensures that the conditions of~\Cref{prop:inn2-uniform-convergence} are satisfied by these parameters. To complete the proof, we will show that the terms $\tau_1, \tau_2$, and $\delta_n$ from~\Cref{prop:inn4-variational-gap} also converge to  $0$ for these parameters. As we showed in the previous section, 
\begin{itemize}
    \item $\tau_1 \lesssim \frac{1}{K_n^{2+\beta}}$ for some $\beta >0$, and hence $\tau_1 \to 0$ with $n \to \infty$ since $K_n \asymp \sqrt{\log n}$. 
    \item $\tau_2 \lesssim \EE_{P}[\|X\|^2: \|X\|_{\infty} > b_n - \log 2] + \EE_{Q_T}[\|X\|^2: \|X\|_{\infty} > b_n - \log 2]$ which converges to $0$ as $\lim_{n \to \infty} b_n = \infty$. 
    \item Finally, we also showed that 
    \begin{align}
        \delta_n \lesssim \frac{K_n}{\sqrt{\gamma_n}} \asymp \frac{1}{K_n^{\epsilon/2}}, 
    \end{align}
    which also vanishes as $n \to \infty$. 
\end{itemize}
This completes the verification of all the conditions for the \ac{INN} example introduced in~\Cref{subsec:INN-example}.

\section{Details of Technical Lemmas}
\label{appendix:technical-lemmas}

\subsection{A Truncation Argument}
\label{appendix:truncation}
\begin{lemma}
    \label[lemma]{lemma:truncation} Suppose $U \sim \mu$ and $V \sim \nu$ denote two random variables taking values in $\mathbb{R}^{j}$ for some $j \geq 1$, and suppose that $\max\lbr \mathbb{E}[\|V\|^{1+a}], \, \mathbb{E}[\|U\|^{1+a}] \rbr = R< \infty$, for some $a>0$. Then, we have the following: 
    \begin{align}
        W_1\lp \mu, \nu\rp = \sup_{\substack{f:\mathbb{R}^j \to \mathbb{R} \\ \mathrm{Lip}(f) \leq 1}} \lv \int f d \mu - \int f d\nu \rv  \lesssim \Delta^{\frac{a}{2(1+a)}}, 
    \end{align}
    where $\Delta = \|\mu - \nu\|_{TV}$ denotes the total variation distance between $\mu$ and $\nu$.  
\end{lemma}

\begin{proof}
    Let $M \in \mathbb{R}$ denote a finite positive real number to be specified later, and let us define the set $E = \{x \in \mathbb{R}^j: \|x\|\leq M\}$. Then, we have the following with $g = f - f(0)$: 
    \begin{align}
    \int f d\mu - \int f d\nu & = \int f d(\mu - \nu) - f(0) \int d(\mu-\nu) = \int g d(\mu-\nu) \\
    &= \int_E g d(\mu-\nu)  + \int_{E^c} g d(\mu-\nu) \\
     & \leq \int_E \|x\| d|\mu-\nu| + \lv \int_{E^c} g(x) d(\mu-\nu) \rv, \label{eq:truncation-proof-1}
    \end{align}
    where the last inequality uses the fact that $g = f - f(0)$ is $1$-Lipschitz. 
    The definition of $E$ leads to the following natural upper bound on the first term 
    \begin{align}
    \int_E \|x\| d|\mu-\nu| \leq M \int_E d|\mu-\nu| \leq M \int_{\mathbb{R}^j}d|\mu - \nu| \leq 2 M \Delta. \label{eq:truncation-proof-2}
    \end{align}
    The last inequality above uses the fact that $2\|\mu-\nu\|_{TV} =  |\mu-\nu|(\mathbb{R}^d)$. Next, we consider the second term in~\eqref{eq:truncation-proof-1}, and observe that  
    \begin{align}
        \lv \int_{E^c} g d(\mu-\nu) \rv &\leq  \int_{E^c} |g| (d\mu + d\nu) \leq \int_{E^c} \|x\| d(\mu+\nu) 
         = \mathbb{E}\lb \|U\|\boldsymbol{1}_{\|U\|>M} \rb +  \mathbb{E}\lb \|V\|\boldsymbol{1}_{\|V\|>M} \rb \\
         & \leq  M^{-a}\mathbb{E}\lb \|U\|^{1+a}\boldsymbol{1}_{\|U\|>M} \rb +  M^{-a}\mathbb{E}\lb \|V\|^{1+a}\boldsymbol{1}_{\|V\|>M} \rb \\
         & \leq  M^{-a}\mathbb{E}\lb \|U\|^{1+a} \rb +  M^{-a}\mathbb{E}\lb \|V\|^{1+a} \rb \leq 2M^{-a} R.  \label{eq:truncation-proof-3}
    \end{align}
    Combining the bounds in~\eqref{eq:truncation-proof-2} and~\eqref{eq:truncation-proof-3}, we get 
    \begin{align}
        W_1(\mu, \nu) = \sup_{g} \lv \int g d(\mu-\nu) \rv \leq 2M \Delta + 2M^{-a} R. \label{eq:truncation-proof-4}
    \end{align}
    The function $M \mapsto 2M\Delta + 2M^{-a} R$ is  convex in $M$, and is optimized at $M^* = (aR/\Delta)^{1/(1+a)}$. Plugging this value back in~\eqref{eq:truncation-proof-4} gives us the required 
    \begin{align}
        W_1(\mu, \nu) \leq  2 \lp a^{\frac{1}{1+a}} + a^{\frac{-a}{1+a}} \rp R^{\frac{1}{1+a}} \Delta^{\frac{a}{1+a}} \coloneqq C_a R^{\frac{1}{1+a}} \Delta^{\frac{a}{1+a}}. 
    \end{align}
    As a sanity check, note that if the supports of $U$ and $V$ are bounded and $a \to \infty$, we recover the bound that holds for bounded random variables; that is, $W_1(\mu, \nu) \lesssim \Delta$. 
\end{proof}

\subsection{A Conditional $f$-Divergence Bound}
\label{appendix:f-divergence-bound}

\begin{lemma}
\label[lemma]{lemma:condition-f-div} Let $P, Q$ denote two probability measures on $\calY \times \calZ$, and for any measurable set $A \subset \calY$, let $P_Y(A) = P(A \times \calZ) > 0$. Define the measures 
\begin{align}
    \PA(dy, dz) = \frac{1}{P_Y(A)} \boldsymbol{1}_{y \in A} P(dy, dz), \qtext{and} Q^{(A)}(dy, dz) = \frac{1}{P_Y(A)} \boldsymbol{1}_{y \in A} Q(dy, dz). 
\end{align}
Then, we have the following: $D_f(P \parallel Q) \leq \tfrac{1}{P_Y(A)} D_f(\PA \parallel Q^{(A)})$. Note that in general, $\QA$ may not be a probability measure, but the inequality still holds. 
\end{lemma}
\begin{proof}
    Let $\mu$ denote a measure that dominates $P, Q$; for example, we may set $\mu = \tfrac 12 (P+Q)$, and let $p, q$ denote the densities of $P, Q$ w.r.t. $\mu$. It then follows that 
    \begin{align}
        \frac{d\PA}{d\mu}(y,z) = \frac{\boldsymbol{1}_{y \in A}}{P_Y(A)} p(y,z), \qtext{and}\frac{d\QA}{d\mu}(y,z) = \frac{\boldsymbol{1}_{y \in A}}{P_Y(A)} q(y,z).
    \end{align}
    As a result, we obtain 
    \begin{align}
        \frac{d\PA}{d\QA}(y,z) = \frac{p(y,z)}{q(y,z)} =  \frac{dP}{dQ}(y,z) \qtext{for all} (y,z) \in A \times \calZ.
    \end{align}
    Now, by the definition of $f$-divergences, we have 
    \begin{align}
        D_f(\PA \parallel \QA) &= \int_{A \times \calZ} f \lp \frac{d\PA}{d\QA}(y,z) \rp \QA(dy, dz) = \frac{1}{P_Y(A)} \int_{A \times \calZ} f\lp \frac{p(y,z)}{q(y,z)} \rp q(y,z) \mu(dy, dz) \\
        & \stackrel{\eqref{eq:f-div-proof}.1}{}\leq  \frac{1}{P_Y(A)} \int_{\calY \times \calZ} f\lp \frac{p(y,z)}{q(y,z)} \rp q(y,z) \mu(dy, dz) = \frac{1}{P_Y(A)} D_f(P \parallel Q),  \label{eq:f-div-proof}
    \end{align}
    where $\eqref{eq:f-div-proof}.1$ uses the nonnegativity of $f$ and $q$. This completes the proof. 
\end{proof}

\subsection{Proof of~\Cref{lemma:nf-Jacobian}}
\label{proof:nf-Jacobian}

Consider any $j \in [M]$, and let $z_j(x) = W_{j,1} x + b_{j,1}$ and $D_j(x) = \operatorname{diag}( \tanh'(z_j(x))) =  \operatorname{diag}(\operatorname{sech}^2(z_j(x)))$. Since $\operatorname{sech}^2(z) \in [0,1]$, it follows that $\boldsymbol{0} \preccurlyeq D_j(x) \preccurlyeq I_{d_j}$. Then, we can verify that the Jacobians of $F_j$ and $G_j$, denoted by $\Jac F_j$ and $\Jac G_j$ respectively,  are equal to 
\begin{align}
    \Jac F_j(x) = W_{j,2} D_j(x) W_{j,1}, \qtext{and} \Jac G_j(x) = I_d + W_{j,2} D_j(x) W_{j,1}. 
\end{align}
By the assumption that $\|W_{j,i}\|_{op} \leq s \in (0, 1)$ for all $j \in [M]$, we obtain 
\begin{align}
    \|\Jac F_j(x)\|_{op} \leq \|W_{j,2}\|_{op} \|D_j(x)\|_{op} \|W_{j,1}\|_{op} \leq s^2 = \rho. 
\end{align}
Here we used the fact that $\|D_j(x)\|_{op} \leq \|I_{d_j}\|_{op} = 1$. As a result, all the singular values of $\Jac G_j(x) = (I_d + \Jac F_j)(x)$ are in the range $[1-\rho, 1+ \rho]$. Now, by the chain rule of the flow, and with $x_0 = x$, $x_j = G_j x_{j-1}$ for $j \in [M]$, we have 
\begin{align}
    \Jac T(x) = \prod_{j=M}^{1} \Jac G_j(x_{j-1}) \quad \implies \quad \log |\det \Jac T(x)| = \sum_{j=1}^M \log |\det (I_d + W_{j,2} D_j(x_{j-1}) W_{j,1})|. 
\end{align}
Since for every $j \in [M]$, the singular values of $\Jac G_j(x_{j-1})$ lie in $[1-\rho, 1+\rho]$, we obtain 
\begin{align}
    M d \log (1-\rho) \leq \sum_{j=1}^M \log |\det \Jac G_j(x_{j-1})| \leq M d \log (1+\rho). 
\end{align}
Note that the above result is uniform in  the input $x$, and uniform over all $T \in \calT$.

\subsection{Proof of~\Cref{lemma:nf-likelihood-ratio}}
\label{proof:nf-likelihood-ratio}

Let $Z \sim N(0, I_d)$, and denote its density by $\phi(z) = (2\pi)^{-d/2} e^{-\|z\|^2/2}$. Then, for any diffeomorphism $T:\mathbb{R}^d \to \mathbb{R}^d$, we have the following by the change of variables theorem: 
\begin{align}
    p_T(x) = \phi(T(x)) \det \Jac T(x). 
\end{align}
Since $\ell_T = \log p_{T^*}/p_T$, this implies 
\begin{align}
    \ell_T(x) = \log \phi(T^*(x)) + \log |\det \Jac T^*(x)| -  \log \phi(T(x)) - \log |\det \Jac T(x)|. 
\end{align}
Since $\log \phi(z) = -(1/2) \|z\|^2 - d/2 \log (2\pi)$, we obtain the following closed form expression for $\ell_T$: 
\begin{align}
    \ell_T(x) = \frac{1}{2} \lp \|T(x)\|^2  -  \|T^*(x)\|^2  \rp + \log |\det \Jac T^*(x)| - \log |\det \Jac T(x)|. \label{eq:nf-likelihood-ratio-proof-1}
\end{align}
We now establish the uniform Lipschitz constant and the maximum value of $\{\ell_T: T \in \calT\}$ when restricted to the cube $C_K = [-K, K]^d$ for some $K>0$. 

\paragraph{Lipschitz Constant.}  To derive the Lipschitz constant of $\ell_T$, we look at its gradient at any $x \in [-K, K]^d$: 
\begin{align}
    \nabla \ell_T(x) = \Jac T(x)^T T(x) - \Jac T^*(x)^T T^*(x) + \nabla \log \det \Jac T^*(x) -  \nabla \log \det \Jac T(x). 
\end{align}
For each coordinate $i$, we have 
\begin{align}
    \partial_{x_i} \log \det \Jac T(x) = \operatorname{tr}\lp (\Jac T(x))^{-1} \partial_{x_i}\Jac T(x) \rp, \qtext{and}
    \partial_{x_i} \log \det \Jac T^*(x) = \operatorname{tr}\lp (\Jac T^*(x))^{-1} \partial_{x_i}\Jac T^*(x) \rp \label{eq:nf-likelihood-ratio-proof-2}
\end{align}
We have already proved in~\Cref{proof:nf-theorem-moment} that 
\begin{align}
    \|T(x)\| \leq L \sqrt{d} K + C_0 A, \qtext{where} L=(1+s^2)^M, \quad C_0 = s\sqrt{H} + B, \qtext{and} A = \frac{(1+ s^2)^M - 1}{s^2}. 
\end{align}
Hence, we get 
\begin{align}
    \max\lbr \|\Jac T(x)^T T(x)\|, \; \|\Jac T^*(x)^T T^*(x)\| \rbr \leq L(L\sqrt{d} K + C_0 A). 
\end{align}

We now need to analyze the norm of the gradient of the log-det terms in~\eqref{eq:nf-likelihood-ratio-proof-2}. We begin with the directional derivative along any direction $v$: 
\begin{align}
    D_v\lp \log \det \Jac T(x) \rp = \operatorname{tr}\lp (\Jac T(x))^{-1} D_v \Jac T(x) \rp \leq d \|\Jac T(x)^{-1}\|_{op} \|D_v \Jac T(x)\|_{op}. 
\end{align}
For any residual block $G_j$, we know that 
\begin{align}
    \|\Jac G_j(u)^{-1}\|_{op} \leq (1-s^2)^{-1}, \qtext{and} \|D_v \Jac G_j(u)\|_{op} \leq 2 s^3 \|v\|. 
\end{align}
On chaining this for $T = G_M \circ \cdots \circ G_1$, we get 
\begin{align}
    \|D_v \log \det \Jac T(x)\| \leq \sum_{j=1}^M \frac{d}{1-s^2} 2 s^3 \|\Jac (G_M \circ \cdots \circ G_1(x) v \| \leq \frac{2 d s^3}{1-s^2} A \|v\|. 
\end{align}
Hence, we get the required bound 
\begin{align}
    \max \lbr \|\nabla \log \det \Jac T(x)\|, \;  \|\nabla \log \det \Jac T^*(x)\| \rbr \leq \frac{2 d s^3}{1-s^2} A. \label{eq:nf-likelihood-ratio-proof-3}
\end{align}
Together,~\eqref{eq:nf-likelihood-ratio-proof-2} and~\eqref{eq:nf-likelihood-ratio-proof-3} imply  the following for $x, y \in [-K, K]^d$: 
\begin{align}
    |\ell_T(x) -  \ell_T(y)| \leq \lb 2 L (L \sqrt{d} K + C_0 A) + \frac{4 d s^3}{1-s^2} A \rb \|x-y\| \eqcolon L_1 \|x-y\|. \label{eq:nf-likelihood-ratio-proof-4}
\end{align}
Note that the Lipschitz constant $L_1$ is independent of $T$, and hence is valid uniformly over the family $\calT$.  

\paragraph{Maximum value of $\ell_T$ on $[-K, K]^d$.} From~\eqref{eq:nf-likelihood-ratio-proof-1}, we know that the value of $\ell_T$ at $0$ is 
\begin{align}
\ell_T(0) =  \frac{1}{2}\lp \|T(0)\|^2 - \|T^*(0)\|^2  \rp  + \log \det \Jac T^*(0) - \log \det \Jac T(0). 
\end{align}
From~\Cref{lemma:nf-Jacobian}, we know that $Md \log(1-s^2) \leq \log \det \Jac T(x) \leq Md \log (1+s^2)$, and from~\Cref{proof:nf-theorem-moment}, we know that $\|T(0)\| \leq C_0 A$. Combining these facts, we get the required bound 
\begin{align}
    \sup_{T \in \calT} |\ell_T(0)| \leq (C_0 A)^2 + 2 M d \log(1/(1-s^2)) \eqcolon L_0.  \label{eq:nf-likelihood-ratio-proof-5}
\end{align}
This fact, combined with the (uniform) Lipschitz constant derived in~\eqref{eq:nf-likelihood-ratio-proof-4}, leads to the following uniform bound on the maximum value of $\ell_T$: 
\begin{align}
    \sup_{T \in \calT} \sup_{x \in [-K, K]^d} |\ell_T(x)| \leq L_0 + L_1 \sqrt{d} K. 
\end{align}
This completes the proof. \hfill %

\subsection{Proof of~\Cref{lemma:nf-approximation}}
\label{proof:nf-approximation}

The high-level idea behind the proof is simple: given the function that we want to approximate~(i.e., $\ell_T$), we construct $h_\sigma$ by convolving this function with a Gaussian kernel with a scale parameter $\sigma^2$~(to be chosen later). Under certain conditions on $\sigma$, we can show that this smoothed function lies in $\calH_n$. Hence, getting a bound on $\delta_n$ reduces to that of evaluating the discrepancy between $\ell_T$ and its smoothed version $h_\sigma$, which can be obtained via standard arguments. 

\paragraph{A compactly supported Lipschitz extension.} The first step in our proof is to construct a proxy for $\ell_T$ that agrees with it on $C_{K_n} = [-K_n, K_n]^d$, and is Lipschitz on the entire domain $\mathbb{R}^d$. To do this, we define 
\begin{align}
    \bar{\ell}_T(x) = 
    \begin{cases}
        \ell_T(x), & \text{ if } x \in C_{K_n} = [-K_n, K_n]^d, \\
        \inf_{y \in C_{K_n}} \lbr \ell_T(y) + L\|x-y\| \rbr, & \text{ if } x \not \in C_{K_n}. 
    \end{cases}
\end{align}
We also need to restrict this Lipschitz extension to a compactly supported domain. To do this introduce the function $\nu_1: \mathbb{R} \to [0, 1]$, and $\nu_d: \mathbb{R}^d \to [0,1]$ as 
\begin{align}
    \nu_1(x) = \begin{cases}
        1, & \text{ if } |x| \leq K_n, \\
        1 - |x|/2K_n, & \text{ if } |x| \in (K_n, 2K_n], \\
        0, & \text{ otherwise},
    \end{cases}
    \qtext{and} \nu_d(x) = \prod_{i=1}^d \nu_1(x_i),
\end{align}
where $x = (x_1, \ldots, x_n)$. With this, we introduce  $f_\nu(x) = \bar{\ell}_T(x) \nu_d(x)$.
and observe that 
\begin{align}
    &\mathrm{Lip}(f_\nu) \leq \mathrm{Lip}(\bar{\ell}_T) \|\nu_d\|_{\infty} + \|\bar{\ell}_T\|_\infty \mathrm{Lip}(\nu_d) \leq L_1 + \frac{d(L_0 + L_1 K_n)}{K_n} \lesssim K_n, \label{eq:nf-approximation-proof-2}\\
    \text{and} \quad & \|f_\nu\|_{L^2} \leq \|f_\nu\|_{\infty} (4K_n)^{d/2} \lesssim K_n^{2 + d/2}. \label{eq:nf-approximation-proof-3}
\end{align}
The compact support property of $f_\nu$ was important for controlling its $L^2$ norm that will be used later. 

\paragraph{Gaussian Convolution.} Fix some $\sigma>0$, and with $\varphi_\sigma(u) = (2\pi \sigma^2)^{-d/2} \exp (-\|u\|^2/2\sigma^2)$, define the function $h_\sigma$ 
\begin{align}
    h_\sigma(x) = (f_\nu * \varphi_\sigma)(x) = \int f_\nu(x-u) \varphi_\sigma(u) du. 
\end{align}
Hence, for any $x \in C_{K_n} = [-K_n, K_n]^d$, we have 
\begin{align}
    |f_\nu(x) - h_\sigma(x)| &= \lv \int f_\nu(x) \varphi_\sigma(u) du  - \int f_\nu(x-u)  \varphi_\sigma(u) du \rv \\
    &  \leq \int |f_\nu(x) - f_\nu(x-u)| \varphi_\sigma(u)  \stackrel{(i)}{\leq} L_1 \int \|u\| \varphi_\sigma(u) du \\
    & \stackrel{(ii)}{\leq}L_1 \lp \int \|u\|^2 \varphi_\sigma(u) du \rp^{1/2} = L_1 \sqrt{d} \sigma \lesssim K_n \sigma. \label{eq:nf-approximation-proof-5} 
\end{align}
The inequality~$(i)$ uses the fact that the function $f_\nu$ is Lipschitz with constant $L_1 \lesssim K_n$ over the entire domain $\mathbb{R}^d$~(and not just over $C_{K_n}$ like $\ell_T$), and $(ii)$ uses Jensen's inequality for the concave map $x \mapsto \sqrt{x}$.

\paragraph{RKHS norm of $h_\sigma$.} We know that the symmetric Fourier transform of the Gaussian kernel is 
\begin{align}
    \hat{k}_n(\omega) = \lp \frac{\pi}{\gamma_n} \rp^{d/2} e^{-\|\omega\|^2/4\gamma_n}. 
\end{align}
Hence, the RKHS norm of any $h$ with $\hat{h} \in L^2$ is 
\begin{align}
    \|h\|_{k_n}^2 &= \frac{1}{(2\pi)^{d/2}} \int_{\mathbb{R}^d} \frac{|\hat{h}(\omega)|^2}{|\hat{k}_n(\omega)|} d\omega  = \lp\frac{\gamma_n}{\pi}\rp^{d/2} \frac{1}{(2\pi)^{d/2}} \int |\hat{h}(\omega)|^2 e^{\|\omega\|^2/4\gamma_n} d\omega  \\
    &\lesssim \gamma_n^{d/2} \int |\hat{h}(\omega)|^2 e^{\|\omega\|^2/4\gamma_n}d\omega
\end{align}
Now, returning to $h_\sigma$, we know that $\hat{h}_{\sigma}= \hat{f}_\nu e^{-\sigma^2\|\omega\|^2/2}$, which implies that if $\sigma^2 \geq 1/4\gamma_n$, we have by Plancherel's theorem, 
\begin{align}
    \|h_\sigma\|_{k_n}^2 \lesssim \gamma_n^{d/2} \int |\hat{f}_\nu(\omega)|^2 e^{-\|\omega\|^2(\sigma^2 - 1/4\gamma_n)} d\omega \leq \gamma_n^{d/2} \int |\hat{f}_\nu(\omega)|^2 d\omega = \gamma_n^{d/2} \|f_\nu\|_{L^2}^2 \lesssim \gamma_n^{d/2} K_n^{4 + d}, \label{eq:nf-approximation-proof-4}
\end{align}
where the last inequality uses the $L^2$ norm bound on $f_\nu$ obtained in~\eqref{eq:nf-approximation-proof-3}.  Hence, $h_\sigma \boldsymbol{1}_{[-K_n, K_n]^d}$ with $\sigma = 1/2 \sqrt{\gamma_n}$ lies in our function class $\calH_n$, and hence reduces our approximation bound to 
\begin{align}
    \sup_{x \in [-K_n, K_n]^d} |f_\nu(x) - h_\sigma(x)| \lesssim \frac{K_n}{\sqrt{\gamma_n}}. \label{eq:nf-approximation-proof-12}
\end{align}
This completes the proof. \hfill %

\newpage

\end{document}